\documentclass[12pt]{amsart}

\usepackage[utf8]{inputenc} 
\usepackage[T1]{fontenc}    
\usepackage[linktocpage=true]{hyperref}       
\usepackage{url}            
\usepackage{booktabs}       
\usepackage{fullpage}
\usepackage{amsthm,amssymb,amsfonts,mathtools}
\usepackage[noabbrev,capitalize,nameinlink]{cleveref}
\crefname{assumption}{Assumption}{Assumptions}
\crefname{informaltheorem}{Informal Theorem}{Informal Theorems}
\crefname{metatheorem}{Meta Theorem}{Meta Theorems}

\usepackage{todonotes}
\usepackage{caption}
\usepackage{subcaption}
\usepackage{algorithm}
\usepackage{algpseudocode}
\usepackage{multirow}
\usepackage{graphicx}
\usepackage{eqparbox}
\usepackage{comment}

\usepackage[
natbib,
style=numeric,
url=false,
doi=false,
sortcites,
backend=biber
]{biblatex}
\usepackage[shortlabels]{enumitem}

\newcommand{\abs}[1]{\left\lvert #1 \right\rvert}
\newcommand{\norm}[1]{\left\lVert #1 \right\rVert}
\DeclareMathOperator{\supp}{supp}

\newcommand{\red}{\normalcolor}

\newcommand{\nc}{\normalcolor}

\newtheorem{theorem}{Theorem}
\newtheorem{informaltheorem}{Informal Theorem}
\newtheorem{metatheorem}{Meta Theorem}
\newtheorem{proposition}{Proposition}
\newtheorem{lemma}{Lemma}
\newtheorem{corollary}{Corollary}
\theoremstyle{definition}
\newtheorem{definition}{Definition}
\newtheorem{assumption}{Assumption}

\newtheorem{example}{Example}
\newtheorem{remark}{Remark}

\DeclareMathOperator{\Per}{Per}
\DeclareMathOperator{\ProbPer}{ProbPer}

\DeclareMathOperator{\ProbTV}{ProbTV}
\DeclareMathOperator{\ProbJ}{ProbJ}
\DeclareMathOperator{\Risk}{R}
\DeclareMathOperator{\SRisk}{S}
\DeclareMathOperator{\ProbSRisk}{ProbS}
\DeclareMathOperator{\ProbRisk}{ProbR}
\DeclareMathOperator*{\esssup}{ess\,sup}

\DeclareMathOperator{\ProbAcc}{ProbAcc}
\newcommand{\de}{\,\mathrm{d}}
\newcommand{\X}{\mathcal{X}}
\newcommand{\Y}{\mathcal{Y}}
\newcommand{\A}{\mathcal{A}}
\newcommand{\B}{\mathfrak{B}}
\newcommand{\R}{\mathbb{R}}
\newcommand{\N}{\mathbb{N}}
\renewcommand{\P}{\mathcal{P}}

\newcommand{\eps}{\varepsilon}
\newcommand{\Exp}[2]{\mathbb{E}_{#1}\left[#2\right]}
\newcommand{\Prob}[2]{\mathbb{P}_{#1}\left[#2\right]}
\newcommand{\Set}[1]{{\left\lbrace#1\right\rbrace}}
\newcommand{\one}{\mathbf{1}}
\newcommand{\U}{\mathcal{U}}
\newcommand{\wsto}{\overset{\ast}{\rightharpoonup}}
\newcommand{\st}{\,:\,}

\DeclareMathOperator{\argmin}{arg min}
\newcommand{\veps}{\epsilon}

\renewcommand{\epsilon}{\varepsilon}
\newcommand{\distr}{\mathfrak{m}}

\title{It begins with a boundary: A geometric view on probabilistically robust learning}

\author{Leon Bungert}
\address{Institute of Mathematics \& Center for Artifical Intelligence and Data Science (CAIDAS), University of Würzburg, Germany}
\email{\href{mailto:leon.bungert@uni-wuerzburg.de}{leon.bungert@uni-wuerzburg.de}}

\author{Nicol\'as {Garc\'ia Trillos}}
\address{Department of Statistics, University of Wisconsin-Madison, US}
\email{\href{mailto:garciatrillo@wisc.edu}{garciatrillo@wisc.edu}}

\author{Matt Jacobs}
\address{Department of Mathematics, University of California Santa Barbara, US}
\email{\href{mailto:majaco@ucsb.edu}{majaco@ucsb.edu}}

\author{Daniel McKenzie}
\address{Department of Applied Mathematics and Statistics, Colorado School of Mines, US}
\email{\href{mailto:dmckenzie@mines.edu}{dmckenzie@mines.edu}}

\author{\DJ or\dj e Nikoli\'c}
\address{Department of Mathematics, University of California Santa Barbara, US}
\email{\href{mailto:nikolic@math.ucsb.edu}{nikolic@math.ucsb.edu}}

\author{Qingsong Wang}
\address{Department of Mathematics, University of Utah, US}
\email{\href{mailto:qswang@math.utah.edu}{qswang@math.utah.edu}}

\begin{document}

\begin{abstract}
Although deep neural networks have achieved super-human performance on many classification tasks, they often exhibit a worrying lack of robustness towards adversarially generated examples. Thus, considerable effort has been invested into reformulating standard Risk Minimization (RM) into an adversarially robust framework. 
    Recently, attention has shifted towards approaches which interpolate between the robustness offered by adversarial training and the higher clean accuracy and faster training times of RM.
    In this paper, we take a fresh and geometric view on one such method---Probabilistically Robust Learning (PRL) \cite{robey2022probabilistically}.
    We propose a mathematical framework for understanding PRL, which allows us to identify geometric pathologies in its original formulation and to introduce a family of probabilistic nonlocal perimeter functionals to rectify them.
    We prove existence of solutions to the original and modified problems using novel relaxation methods and also study properties, as well as local limits, of the introduced perimeters. We also clarify, through a suitable $\Gamma$-convergence analysis, the way in which the original and modified PRL models interpolate between risk minimization and adversarial training. 
    \nc 
\end{abstract}

\maketitle

\tableofcontents

\section{Introduction}
\label{sec:Introduction}
The fragility of deep neural network (DNN) based classifiers in the face of adversarial examples \cite{goodfellow2014explaining,chen2017zoo,qin2019imperceptible,cai2021zeroth} and distributional shifts \cite{quinonero2008dataset,hendrycks2021many} is by now nearly as familiar as their success stories. In light of this, a multitude of works (see \Cref{sec:related_work}) propose replacing standard Risk Minimization (RM) \cite{vapnik1999nature} with a more robust alternative, with adversarial training (AT) \cite{madry2017towards,goodfellow2014explaining} being a leading example. Unfortunately, there is no free lunch: robust classifiers frequently exhibit degraded performance on clean data and significantly longer training times \cite{tsipras2018robustness}. 
Consequently, identifying frameworks which balance performance and robustness is of pressing interest to the Machine Learning (ML) community, and over the past several years a few such frameworks have been proposed \cite{zhang2019theoretically,wang2020enresnet,robey2022probabilistically}. From the theoretical perspective, it is crucial that the mechanisms by which such frameworks balance these competing aims are understood.  

In this paper we revisit the Probabilistically Robust Learning (PRL) framework introduced in \cite{robey2022probabilistically} and discuss it, analyze it, and extend it using a geometric perspective. This perspective reveals certain subtle and paradoxical aspect of PRL:
It can have solutions which are very pathological classifiers, see \cref{fig:spike_solution}. Building on this observation one can conclude that solutions of PRL do not necessarily interpolate between standard risk minimization and AT. This interpolating property was one of the main theoretical motivations for designing the PRL model in~\cite{robey2022probabilistically}; see \cref{sec:Interpolation} for an extensive discussion on this. \nc 
\begin{figure}[ht]
    \begin{subfigure}{0.48\textwidth}
        \fbox{%
        \begin{tikzpicture}
        \node[anchor=west] at (0,0) {\includegraphics[width=0.9\textwidth,trim=0pt 80pt 0pt 0pt,clip]{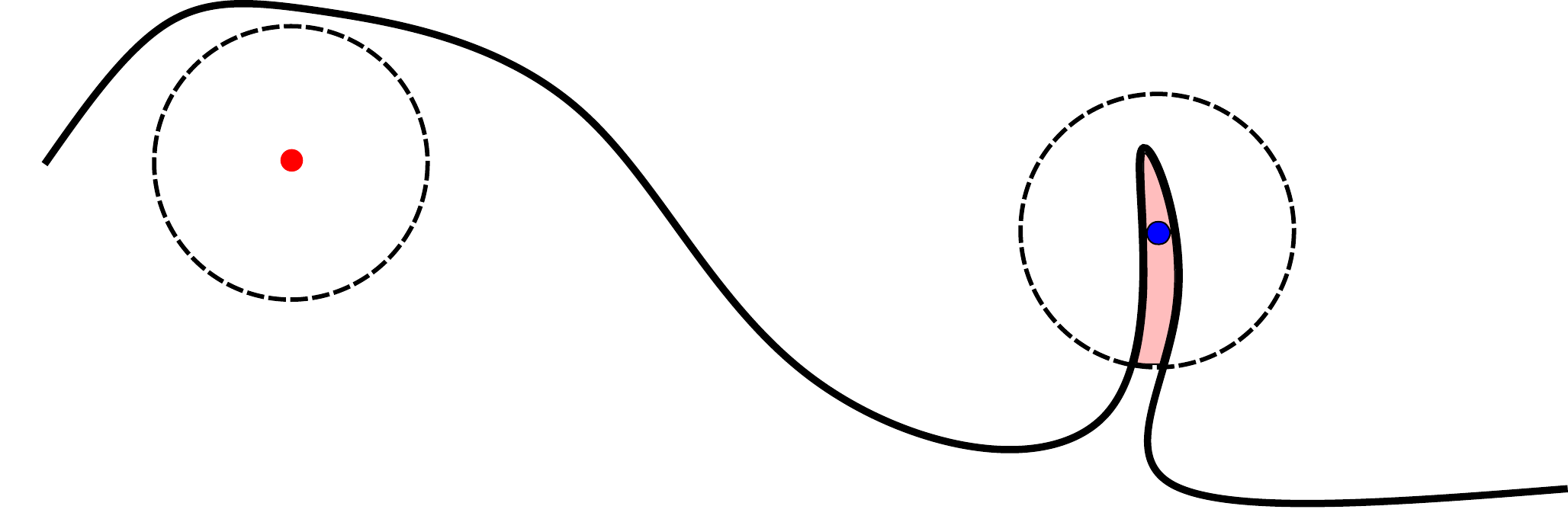}};
        \node[anchor=west] at (20pt,-20pt) {\nc$A\triangleq\text{class red}$};
        \node[anchor=west] at (110pt,35pt) {\nc$A^c\triangleq\text{class blue}$};
        \end{tikzpicture}}
        \caption{This pathological classifier minimizes the original PRL risk \emph{but not} the modified one.
        \label{fig:flaw}}
    \end{subfigure}
    \hfill%
    \begin{subfigure}{0.48\textwidth}
        \fbox{%
        \begin{tikzpicture}
        \node[anchor=west] at (0,0) {        \includegraphics[width=0.9\textwidth,trim=0pt 80pt 0pt 0pt,clip]{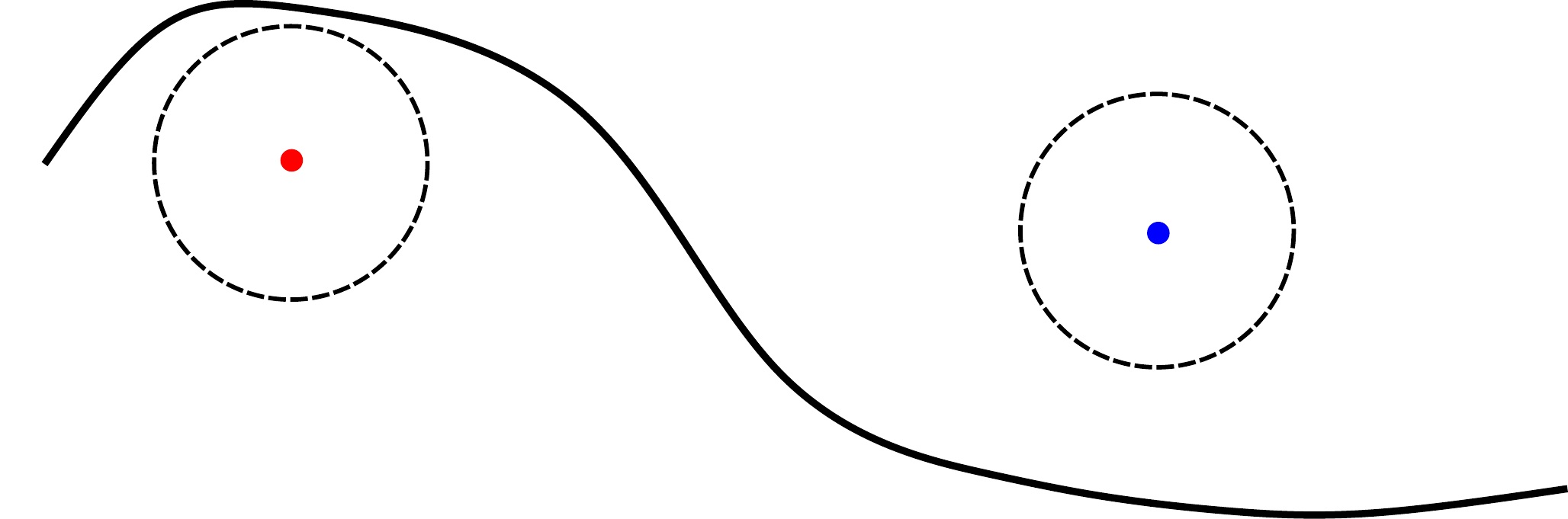}};
        \node[anchor=west] at (20pt,-20pt) {\nc$A\triangleq\text{class red}$};
        \node[anchor=west] at (110pt,35pt) {\nc$A^c\triangleq\text{class blue}$};
        \end{tikzpicture}}
        \caption{This classifier minimizes the original PRL risk \emph{and} the modified one.
        \label{fig:no_flaw}}
    \end{subfigure}
    \caption{Even for the simple task of classifying two data points one can easily construct a pathological solution (\cref{fig:flaw}) of PRL, observing that all but the small red set of perturbations of the \emph{misclassified} blue point are correctly classified as blue. 
    Both classifiers depicted in \cref{fig:flaw,fig:no_flaw} have zero PRL loss.}
    \label{fig:spike_solution}
\end{figure}

Fortunately, our geometric perspective suggests a natural remedy which leads to an interpretation of the modified PRL as regularized RM, where a new notion of nonlocal length (or perimeter) of decision boundaries acts as a regularizer; this notion of perimeter (see \labelcref{eq:ProbPerPsi} for its definition) is of interest in its own right and in \cref{sec:Perimeters} we will discuss some of its properties. The interpretation of PRL as perimeter-regularized RM leads us to further generalizations and allows us to provide a novel view of the Conditional Value at Risk (CVaR) relaxation of PRL proposed in \cite{robey2022probabilistically}. We provide conditions that guarantee the existence of solutions to the optimization problems determining our models as well as the original PRL models, and clarify, through rigorous analysis using the notion of $\Gamma$-convergence, the sense in which the original PRL and the modified PRL models interpolate between adversarial training and standard risk minimization. \red Our existence proofs exploit very interesting structural properties of our functionals that allow us to circumvent the apparent incompatibility between the topologies under which our functionals are lower semicontinuous and under which the compactness of minimizing sequences can be guaranteed. \nc

\subsection{From empirical risk minimization to robustness}

Given an input space $\X$, an output space $\Y$, a probability measure $\mu\in\P(\X\times\Y)$, a loss function $\ell :\Y\times\Y\to\R$, and a hypothesis class $\mathcal{H}$ (i.e., a class of measurable functions from $\X$ into $\Y$), the standard risk minimization problem is
\begin{align}\tag{RM}
    \inf_{h \in \mathcal{H}} \Exp{(x,y) \sim \mu}{\ell(h(x),y)}.
    \label{eq:ERM}
\end{align}
To train classifiers which are robust against adversarial attacks, \cite{goodfellow2014explaining,madry2017towards} suggested adversarial training:
\begin{align}\tag{AT}
    \inf_{h \in \mathcal{H}} \Exp{(x,y) \sim \mu}{\sup_{x' \in B_\eps(x)}\ell(h(x'),y)}.
    \label{eq:AT}
\end{align}
In the above, $\X$ is assumed to have the structure of a metric space and $B_\eps(x)$ for $\eps>0$ denotes the (open or closed) ball of radius~$\eps$ around~$x$, although for convenience we will consider open balls whenever is needed.

The recent work \cite{robey2022probabilistically} offered an alternative to adversarial training that attempts to reduce the (in general) large trade-off between accuracy and robustness inherent in \labelcref{eq:AT}; see \cite{tsipras2018robustness,robey2022probabilistically} for discussion. 
Instead of requiring classifiers to be robust to \emph{all} possible attacks around a point $x$---as enforced through the supremum in \labelcref{eq:AT}---one may consider a less stringent notion of robustness, only requiring classifiers to be robust to $100\times (1-p)\%$ of attacks drawn from a certain distribution $\distr_x$ centered at $x$. 
For this, the authors replace the supremum in \labelcref{eq:AT} by the so-called $p$-$\esssup$ operator for some value $p\in[0,1)$.
To define this operator, consider a probability distribution $\distr$ and a function $f$ and let
\begin{align*}
    p\text{-}\esssup_{x' \sim \distr} f(x') 
    :=
    \inf\Set{t\in\R\st\Prob{x'\sim\distr}{f(x')> t}\leq p}.
\end{align*}
In the mathematical finance or economics literature the $p$-$\esssup$ operator is better known as the value at risk (VaR) of a random variable at level $p$.
It is the smallest value $t\in\R$ such that the probability that a randomly chosen point $x'\sim\distr$ satisfies $f(x')>t$ is smaller than $p$. This notion reduces to the usual essential supremum of $f$ with respect to $\distr$ if $p=0$, which explains the name $p$-$\esssup$. In \cite{robey2022probabilistically}, it was thus suggested to replace~\labelcref{eq:AT} with the \textit{probabilistically robust learning} problem 
\begin{align}\tag{PRL}
    \inf_{h \in \mathcal{H}} \Exp{(x,y) \sim \mu}{p\text{-}\esssup_{x' \sim \distr_{x}}\ell(h(x'),y)},
    \label{eq:Pappas}
\end{align}
where $\{\distr_{x}\}_{x\in \X}$ is a suitable family of probability distributions, one for each $x$ in the data space, interpreted as user-chosen distributions ``centered'' around the different $x$. The prototypical example to keep in mind for $\X=\R^d$ is the uniform distribution over the $\eps$-ball around $x$, i.e., $\distr_{x} := \operatorname{Unif}(B_\eps(x))$, which is particularly relevant when dealing with adversarial attacks on image classifiers. We will provide further theoretical discussion on the choice of the \red measures $\distr_x$ in \cref{sec:Interpolation}.\nc

\subsection{Modifying PRL for binary classification}

To understand \labelcref{eq:Pappas} better, we focus first (and mostly) on the binary classification problem (i.e., $\Y=\{0,1\}$) using indicator functions of admissible sets (i.e., $\mathcal{H} := \Set{\one_A \st A \in \A}$). For now, one can think of $\A\subset 2^{\X}$ as a suitable $\sigma$-algebra which fits to the problem.
Note that we identify the two expressions $\one_A(x) = \one_{x\in A}$.
We focus on the standard  $0$-$1$ loss $\ell(\tilde y,y)=\one_{\tilde y\neq y}$, which equals one if $y\neq\tilde y$ and zero otherwise.
In this scenario, the standard risk minimization problem \labelcref{eq:ERM} reduces to 
\begin{align}\label{eq:ERM_binary}
    \inf_{A \in \A} 
    \Big\lbrace
    \Risk_\mathrm{std}(A):=
    \Exp{(x,y) \sim \mu}{y\one_{x\in A^c} + (1-y)\one_{x\in A}}
    \Big\rbrace,
\end{align}
and we refer to its minimizers as \textit{Bayes classifiers} and to its minimum value as \textit{Bayes risk}.

Similarly, \labelcref{eq:AT} can be rewritten as
\begin{align}\label{eq:AT_binary}
    \inf_{A \in \A} 
    \bigg\lbrace
    \Risk_\mathrm{adv}(A):=
    \Exp{(x,y) \sim \mu}{y\one_{x\in\left(A^c\right)^{\oplus\eps}} + (1-y)\one_{x\in A^{\oplus\eps}}}
    \bigg\rbrace,
\end{align}
where for a set $A\in\A$ its fattening by $\eps$-balls is defined as 
$A^{\oplus\eps}:=\bigcup_{x\in  A}B_\eps(x)$. As for \labelcref{eq:Pappas}, it reduces to
\begin{equation}
    \begin{split}
    \inf_{A \in \A} 
    \bigg\lbrace
    \Risk_\mathrm{prob}(A)
    :=
    \Exp{(x,y) \sim \mu}{y\one_{\Prob{{x'}\sim\distr_x}{{x'}\in A^c}>p} + 
    (1-y)
    \one_{\Prob{{x'}\sim\distr_x}{{x'}\in A}>p}}
    \bigg\rbrace,
    \end{split}
    \label{eq:Pappas_binary}
\end{equation}
where, in comparison to \labelcref{eq:AT_binary}, the fattenings $A^{\oplus\eps}$ and $(A^c)^{\oplus\eps}$ are replaced by the set of all points $x\in\X$ for which the probability that a neighboring point sampled from $\distr_x$ lies inside $A$, or $A^c$, respectively, is larger than $p$. From this definition one can see that PRL may lead to counter-intuitive classifiers like the ones observed in \cref{fig:flaw}. Indeed, taking $\distr_x$ to be the uniform measure over $B_\veps(x)$, we see that the red-shaded spiky region therein has such a small volume that the blue point with label $y=0$ lies in the set $\{x\in\X\st\Prob{x'\sim\distr_x}{[x'\in A]\leq p}\}$.
Hence, the second term in \labelcref{eq:Pappas_binary} is zero and the spike adds zero cost to the overall energy.
At an abstract level, the issue with \labelcref{eq:Pappas_binary} is that sets $\{x\in\X\st\Prob{x'\sim\distr_x}{[x'\in A^c]> p}\}$ and $\{x\in\X\st\Prob{x'\sim\distr_x}{[x'\in A^c]> p}\}$ are not fattenings, i.e.,  supersets, of $A^c$ and $A$.
Speaking the language of classification, it means that points which are incorrectly classified incur zero loss if most of their perturbations are correctly classified.

In order to obtain a model which does not exhibit this counter-intuitive behavior we modify \labelcref{eq:Pappas_binary} slightly, thereby obtaining a model which has a geometric regularization interpretation and can be embedded in a rich class of models that allows to interpolate between risk minimization and adversarial training. \red To motivate our modified PRL risk, we first observe that the original adversarial training risk $\Risk_{\mathrm{adv}}$ admits the following decomposition:
\[     \Risk_\mathrm{adv}(A)=    \Risk_\mathrm{std}(A) + \Per_{\mathrm{adv}} (A)  ,\]
where $\Per_{\mathrm{adv}} (A)$ is the non-negative functional
\[ \Per_{\mathrm{adv}} (A):=  \int_{A^c}  \sup_{x' \in B_\veps(x)} \one _{A}(x') \de \rho_0(x)
        +
        \int_{A}  \sup_{x' \in B_\veps(x)} \one _{A^c}(x') \de \rho_1(x) \]
and $\rho_i(\bullet) := \mu(\bullet\times\{i\})$, i.e., the weighted conditional distribution of the points with label $i\in\{0,1\}$. We refer the interested reader to \cite{bungert2023geometry}, where a detailed discussion on this decomposition, results on existence of minimizers for adversarial training, and a discussion on the interpretation of the functional $\Per_{\mathrm{adv}}$ as a nonlocal \textit{perimeter} are presented. The two terms in the decomposition of $\Risk_{\mathrm{adv}}$ can be interpreted as follows: 1) the term $\Risk_\mathrm{std}(A)$ is the standard risk and captures the contribution to the overall adversarial risk by misclassified points; 2) $ \Per_{\mathrm{adv}}(A)$ captures the contribution of the points that, although classified correctly by the binary classifier $\one_A$, are within distance $\veps$ from the decision boundary and thus can be perturbed by the adversary to make the classifier's output differ from the correct label. Motivated by the above decomposition for $ \Risk_\mathrm{adv}$, by substituting the suprema in the definition of $\Per_{\mathrm{adv}}$ with a softer penalty, the following modified PRL model becomes natural:
\begin{equation}\label{eq:modified_binary}
    \inf_{A \in \A} 
    \bigg\lbrace
    \ProbRisk(A)
    := \Risk_\mathrm{std}(A) + \ProbPer(A)
    \bigg\rbrace,
\end{equation}
where $\ProbPer$ is the non-negative functional 
\begin{equation}
\label{eq:ProbPer}
  \ProbPer(A):=  \int_{A^c} \one_{\Prob{{x'} \sim \distr_{x}}{{x'}\in A} > p } \de \rho_0(x)
        +
        \int_{A} \one_{\Prob{{x'} \sim \distr_{x}}{{x'}\in A^c} > p } \de \rho_1(x).
        \end{equation}
 The notation that we have chosen for the functional $\ProbPer$ suggests a perimeter interpretation, and we will justify this choice shortly. With this interpretation, it becomes clear that the modified PRL model \labelcref{eq:modified_binary} has the structure of a regularized risk minimization problem where the regularization term penalizes the size of the boundary of a set. As discussed earlier, this interpretation also holds for AT.
        
A simple computation (see \cref{prop:rewrite_max} below) allows us to rewrite $\ProbRisk$ as
\begin{equation}\label{eq:OtherFormProbR}
    \ProbRisk(A)
    =
    \Exp{(x,y) \sim \mu}{y\one_{x \in A^c \vee \Prob{{x'}\sim\distr_x}{{x'}\in A^c}>p} + 
    (1-y)
    \one_{x \in A \vee \Prob{{x'}\sim\distr_x}{{x'}\in A}>p}}.
\end{equation}
From this rewriting it is apparent that a point $x$ with label $y$ contributes to the modified PRL risk if it is either non-robust in a probabilistic sense \textit{or} if it is misclassified. For instance, if $y=0$, the loss is one when  $\Prob{x'\sim\distr_x}{x'\in A}>p$ or when $x\in A$. This second condition is not captured by the original PRL formulation.

\nc

\subsection{A larger class of PRL models for binary classification}

For theoretical and computational reasons that we will soon discuss, it is convenient to generalize problems \labelcref{eq:Pappas_binary} and \labelcref{eq:modified_binary} further. \red Precisely, given an arbitrary function $\Psi: [0,1] \rightarrow [0,1]$ we define 
  \begin{align}
  \label{eq:RiskPsi}
        \Risk_\Psi(A) := 
        \int_\X \Psi\left(\Prob{x'\sim\distr_x}{x'\in A}\right)
        \de\rho_0(x)
        +
        \int_\X \Psi\left(\Prob{x'\sim\distr_x}{x'\in A^c}\right)
        \de\rho_1(x),
    \end{align}
\nc
as well as 
\begin{align}\label{eq:ProbRiskPsi}
    \ProbRisk_\Psi(A) := 
    \Risk_\mathrm{std}(A) + \ProbPer_\Psi(A),
\end{align}
where 
\begin{align}
    \begin{split}\label{eq:ProbPerPsi}
        \ProbPer_{\Psi}(A)
        := 
        \int_{A^c} \Psi\left(\Prob{{x'} \sim \distr_{x}}{{x'}\in A}\right) \de \rho_0(x)
        +
        \int_{A}\Psi\left(\Prob{{x'} \sim \distr_{x}}{{x'}\in A^c}\right) \de \rho_1(x).
    \end{split}
\end{align}

The functionals $\Risk_\Psi$ and $\ProbRisk_\Psi$ are indeed generalizations of $\Risk_{\mathrm{prob}}$ and $\ProbRisk$, respectively. This can be seen by taking $\Psi$ to be $\Psi(t):= \one_{t >p}$, as can be easily verified. Another important choice for $\Psi$ that we will discuss extensively in the sequel is the function $\Psi_{p}(t):= \min\{ \frac{t}{p} , 1 \}$, which is the smallest concave function larger than $\one_{t >p}$. Later, in \Cref{prop:rewrite_cvar_Psi} and in \cref{sec:general_models}, we will discuss the connection between the CVaR relaxation of PRL introduced for computational convenience in \cite{robey2022probabilistically} (see more discussion in \cref{sec:computational}) and the optimization of the risks $\Risk_\Psi$ and $\ProbRisk_\Psi$ for the choice $\Psi=\Psi_p$. From a theoretical perspective, we will establish several structural properties satisfied by the functionals $\Risk_\Psi$ and $\ProbRisk_\Psi$ when $\Psi$ is concave. In particular, while we are not able to prove existence of minimizers (in a large class of sets $\A$) for the original binary problem \labelcref{eq:Pappas_binary} nor for our modification \labelcref{eq:modified_binary}, we will be able to do it for the minimization of $\Risk_\Psi$ and $\ProbRisk_\Psi$ when $\Psi$ is concave, e.g., when $\Psi= \Psi_p$. Without concavity the problems are degenerate because of the hard thresholding imposed by the constraints $\Prob{x'\sim\distr_x}{x'\in A^{(c)}}$ and we will only be able to establish existence of solutions within the enlarged family of ``soft'' classifiers. A detailed discussion on this is presented in \cref{sec:ExistenceSection} below.

\nc

\subsection{Informal statements of main results}
\label{sec:main_results}

\red 
This paper investigates three main questions related to PRL. First, we study the existence of minimizers to PRL and to modified PRL. Second, we explore different geometric interpretations of the functionals $\ProbPer_{\Psi}$. In particular, we establish two types of results that characterize these functionals as \textit{perimeters}, thereby further suggesting how the models studied in this paper for robust training of classifiers are closely related to geometric regularization methods that explicitly penalize the size of decision boundaries of classifiers. Finally, we present rigorous analysis that describes, in detail, the way in which PRL and modified PRL interpolate between adversarial training and standard risk minimization. 
\nc

\subsubsection{Existence of minimizers of PRL models}
\label{sec:ExistInformal}
In \cite{robey2022probabilistically}, an explicit solution to \labelcref{eq:Pappas} is presented for the case where the hypothesis class consists of linear classifiers 
and $\mu$ is a mixture  of normal distributions. To the best of our knowledge, existence of solutions to \labelcref{eq:Pappas} for general data distributions or hypothesis classes---even for the binary problem \labelcref{eq:Pappas_binary}---has not been proved so far.  

Our first main results assert the existence of minimizers of the risks $\ProbRisk_\Psi$ and $\Risk_\Psi$ for suitable choices of $\Psi$. We impose no restrictions on the data distribution but consider families of classifiers satisfying certain closure relations discussed precisely in later sections. 

\begin{informaltheorem}[Existence of hard classifiers]\label{formalthm:existence_hard}
    For every concave and non-increasing function $\Psi:[0,1]\to[0,1]$ there exists a solution to the problem
    \begin{align*}
        \min_{A\subset\X}\ProbRisk_\Psi(A).
    \end{align*}
\end{informaltheorem}

\red 

\begin{informaltheorem}[Existence of hard classifiers]\label{formalthm:existence_hardOrignal}
    For every concave and non-increasing function $\Psi:[0,1]\to[0,1]$ there exists a solution to the problem
    \begin{align*}
        \min_{A\subset\X}\Risk_\Psi(A).
    \end{align*}
\end{informaltheorem}
Note that these existence results do not cover the problem \labelcref{eq:modified_binary} since the function $\Psi(t) = \one_{t>p}$ is not concave. They do include, however, the CVaR model from \cref{prop:rewrite_cvar_Psi} below since this model is realized by the choice $\Psi_p(t):=\min\left\lbrace t/p,1\right\rbrace$, which is a concave and non-decreasing function. \red  We believe that our proof strategy, which takes advantage of some interesting structural properties of the functionals $\ProbRisk_\Psi$ and $\Risk_\Psi$, is of interest in its own right and is captured by our \cref{thm:meta} in \cref{sec:existence_proofs} below. Precise statements of our results can be found in \cref{sec:hard_classifiers}.

\nc 

\medskip

We also present existence results for more general, non-concave functions $\Psi$, including the choice $\Psi(t)=\one_{t>p}$ giving rise to the original PRL model and its modified version introduced here. These results, however, apply when we optimize over ``soft'' classifiers instead of ``hard'' ones, meaning we optimize over functions $u:\X\to[0,1]$ in a suitably closed hypothesis class $\mathcal{H}$ instead of characteristic functions $u = \one_A$ of sets $A$. To make sense of our statements, we first need to replace the functional $\ProbPer_\Psi(A)$ of a set $A$ by a suitable regularization functional defined according to
\begin{align}\label{eq:J_Psi}
    \begin{split}
    \ProbJ_\Psi(u) 
    &:=
    \int_\X \left(1-u(x)\right)\Psi\left(\Exp{{x'}\sim\distr_x}{u({x'})}\right)
    \de\rho_0(x)
    \\
    &\hspace{4em}
    +
    \int_\X u(x)\Psi\left(\Exp{{x'}\sim\distr_x}{1-u({x'})}\right)
    \de\rho_1(x).
    \end{split}
\end{align}
A straightforward computation reveals that $\ProbJ_\Psi(\one_A) = \ProbPer_\Psi(A)$.

\begin{informaltheorem}[Existence of soft classifiers]\label{formalthm:existence_soft}
    For every lower semicontinuous function $\Psi:[0,1]\to[0,1]$ and every hypothesis class $\mathcal{H}$ which is closed in a suitable sense there exists a solution to the problem
    \begin{align*}
        \min_{u \in \mathcal{H}}\Exp{(x,y)\sim\mu}{\abs{u(x)-y}} + \ProbJ_\Psi(u).
    \end{align*}
\end{informaltheorem}

\red 
\begin{informaltheorem}[Existence of soft classifiers]\label{formalthm:existence_soft_2}
    For every lower semicontinuous function $\Psi:[0,1]\to[0,1]$ and every hypothesis class $\mathcal{H}$ which is closed in a suitable sense there exists a solution to the problem
    \begin{align*}
        \min_{u \in \mathcal{H}} \int_\X \Psi\left(\Exp{{x'}\sim\distr_x}{u({x'})}\right)
    \de\rho_0(x)
    +
    \int_\X \Psi\left(\Exp{{x'}\sim\distr_x}{1-u({x'})}\right)
    \de\rho_1(x).
    \end{align*}
\end{informaltheorem}

Note that the objective functions in Informal Theorems 3 and 4 coincide with $\ProbRisk_\Psi$ and $\Risk_\Psi$, respectively, when $u$ is an indicator function. When $\Psi$ is non-decreasing and concave, and $\mathcal{H}$ is suitably chosen, these optimization problems are continuous and exact relaxations of $\inf_{A \in \A}\ProbRisk_\Psi$ and $\inf_{A \in \A}\Risk_\Psi$, respectively. Other hypothesis classes $\mathcal{H}$ of interest that are admissible in \cref{formalthm:existence_soft,formalthm:existence_soft_2} include functions which are parameterized by neural networks, as discussed in \cref{ex:hypothesis_classes}. 
\nc 

\subsubsection{Interpretation of $\ProbPer_\Psi$ as a perimeter}

\red 
As suggested by our discussion in earlier sections, the functional $\ProbPer_\Psi$ can be interpreted as a perimeter functional. We can justify this interpretation from two different perspectives.

First, we establish the following structural properties for the functional $\ProbPer_\Psi$ when $\Psi$ is non-decreasing and concave, the exact same assumptions needed in our existence results discussed in \cref{sec:ExistInformal}. These properties allow us to view $\ProbPer_\Psi$ as a generalized perimeter.

\begin{informaltheorem}[Submodularity]\label{formalthm:submodularity} Provided $\Psi$ is concave, non-decreasing, and $\Psi(0) =0$ the functional $\ProbPer_\Psi$ is a generalized perimeter. In particular, \[\ProbPer_\Psi(\emptyset)= \ProbPer_\Psi(\X) =0\] 
and $\ProbPer_\Psi$ is submodular, i.e., 
\begin{align*}
    \ProbPer_\Psi(A\cup B)
    +
    \ProbPer_\Psi(A\cap B)
    \leq 
    \ProbPer_\Psi(A)
    +
    \ProbPer_\Psi(B)
    \quad\forall A,B\in\A.
\end{align*}
\end{informaltheorem}

As a generalized perimeter, $\ProbPer_\Psi$ admits a convex extension (its total variation) that is denoted by $\ProbTV_\Psi$ and that we define precisely in \labelcref{eq:ProbTV_Psi}. Interestingly, $\ProbPer_\Psi$ has a second extension, the functional $\ProbJ_\Psi$ defined in \labelcref{eq:J_Psi}, which is sequentially lower-semicontinuous (although it fails to be convex) w.r.t. the weak topology where we can guarantee precompactness of minimizing sequences and is always greater than or equal to $\ProbTV_\Psi$. We use this ordering in our proof strategy for the existence of optimal hard classifiers. We also note that, while non-convex, the extension $\ProbJ_\Psi$ has a relatively simple explicit expression that makes it useful for computations, while the same is not true for $\ProbTV_\Psi$, which instead has a complicated expression that would be hard to compute with.

When $\Psi$ is more general, we can still give a perimeter interpretation to the functional $\ProbPer_\Psi$, at least in an asymptotic sense. For this purpose, we assume that $\distr_x$ localizes to a point mass at $x$.
\begin{informaltheorem}[Local asymptotics]\label{formalthm:asymptotics}
    Let $\X= \R^d$. If $\distr_x \equiv \distr_{x,\eps}$ for $\eps>0$ and $\distr_{x,\eps}$ converges to the Dirac delta $\delta_x$ in a suitable way, then the rescaled probabilistic perimeters $\tfrac{1}{\eps}\ProbPer_\Psi(A)$ of a smooth set $A$ converge to the weighted local perimeter
    \begin{align*}
        \Per(A) := \int_{\partial A} f(x,n(x))\de\mathcal{H}^{d-1},
    \end{align*}
    where $n(x)$ denotes the outer unit normal at $x\in\partial A$ and $f$ is a suitable weight function, depending on the distributions $\distr_{x,\eps}$, and $\rho_i$ for $i\in\{0,1\}$. $\mathcal{H}^{d-1}$ denotes the $(d-1)$-dimensional Hausdorff measure in $\R^d$.
\end{informaltheorem}

\subsubsection{Interpolation properties of PRL and modified PRL}

Our last main results describe the relation between adversarial training, RM, and the PRL and modified PRL models, at least in the agnostic setting where the family of admissible binary classifiers consists of all Borel measurable sets. We focus on the  energies $\ProbRisk_{\Psi_p}$ and $\Risk_{\Psi_p}$ as $p$ tends to zero or $\infty$ for the choice $\Psi_p(t)= \min\{t/p,1 \}$.

\begin{informaltheorem}[Convergence to minimum adversarial training]\label{formalthm:convergence}
    Let $\Psi_p(t):=\min\left\lbrace{t}/{p},1\right\rbrace$.  Then the minimum value of $\ProbRisk_{\Psi_p}$ converges to the minimum value of the adversarial training problem \labelcref{eq:AT_binary} as $p\to 0$. If, in addition, for every $x$ the measure $\distr_x$ dominates the data measures $\rho_0$ and $\rho_1$ restricted to the support of $\distr_x$, then minimizers of $\ProbRisk_{\Psi_p}$ converge to minimizers of \labelcref{eq:AT_binary} as $p \to 0$. 

    Analogous statements hold for the minimization of $\Risk_{\Psi_p}$.

    \end{informaltheorem}
While the convergence of the minimal risks holds under fairly general assumptions, we emphasize that the convergence of minimizers holds provided we impose additional conditions on the measures $\{ \distr_x\}_{x \in \X}$, as our examples in \cref{sec:Interpolation} illustrate. 

\medskip 

For large values of $p$ we have the following result.
\begin{informaltheorem}[Convergence to minimum standard risk]\label{formalthm:convergenceERM}
Let $\Psi_p(t):=\min\left\lbrace{t}/{p},1\right\rbrace$.  Then the minimum value of $\ProbRisk_{\Psi_p}$ converges to the minimum value of the RM problem \labelcref{eq:AT_binary} as $p\to \infty$. Also,  minimizers of $\ProbRisk_{\Psi_p}$ converge to minimizers of \labelcref{eq:AT_binary} as $p\to\infty$. 
    \end{informaltheorem}

We emphasize that the above statement holds for $\ProbRisk_{\Psi_p}$ but not for $\Risk_{\Psi_p}$. In fact, there is no value or limiting value of $p$ that makes $\Risk_{\Psi_p}$ into the standard risk functional. This means that the family of functionals $\{ \Risk_{\Psi_p} \}_{p }$ does not actually interpolate between AT and RM, while the family of functionals $\{ \ProbRisk_{\Psi_p} \}_{p }$ does. The latter model has the additional advantage of being interpretable as a geometric regularization model.

\nc

\medskip

\subsection{Related work}
\label{sec:related_work}

Adversarial training was developed in \cite{goodfellow2014explaining,madry2017towards} as an approach to produce networks that are less sensitive to adversarial attacks.
\cite{shafahi2019adversarial} reduced its computational complexity by reusing gradients from the backpropagation when training neural networks. 
\cite{wong2020fast} showed that training with noise perturbations followed by a single signed gradient ascent (FGSM) step can be on par with adversarial training while being much cheaper.
This approach was picked up and improved upon in \cite{andriushchenko2020understanding} based on gradient alignment.
Different authors also investigated test-time robustification of pretrained classifiers using randomized smoothing \cite{cohen2019certified} or geometric / gradient-based approaches \cite{schwinn2021identifying,schwinn2022improving}.
While some of the previous models use a combination of random perturbations and gradient-based adversarial attacks to robustify classifiers, \cite{robey2022probabilistically} proposed probabilistically robust learning, which is entirely based on random perturbations.
PRL aims to interpolate between clean and adversarial accuracy and enjoys the favorable sample complexity of vanilla empirical risk minimization; see also \cite{raman2023proper} for more insights on this issue.
Connections between adversarial training and local perimeter regularization as well as mean curvature flows of decision boundaries were explored in \cite{zhang2019theoretically,NGTMurrayJMLR} and recently rigorously tied in \cite{bungert2024gamma,bungert2024meancurvature}.
Furthermore, in \cite{bungert2024gamma} it was proved that solutions to adversarial training converge to distinguished minimizers of standard risk minimization as $\eps\to 0$ which was later quantified and extended to the probablistically robust setting in \cite{morris2024uniform}.
An interpretation of adversarial training as gradient flow is given in \cite{weigand2024adversarialflows}.

Our work is in line with a series of papers \cite{pydi2021many,
awasthi2021existence, awasthi2021extended,
frank2022consistency, frank2022existence, bungert2023geometry,trillos2023existence} that explore the existence of solutions to adversarial training problems in different settings.
These existence proofs involve dealing with different kinds of measurability issues, depending on whether open or closed balls $B_\eps(x)$ are used in the attack model.
For open balls one can work with the Borel $\sigma$-algebra $\A=\B(\X)$ \cite{bungert2023geometry}, whereas closed balls require the use of the universal $\sigma$-algebra to make sure that $A^{\oplus\eps}$ is measurable \cite{pydi2021many,awasthi2021existence,awasthi2021extended}.
Recently, these results were improved in \cite{trillos2023existence} where it was proved that, for the case of multi-class classification, even for the closed ball model Borel measurable classifiers exist, although they are not necessarily indicator functions of measurable sets. Moreover, \cite{trillos2023existence} showed that for all but countably many values of the adversarial budget $\eps>0$ the open and the closed ball models have the same minimal value.

\subsection{Outline}

\red 
The rest of the paper is organized as follows. In Section 2 we make the setting for our variational problems precise and then present a series of results on the existence of minimizers of these problems. In particular, in  \cref{sec:hard_classifiers,sec:soft_classifiers} we make precise our \cref{formalthm:existence_hard,formalthm:existence_hardOrignal}, and in \cref{sec:existence_proofs} present their proofs. \cref{sec:Perimeters} is devoted to the study of the functional $\ProbPer_\Psi$ and its interpretation as a perimeter. There, we make the \cref{formalthm:submodularity,formalthm:asymptotics} precise and present their proofs. \cref{sec:Interpolation} is devoted to the interpolation properties of the different PRL models. There we make our \cref{formalthm:convergence,formalthm:convergenceERM} precise using the notion of $\Gamma$-convergence. In \cref{sec:general_models}, we introduce a modified PRL model for general supervised learning problems beyond binary classification and discuss the connection between the CVar relaxation of PRL in \cite{robey2022probabilistically} and our geometric framework. In \cref{sec:numerics} we present some numerical results of an implementation of our models in real data settings. We wrap up the paper in \cref{sec:Conclusions} with some conclusions and perspectives for future research.
\nc

\section{Existence of probabilistically robust classifiers}
\label{sec:ExistenceSection}

\subsection{Preliminaries}

In this section we discuss a few other reformulations for the energies discussed in the introduction. These reformulations provide some additional context and connections to the literature that will later be used in our discussion.

\red 

To start, we first present a reformulation of the probabilistic risk $\ProbRisk$ as the expected maximum of the sample-wise standard risk and the probabilistically robust risk from \cite{robey2022probabilistically}.

\begin{proposition}\label{prop:rewrite_max}
    For the $0$-$1$ loss $\ell(\tilde y,y) := \one_{\tilde y\neq y}$ we can rewrite the probabilistic risk $\ProbRisk$ as in \labelcref{eq:OtherFormProbR}. In that same setting, $\ProbRisk$ can also be written as the expectation of a sample-wise maximum of the standard loss and the loss suggested in \cite{robey2022probabilistically}, that is:
    \begin{align}
    \label{eq:OtherFormProbRMax}
        \ProbRisk(A) = \Exp{(x,y)\sim\mu}{\max\left\lbrace
        \ell(\one_A(x),y),
        p\text{-}\esssup_{x'\sim\distr_x}\ell(\one_A(x'),y)
        \right\rbrace}.
    \end{align}

\end{proposition}
\begin{remark}
     This result will be key for generalizing $\ProbRisk$ to multi-class classification as well as general hypothesis classes and loss functions. See \cref{sec:general_models}
\end{remark}


\begin{remark}
In its original expression, $\ProbRisk$ is written as the sum of the standard risk and a nonlocal perimeter functional. This \emph{probabilistic perimeter}, which we will discuss in more detail in \cref{sec:Perimeters}, can be rewritten as:
\begin{align*}
    \begin{split}
    \ProbPer(A) 
    &=
    \rho_0\left(
    \Set{x \in A^c \st \Prob{{x'} \sim \distr_{x}}{{x'}\in A}>p}
    \right)
    \\
    &\hspace{4em}
    +
    \rho_1\left(
    \Set{x \in A \st \Prob{{x'} \sim \distr_{x}}{{x'}\in A^c}>p}
    \right).
    \end{split}
\end{align*}
Note that $\ProbPer(A)$ counts the proportion of correctly classified points $x$ for which more than $100\times p\,\%$ of their neighbors, sampled from $\distr_x$, constitute an attack. From this it should be intuitively clear that $\ProbPer(A)$ is a boundary term. 
\end{remark}

\begin{remark}
    The interpretation of the statement of this proposition in the light of \Cref{fig:spike_solution} is clear: Only if a point $x$ is correctly classified---meaning $\one_{\one_A(x)\neq y}=0$---the probabilistically robust regularization kicks in through the second term in the maximum.
Points which are incorrectly classified will always be penalized even if most attacks correct the label, i.e., if $\one_{\Prob{{x'}\sim\distr_{x}}{\one_A({x'})\neq y}>p}=0$. 
Thus, minimizing $\ProbRisk$ instead of $\Risk_\mathrm{prob}$ corrects the pathology of the original PRL formulation.
\end{remark}

\begin{proof}[Proof of \Cref{prop:rewrite_max}]
    
    Disintegrating the risk functional $\Risk_{\mathrm{std}}$ we can write 
    \[  \Risk_{\mathrm{std}}(A) =
        \int_\X 
        \one_{x\in A}
        \de \rho_0(x)
        +
        \one_{x\in A^c}
        \de\rho_1(x). \]
   Also,
    \begin{align*}
    \begin{split}
    \ProbPer(A)
    &=
    \int_\X 
    \one_{x\in A\,\vee\,\Prob{{x'} \sim \distr_{x}}{{x'}\in A}>p} - \one_{x\in A} \de \rho_0(x)
    \\
    &\hspace{10em}
    +
    \int_\X 
    \one_{x\in A^c\,\vee\,\Prob{{x'} \sim \distr_{x}}{{x'}\in A^c}>p} - \one_{x\in A^c} \de \rho_1(x).
    \end{split}
\end{align*}
Adding the above two identities we deduce \labelcref{eq:OtherFormProbR}. Now that \labelcref{eq:OtherFormProbR} has been established we see that   
 \begin{align*}
        &\phantom{{}={}}
        \ProbRisk(A) 
        \\
        &=
        \int_\X 
        \one_{x\in A\,\vee\,\Prob{{x'} \sim \distr_{x}}{{x'}\in A}>p}
        \de \rho_0(x)
        +
        \one_{x\in A^c\,\vee\,\Prob{{x'} \sim \distr_{x}}{{x'}\in A^c}>p}
        \de\rho_1(x)
        \\
        &=
        \int_\X 
        \max
        \left\lbrace 
        \one_{x\in A},
        \one_{
        \Prob{{x'}\sim\distr_x}
        {{x'}\in A} > p
        }
        \right\rbrace 
        \de\rho_0(x)
        +
        \int_\X 
        \max
        \left\lbrace 
        \one_{x\in A^c},
        \one_{
        \Prob{{x'}\sim\distr_x}
        {{x'}\in A^c} > p
        }
        \right\rbrace
        \de\rho_1(x),
    \end{align*}
    where we used the fact that the indicator function of the union of two sets equals the maximum of the two indicator functions. \labelcref{eq:OtherFormProbRMax} follows immediately.
\end{proof}
\nc

\red 
In the same spirit as in \cref{prop:rewrite_max}, we can rewrite $\ProbRisk_\Psi$ for general $\Psi$ as follows. 
\nc 

\begin{proposition}\label{prop:rewrite_max_Psi}
    For the $0$-$1$ loss $\ell(\tilde y,y) := \one_{\tilde y\neq y}$ and for any function $\Psi:[0,1]\to[0,1]$ we can rewrite the probabilistic risk $\ProbRisk_\Psi$ as sample-wise maximum in the following way:
    \begin{align*}
        \ProbRisk_\Psi(A) 
        =
        \Exp{(x,y)\sim\mu}{
        \max
        \left\lbrace
        \ell(\one_A(x),y),
        \Psi\left(\Prob{{x'}\sim\distr_{x}}{\ell(\one_A(x'),y)}\right)
        \right\rbrace}.
    \end{align*}
\end{proposition}
Besides the choice $\Psi(t) = \one_{t>p}$ another interesting choice is the smallest concave function which lies above it, i.e., $\Psi(t) = \min\left\lbrace t/p,1\right\rbrace$.
For this choice we can connect the probabilistic risk with the conditional value at risk (CVaR) \cite{rockafellar2000optimization} which was suggested in \cite{robey2022probabilistically} as a computationally feasible replacement for the $p\text{-}\esssup$ operator in problem \labelcref{eq:Pappas}.
Its precise definition will be important later and can be found in \labelcref{eq:CVaR} in \cref{sec:general_models}.
\begin{proposition}\label{prop:rewrite_cvar_Psi}
    For the $0$-$1$ loss $\ell(\tilde y,y) := \one_{\tilde y\neq y}$ and for $\Psi_p(t) := \min\left\lbrace t/p,1\right\rbrace$ with $p \in (0,1)$ it holds that 
    \begin{align*}
        \ProbRisk_{\Psi_p}(A) 
        =
        \Exp{(x,y)\sim\mu}{
        \max
        \left\lbrace
        \ell(\one_A(x),y),
        \operatorname{CVaR}_p\left(\ell(\one_A(\cdot),y);\distr_x\right)
        \right\rbrace}.
    \end{align*}
\end{proposition}

\nc


We proceed to prove \cref{prop:rewrite_max_Psi}.

\begin{proof}[Proof of \Cref{prop:rewrite_max_Psi}]
The proof is similar to that of \Cref{prop:rewrite_max} after noting that for $\Psi: [0,1] \to [0,1]$
\begin{equation*}
    \one_{x\in A} + \one_{x\in A^c} \Psi\left(\Prob{{x'} \sim \distr_{x}}{{x'}\in A}\right) = \max\left\{\one_{x\in A}, \Psi\left(\Prob{{x'} \sim \distr_{x}}{{x'}\in A}\right)\right\}
\end{equation*}
which can easily be shown by checking cases. Then:
\begin{align*}
        \ProbRisk_\Psi(A) 
        &= 
        \int_\X \one_{x\in A}\de\rho_0(x)
        +
        \int_\X \one_{x\in A^c}\de\rho_1(x)
        \\
        &\hspace{4em}
        +
        \int_\X \one_{x\in A^c} \Psi\left(\Prob{{x'} \sim \distr_{x}}{{x'}\in A}\right) \de \rho_0(x)
        +
        \int_\X \one_{x\in A}\Psi\left(\Prob{{x'} \sim \distr_{x}}{{x'}\in A^c}\right) \de \rho_1(x) \\
        & = \int_\X\left[\one_{x\in A} + \one_{x\in A^c} \Psi\left(\Prob{{x'} \sim \distr_{x}}{{x'}\in A}\right) \right]\de\rho_0(x) \\
        & \hspace{4em} + \int_\X\left[\one_{x\in A^c} + \one_{x\in A}\Psi\left(\Prob{{x'} \sim \distr_{x}}{{x'}\in A^c}\right) \right]\de\rho_1(x) \\
        &= \int_\X\max\left\{\one_{x\in A}, \Psi\left(\Prob{{x'} \sim \distr_{x}}{{x'}\in A}\right)\right\}\de \rho_0(x) \\
        & \hspace{4em} + \int_\X\max\left\{\one_{x\in A^{c}}, \Psi\left(\Prob{{x'} \sim \distr_{x}}{{x'}\in A^{c}}\right)\right\}\de \rho_1(x).
\end{align*}
\end{proof}
We postpone the proof of \cref{prop:rewrite_cvar_Psi}, which discusses another reformulation of $\ProbRisk_\Psi$ in terms of the so-called conditional value at risk (CVaR), to \cref{sec:general_models}, since \cref{prop:rewrite_cvar_Psi} will not be used in the remainder of this section.

\subsection{Model family for hard classifiers}
\label{sec:hard_classifiers}

We now begin the discussion of existence of solutions to the minimization problems associated with the risks $\Risk_\Psi$ and $\ProbRisk_\Psi$ defined in \labelcref{eq:RiskPsi} and \labelcref{eq:ProbRiskPsi}. As was discussed in the introduction, these families of risks include $\Risk_{\mathrm{prob}}$ and $\ProbRisk$ (for $\Psi(t):=\one_{t>p}$) but also a relaxation of the adversarial risk $\Risk_\mathrm{adv}$ (for $\Psi(t):=\one_{t>0}$).

We start our discussion by making our setting precise.
\begin{assumption}\label{ass:main_assumption}
We let $\X$ be a set and $\A\subset 2^\X$ be a $\sigma$-algebra.
We assume that:
\begin{itemize}
    \item $(\X\times\Y,\A\otimes 2^{\{0,1\}},\mu)$ is a probability space;
    \item $(\X,\A,\rho)$ is a probability space, where we define $\rho(\bullet) := \mu(\bullet\times\{0,1\})$;
    \item $\{\distr_x\}_{x\in\X}$ is a family such that $(\X,\A,\distr_x)$ is a probability space for $\rho$-almost every $x\in\X$.
\end{itemize}
\end{assumption}
We also define the probability measures 
\begin{align}
    \label{eq:rho}
    \rho &:= \rho_0 + \rho_1,\\
    \label{eq:nu}
    \nu(A) &:= \frac12\int_{\X}\distr_{x}(A)\de\rho(x)
    +
    \frac12
    \rho(A)
    ,\qquad A\in\A.
\end{align} 
\nc 
The measure $\rho$ equals the first marginal of $\mu$ and models the distribution of all data, irrespective of their label.
The first summand of the measure $\nu$ is the convolution of $\rho$ with the family of probability measures $\{\distr_{x}\}_{x\in\X}$.

By construction we have the following two important properties:
\begin{align}
    \label{eq:ac_nu}
    \nu(A) &= 0 \implies 
    \Big[
    \rho(A) = 0 
    \quad 
    \text{and}
    \quad
    \distr_x(A) = 0\text{ for $\rho$-almost every $x\in\X$}
    \Big],\\
    \label{eq:ac_rho}
    \rho(A) &= 0 \implies 
    \Big[
    \rho_0(A) = 0 
    \quad 
    \text{and}
    \quad
    \rho_1(A) = 0
    \Big].
\end{align}
A simple example for $\X=\R^d$ is $\rho := \frac1N\sum_{i=1}^N\delta_{x_i}$ and $\distr_x := \operatorname{Unif}(B_\eps(x))$ in which case
\begin{align*}
    \nu = \frac1{2N}\sum_{i=1}^N \Big( \operatorname{Unif}(B_\eps(x_i)) + \delta_{x_i}\Big)
\end{align*}
is a sum of absolutely continuous measures on balls centered at $x_i$ and the empirical measure of the points $x_i$.

Using the measure $\nu$ we will work with the weak-* topology of $L^\infty(\X;\nu)$ which is the dual space of $L^1(\X;\nu)$ since $\nu$ is \emph{a fortiori} a $\sigma$-finite measure \cite[IV.8.3, Theorem 5]{dunford1958linear}.
\begin{definition}\label{def:weak_star_nu}
Under \Cref{ass:main_assumption} we say that a sequence of functions $(u_n)_{n\in\N}\subset L^\infty(\X;\nu)$ converges to $u\in L^\infty(\X;\nu)$ in the weak-* sense (written $u_n\wsto u$) as $n\to\infty$ if
\begin{align}
    \lim_{n\to\infty}\int_\X u_n \varphi \de\nu = 
    \int_\X u \varphi \de\nu\qquad \forall \varphi \in L^1(\X;\nu).
\end{align}
\end{definition}
We will use this topology 
to prove existence of minimizers of the probabilistic risks $\Risk_\Psi$ and $\ProbRisk_\Psi$. Furthermore, we use this same topology for proving that our modified PRL using the function $\Psi_p$ converges to adversarial training as $p\to 0$, at least when the family of measures $\{ \distr_x \}_x$ satisfies a suitable assumption; see the discussion in \cref{sec:Interpolation}.

The following theorem establishes existence of minimizers of the risk $\ProbRisk_\Psi$ for concave and non-decreasing functions $\Psi$.


{Interestingly, we establish the existence of hard classifiers using a novel abstract strategy, rather than the usual direct method in the calculus of variations. We take this approach as it is rather nontrivial to establish lower semicontinuity for an extension of $\ProbPer_\Psi$ to soft classifiers that behaves well with respect to thresholding operations. Instead, we introduce two different relaxations; one which is lower semicontinuous, and one which is defined through a coarea formula and hence behaves well with respect to thresholding; our result then follows from an ordering relation on the two relaxations if $\Psi$ is concave and non-decreasing.}

We arrive at the rigorous versions of \cref{formalthm:existence_hard} and \cref{formalthm:existence_hardOrignal}.
\begin{theorem}\label{thm:existence_Psi}
    Suppose $\Psi:[0,1]\to[0,1]$ is concave and non-decreasing, and that \cref{ass:main_assumption} holds. Then, there exists a solution to the problem
    \begin{align}
        \inf_{A\in\A} \ProbRisk_\Psi(A). 
        \label{eq:Prob-Psi-Minimization}
    \end{align}
\end{theorem}

\begin{theorem}\label{thm:existence_PsiOriginal}
    Suppose $\Psi:[0,1]\to[0,1]$ is concave and non-decreasing, and that \cref{ass:main_assumption} holds. Then, there exists a solution to the problem
    \begin{align}
        \inf_{A\in\A} \Risk_\Psi(A). 
        \label{eq:Prob-Psi-MinimizationOriginal}
    \end{align}
\end{theorem}

The proofs of \cref{thm:existence_Psi,thm:existence_PsiOriginal} are provided in \cref{sec:existence_proofs}. The reason why an existence proof for \labelcref{eq:Prob-Psi-Minimization} (or \labelcref{eq:Prob-Psi-MinimizationOriginal}) for non-concave functions $\Psi$ is currently not available is illustrated in the following example. 
\begin{example}[Homogenizing solutions]\label{ex:homogenization}
    \begin{figure}[htb]
        \centering
        \begin{subfigure}{0.48\textwidth}
        \fbox{\includegraphics[width=0.95\textwidth,trim=3cm 1.9cm 3cm 1cm,clip]{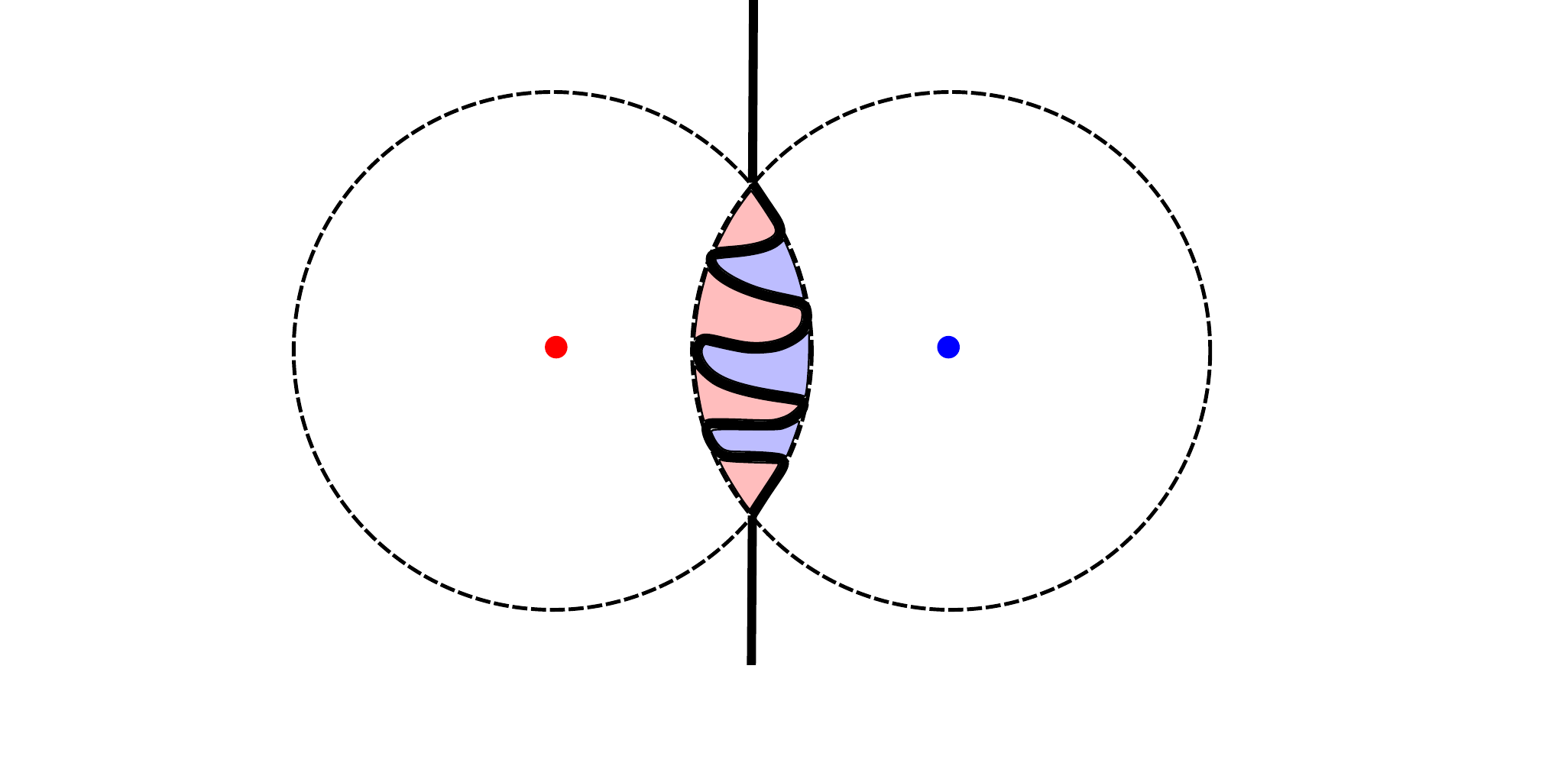}}
        \caption{Homogenizing minimizer for $\Psi(t):=\one_{t>p}$.
        \label{fig:homogenized}}
        \end{subfigure}
        \hfill
        \begin{subfigure}{0.48\textwidth}
        \fbox{\includegraphics[width=0.95\textwidth,trim=3cm 1.9cm 3cm 1cm,clip]{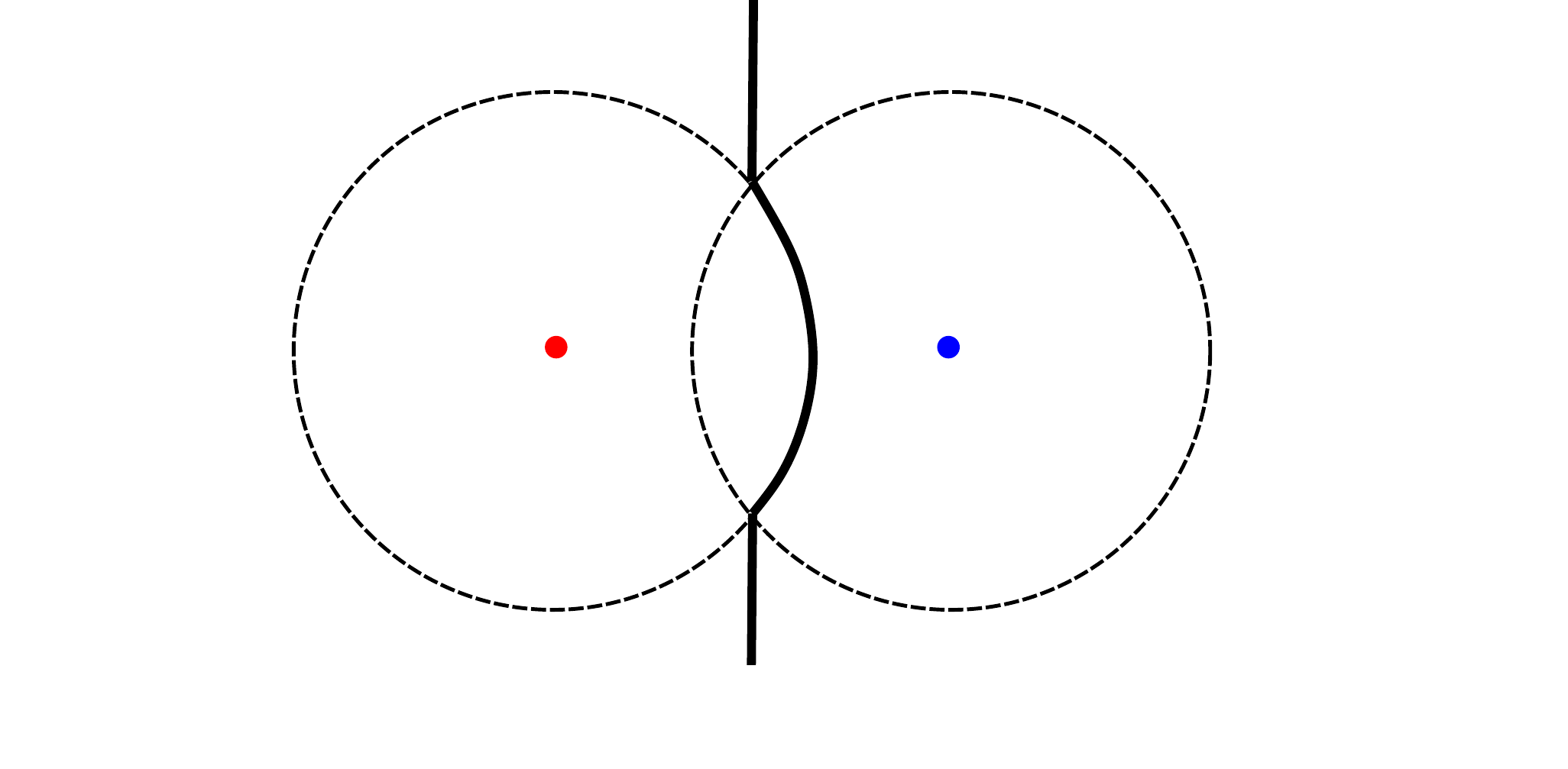}}
        \caption{Superlevel set which is not a  minimizer.
        \label{fig:AT}}
    \end{subfigure}
        \caption{Non-compactness of minimizing sequences.}
        \label{fig:enter-label}
    \end{figure}
    In this example we consider the situation of two data points (red and blue) which are so close that the $\eps$-balls around them intersect. 
    Furthermore, we assume that $p\in(0,1)$ is $0.9$ times the ratio of the volume of the intersection to the volume of one $\eps$-ball. 
    If we pick $\Psi_p=\one_{t>p}$, there exist infinitely many minimizers of \labelcref{eq:Prob-Psi-Minimization} with oscillatory boundaries in the intersection. One such minimizer is depicted in \cref{fig:homogenized}.
    The minimizers can be arranged such that the ratios of the red or the blue volume to the volume of the intersection are in the interval $[0.4,0.6]$. 
    Consequently, all such sets have zero loss and are global minimizers.
    However, the set of such minimizers is not compact and has an accumulation point which is not a characteristic function, namely the function $u:\X\to[0,1]$ with values $u=1$ on the red ball without the blue ball, $u=0$ on the blue ball without the red ball, and $u=0.5$ on the intersection of the two balls.
    Hence, minimizing sequences of \labelcref{eq:Prob-Psi-Minimization} are not compact.
    The issue of non-compactness of minimizing sequences can in fact also happen for concave functions $\Psi$ which is why in the proof of \cref{thm:existence_Psi} we show that almost every superlevel set $\Set{u>t}$ of a function $u$ which is an accumulation point of a minimizing sequences is a minimizer of \labelcref{eq:Prob-Psi-Minimization}. 
    However, not even this strategy works in the non-concave case since the superlevel set $\{u>t\}$ \emph{for all} $t\in[0,1/2]$ is given by the set in \cref{fig:AT} which is not a minimizer of PRL.
\end{example}


\subsection{Model family for soft classifiers}
\label{sec:soft_classifiers}

Note that, besides the reasons given in \cref{ex:homogenization}, an existence result similar to \cref{thm:existence_Psi} for the non-concave function $\Psi(t) = \one_{t>p}$ is currently not available since our relaxation techniques rely on concavity of $\Psi$.
However, in this section we shall state an existence theorem for ``soft classifiers'' which is valid for very general functions $\Psi$, including $\Psi(t)=\one_{t>p}$, since no relaxation is needed.

Such soft classifiers are particularly relevant since they include neural network based models with \texttt{Softmax} activation in the last layer which are used in practice.
A suitable regularization functional for soft classifiers is given by $\ProbJ_\Psi$ defined in \labelcref{eq:J_Psi}.
For convenience we repeat its definition.
Given an $\A$-measurable function $u:\X\to[0,1]$ we let
\begin{align*}
    \begin{split}
    \ProbJ_\Psi(u) 
    &:=
    \int_\X \left(1-u(x)\right)\Psi\left(\Exp{{x'}\sim\distr_x}{u({x'})}\right)
    \de\rho_0(x)
    \\
    &\hspace{4em}
    +
    \int_\X u(x)\Psi\left(\Exp{{x'}\sim\distr_x}{1-u({x'})}\right)
    \de\rho_1(x)
    \end{split}
\end{align*}
which satisfies $\ProbJ_\Psi(\one_A) = \ProbPer_\Psi(A)$ for every choice of $\Psi$. 
Hence, it is a natural generalization of the perimeter to soft classifiers and one could call $\ProbJ_\Psi$ a total variation.
However, it is neither positively homogeneous nor convex so this name would be misleading.
Instead, for the proof of \cref{thm:existence_Psi} we shall also work with the \textit{total variation functional} $\ProbTV_\Psi$ defined as
\begin{align}\label{eq:ProbTV_Psi}
    \ProbTV_\Psi(u) := \int_0^1\ProbPer_\Psi(\{u>t\})\de t,
\end{align}
for all $\A$-measurable functions $u: \X \rightarrow [0,1]$.
Note that also the total variation satisfies $\ProbTV_\Psi(\one_A) = \ProbPer_\Psi(A)$.
Both $\ProbJ_\Psi$ and $\ProbTV_\Psi$ are of paramount importance for the proof of \cref{thm:existence_Psi}.
  
The next theorem is the rigorous version of \cref{formalthm:existence_soft} and asserts existence of soft classifiers for the regularized risk minimization using $\ProbJ_\Psi$ for very general functions $\Psi$ and hypothesis classes $\mathcal{H}$.
It is sufficient that $\Psi$ is lower semicontinuous which is satisfied by every continuous function and also by $\Psi(t)=\one_{t>p}$ for $p\in[0,1]$.
Furthermore, the existence theorem is valid for all hypotheses classes which are closed in the weak-* topology of $L^\infty(\X;\nu)$.

\begin{theorem}[Existence of soft classifiers for modified PRL]\label{thm:existence_soft}
    Under \cref{ass:main_assumption}, for every lower semicontinuous function $\Psi:[0,1]\to[0,1]$, and whenever $\mathcal{H}$ is a weak-* closed hypothesis class of $\A$-measurable functions $u:\X\to[0,1]$ in the sense of \cref{def:weak_star_nu}, there exists a solution to the problem
    \begin{align}\label{eq:relaxed_problem}
        \inf_{u\in\mathcal{H}} \Exp{(x,y)\sim\mu}{\abs{u(x)-y}}
        +
        \ProbJ_\Psi(u).
    \end{align}
\end{theorem}

\begin{example}[Hypothesis classes]\label{ex:hypothesis_classes}
    The hypothesis class consisting of measurable sets, i.e., $\mathcal{H}=\Set{\one_A\st A\in\mathcal{A}}$ is not weak-* closed which is why \cref{thm:existence_Psi} needs stronger assumptions on $\Psi$ than \cref{thm:existence_soft}. 
    Let us instead consider three interesting hypothesis classes of weak-* closed (in fact, even compact) classifiers for which \cref{thm:existence_soft} applies.
    \begin{enumerate}
        \item The simplest such class $\mathcal{H}$ is the class of \emph{all} $\A$-measurable soft classifiers $u:\X\to[0,1]$ which could be referred to as \emph{agnostic} classifiers since they are not parametrized.
        This class is a bounded subset of $L^\infty(\X;\nu)$ and therefore, by the Banach--Alaoglu theorem, it is weak-* compact.
        \item An example with more practical relevance is the class of (feedforward or residual) neural networks defined on the unit cube $\X:=[-1,1]^d$ with uniformly bounded parameters
        \begin{align*}
            \mathcal{H}
            :=
            \Big\{
            \Phi_L \circ \dots \circ \Phi_1
            :[-1,1]^d\to[0,1]
            \st
            &
            \Phi_l(\bullet) = A_l \bullet + \sigma_l(W_l \bullet + b_l),
            \\
            &
            \norm{(A_l, W_l,b_l)} \leq C \; \forall l \in \{1,\dots,L\}
            \Big\},
        \end{align*}
        where we assume that the activations $\sigma_l:\R\to\R$ are continuous.
        Note that the boundedness of the weights cannot be relaxed.
        To see this, consider the (very simplistic) neural network $u_n(x) = \tanh(w_n x)$ for $x\in[-1,1]$ and $w_n\in\R$.
        For $w_n\to\infty$ it is easy to see that $u_n$ converges to $u(x):=\operatorname{sign}(x)$ which does not lie in the same hypothesis class.       

        To argue why neural networks with bounded parameters and continuous activation functions are weak-* compact, let $(u_n)_{n\in\N}\subset\mathcal{H}$ be a sequence. Thanks to finite-dimensional compactness a subsequence of the associated parameters converge to some limiting parameters.
        The continuity of the activations implies that the associated neural networks converge (uniformly in the space of continuous functions on the unit cube $[-1,1]^d$) to a limiting neural network $u\in\mathcal{H}$.
        In particular, the convergence is true in the weak-* sense, which shows that $\mathcal{H}$ is weak-* compact.
        \item Finally, one can also consider the class of hard linear classifiers on $\R^d$.
        Letting $\theta(t) := \one_{t>0}$ denote the Heaviside function, this class is given by
        \begin{align*}
            \mathcal{H}
            :=
            \Set{\theta(w\cdot x + b)
            \st 
            w \in \R^d,\;\abs{w}=1,\;b\in[-\infty,\infty]},
        \end{align*}
        where one interprets $u(x):=\theta(w\cdot x+b)$ as $u\equiv 1$ if $b=\infty$ and $u\equiv 0$ if $b=-\infty$.
        If the distributions $\rho_0$, $\rho_1$, and $\distr_x$ are are such that $\nu$ defined in \labelcref{eq:nu} has a density with respect to the Lebesgue measure, then $\mathcal H$ has the desired closedness property.
        A sufficient condition for this to hold is that $\rho_0$, $\rho_1$, and $\distr_x$ have densities with respect to the Lebesgue measure.
        
        Let us prove the closedness under the above conditions. 
        If $(u_n)_{n\in\N}\subset\mathcal{H}$ is a sequence of linear classifiers, thanks to finite-dimensional compactness a subsequence (which we do not relabel) of the associated parameters $(w_n,b_n)$ converges to $w\in\R^d$ with $\abs{w}=1$ and $b\in[-\infty,\infty]$.
        For simplicity we only consider the case where $b\neq\pm\infty$. 
        In this case one can define the half-spaces
        \begin{align*}
            A_n &:= \Set{x\in\R^d \st w_n\cdot x + b > 0},
            \\
            A &:= \Set{x\in\R^d \st w\cdot x + b > 0},
        \end{align*}
        upon which $u_n$ and $u$ are supported.
        Then for any $\phi \in L^1(\R^d;\nu)$ it holds
        \begin{align*}
            \abs{
            \int_{\R^d} (u_n - u) \phi \de\nu
            }
            =
            \abs{
            \int_{A_n} \phi \de\nu
            -
            \int_{A} \phi \de\nu
            }
            \leq 
            \int_{A_n \triangle A}
            \abs{\phi}
            \de\nu
        \end{align*}
        where we used the symmetric difference $A_n \triangle A := \left(A_n \setminus A\right) \cup \left(A\setminus A_n\right)$.
        Note that this set is either a double cone (if $w_n\neq w$) or a strip of width $\abs{b_n-b}$ (if $w_n=w$).
    
        Since $\phi \in L^1(\R^d;\nu)$ and $\nu$ is a probability measure, for every $\eps>0$ there exists a compact set $K\subset\R^d$ such that $\int_{\R^d\setminus K}\abs{\phi}\de\nu<\eps$.
        Using this, we can compute
        \begin{align*}
            \abs{
            \int_{\R^d} (u_n - u) \phi \de\nu
            }
            \leq 
            \int_{A_n \triangle A}
            \abs{\phi}
            \de\nu
            \leq 
            \int_{(A_n\triangle A) \cap K}
            \abs{\phi}\de\nu 
            +
            \eps.
        \end{align*}
        Using that $\nu$ has a density with respect to the Lebesgue measure $\mathcal{L}^d$ and using also that $\mathcal{L}^d(A_n\triangle A \cap K)\to 0$ as $n\to\infty$, we obtain
        \begin{align*}
            \lim_{n\to\infty}
            \abs{
            \int_{\R^d} (u_n - u) \phi \de\nu
            }
            \leq 
            \eps 
        \end{align*}
        and since $\eps>0$ was arbitrary we get
        \begin{align*}
            \lim_{n\to\infty}
            \int_{\R^d} (u_n - u) \phi \de\nu
            =0,
        \end{align*}
        which implies the weak-* convergence of $u_n$ to $u\in\mathcal{H}$ and hence the weak-* compactness of $\mathcal{H}$.
    
        Note that for general measures $\nu$ the above argument fails.
        For instance,the sequence of linear classifiers $u_n(x) = \one_{x_1>-1/n}$ has the natural limit $u(x) = \one_{x_1>0}$. 
        However, if $\nu = \delta_0$ then $\int_{\R^d}u_n\de\nu=1$ for all $n\in\N$ but $\int_{\R^d}u\de\nu=0$, meaning that $u$ is not the weak-* limit of $u_n$.
    \end{enumerate}
\end{example}

\red  
Similarly, we can state a precise version of \Cref{formalthm:existence_soft_2}. For this, it is convenient to introduce the probability measure $\sigma$ defined via
    \begin{align}\label{eq:sigma}
        \sigma(A) := \int_\X \distr_x(A)\de\rho(x),
    \end{align}
  which clearly satisfies
    \red 
\begin{equation}
     \label{eq:ac_sigma}
    \sigma(A) = 0 \implies \distr_x(A) = 0\text{ for $\rho$-almost every $x\in\X$}.    
\end{equation}

\begin{theorem}[Existence of soft classifiers for original PRL]\label{thm:existence_soft_2}
   Under \cref{ass:main_assumption}, for every lower semicontinuous function $\Psi:[0,1]\to[0,1]$, and whenever $\mathcal{H}$ is a hypothesis class of $\A$-measurable functions $u:\X\to[0,1]$ that is closed in the weak* topology of $L^\infty(\X;\sigma)$, there exists a solution to the problem
    \begin{align}\label{eq:relaxed_problem_2}
        \inf_{u\in\mathcal{H}}  \int_\X \Psi\left(\Exp{x'\sim\distr_x}{u(x')}\right)
        \de\rho_0(x)
        +
        \int_\X \Psi\left(\Exp{x'\sim\distr_x}{1-u(x')}\right)
        \de\rho_1(x).
    \end{align}
\end{theorem}

\nc

\subsection{Existence proofs}
\label{sec:existence_proofs}

In this section we will prove all theorems stated in the previous section using the direct method of calculus of variations together with new relaxation arguments. For this we first outline our proof strategy in a very abstract way by using the quadruple
\begin{align}
    \mathfrak Q := (R,\,S,\,T,\,\mathcal{T}),
\end{align}
where $R:\A\to[0,\infty]$ is a functional defined on $\A$-measurable sets, $S,T:\mathcal{U} \to [0,\infty]$ are proper functionals defined on the unit ball $\mathcal{U}:=\{u:\X\to[0,1]\;\A\text{-measurable}\}$ of the $\A$-measurable functions, and $\mathcal{T}$ is a topology on the unit ball. 
The key ingredients and properties of $\mathfrak Q$ are the following:
\begin{enumerate}[(D\theenumi)]
    \item \emph{Compactness}: $\mathcal{U}$ is compact with respect to $\mathcal{T}$;
    \label{des:compact}
    \item \emph{Lower semicontinuity}: For all sequences of measurable functions $(u_n)_{n\in\N}\subset\mathcal{U}$ converging to $u\in\mathcal{U}$ in the topology $\mathcal{T}$ it holds 
    \begin{align*}
        S(u)&\leq\liminf_{n\to\infty}S(u_n);
    \end{align*}
    \label{des:LSC}
    \item \emph{Consistency}: It holds for all $A\in\A$ that $$R(A)=S(\one_A).$$
    \label{des:consistency}
    \item \emph{Coarea formula}: For all $u\in\mathcal{U}$ it holds $$T(u) = \int_0^1 R(\{u\geq t\})\de t,$$
    and in particular $R(A)=T(\one_A)$ for all $A\in\A$;
    \label{des:coarea}
    \item \emph{Ordering}: For all $u\in\mathcal{U}$ it holds $$T(u)\leq S(u).$$
    \label{des:order}
\end{enumerate}
\crefname{enumi}{desideratum}{desiderata}
Under these conditions we can prove the following meta-theorem.
\begin{metatheorem}\label{thm:meta}
    Assuming \cref{des:compact,des:LSC} there exists
    \begin{align}\label{eq:meta_problem_soft}
        u \in \argmin_{u\in\mathcal{U}}S(u).
    \end{align}
    Assuming \cref{des:compact,des:LSC,des:consistency,des:coarea,des:order}, for Lebesgue almost every $t\in[0,1]$ the set $A_t := \Set{u\geq t}$ satisfies
    \begin{align}\label{eq:meta_problem_hard}
        A_t \in \argmin_{A\in\A}R(A).
    \end{align}
\end{metatheorem}
\begin{proof}
Since $S$ is proper, there exists a minimizing sequence $\{u_n\}_{n \in \N}\subset\mathcal{U}$ for \labelcref{eq:meta_problem_soft}. 
By assumption, $\{u_n\}_{n \in \N}$ is precompact in the topology $\mathcal{T}$. Therefore, up to a subsequence that we do not relabel, we can assume that $u_n \rightarrow_{\mathcal{T}} u$ for some $u: \X \rightarrow [0,1]$ which solves \labelcref{eq:meta_problem_soft} by lower semicontinuity of $S$.

Defining $A_t := \Set{u\geq t}$ for $t\in[0,1]$ it holds
\begin{align*}
    \inf_{A \in\A} R(A) 
    &=
    \int_0^1
    \inf_{A \in\A} R(A) 
    \de t
    \leq 
    \int_0^1
    R(A_t) 
    \de t
    =
    T(u)
    \leq 
    S(u)
    \leq
    S(\one_A)
    =
    R(A)
    \quad
    \forall A \in \A,
\end{align*}
and therefore, after taking the infimum over $A\in\A$, we get
\begin{align}\label{eq:int_risk_inf}
    \int_0^1
    R(A_t) 
    \de t
    = 
    \inf_{A \in\A} R(A).
\end{align}
Let us assume that for a subset $N\subset[0,1]$ of positive Lebesgue measure it holds
\begin{align*}
    \inf_{A\in\A} R(A)
    <
    R(A_t)\quad\forall t \in N.
\end{align*}
Integrating this inequality and using \labelcref{eq:int_risk_inf} we would then have
\begin{align*}
    \inf_{A\in\A}R(A)
    <
    \int_0^1
    R(A_t)
    \de t
    = 
    \inf_{A\in\A}R(A),
\end{align*}
which is a contradiction.
Hence we have proved that for Lebesgue almost every $t\in[0,1]$ the set $A_t$ solves \labelcref{eq:meta_problem_hard}.
\end{proof}
To use this theorem we need to specify $\mathfrak Q$ and verify the desiderata.
We consider the following choices:
\begin{align*}
    R(A) &:= \ProbRisk_\Psi(A),\quad A\in\A,\\
    S(u) &:= \Exp{(x,y)\sim\mu}{\abs{u(x)-y}} + \ProbJ_\Psi(u),\quad u \in \mathcal{U},\\
    T(u) &:= \Exp{(x,y)\sim\mu}{\abs{u(x)-y}} + \ProbTV_\Psi(u),\quad u \in \mathcal{U},\\
    \mathcal T &:= \text{weak-* topology of $L^\infty(\X;\nu)$,}
\end{align*}
where $\nu$ was defined in \labelcref{eq:nu}.
We note that \cref{des:compact} follows from the Banach--Alaoglu theorem.
By definition of the functionals $\ProbJ_\Psi$ and $\ProbTV_\Psi$ in \labelcref{eq:J_Psi,eq:ProbTV_Psi} the \cref{des:consistency,des:coarea} are satisfied, as well.
It remains to prove the lower semicontinuity~\cref{des:LSC} and the ordering~\cref{des:order}.

We start with the following lemma:
\begin{lemma}\label{lem:pointwise_cvgc}
    Under \Cref{ass:main_assumption}, let $\nu$ be a measure on $\X$ which satisfies $\distr_x\ll\nu$ for $\rho$-almost every $x\in\X$, and let $(u_n)_{n\in\N}\subset L^\infty(\X;\nu)$ satisfy $u_n\wsto u$ in the sense of \Cref{def:weak_star_nu}.
    Then it holds
    \begin{align*}
        \lim_{n\to\infty}\Exp{{x'}\sim\distr_x}{u_n({x'})}
        =
        \Exp{{x'}\sim\distr_x}{u({x'})}
        \qquad
        \text{for $\rho$-almost every $x\in\X$}.
    \end{align*}
\end{lemma}
\begin{proof}
    Using the Radon--Nikod\'ym theorem, the weak-* convergence implies that for $\rho$-almost every $x\in\X$ it holds
    \begin{align*}
        \lim_{n\to\infty} 
        \Exp{{x'}\sim\distr_x}{u_n({x'})}
        &=
        \lim_{n\to\infty} 
        \int_\X u_n({x'}) \de\distr_x({x'})
        =
        \lim_{n\to\infty} 
        \int_\X u_n({x'}) \frac{\de\distr_x}{\de\nu}({x'})\de\nu({x'})
        \\
        &=
        \int_\X u({x'}) \frac{\de\distr_x}{\de\nu}({x'})\de\nu({x'})
        =
        \int_\X u({x'}) \de\distr_x({x'})
        =
        \Exp{{x'}\sim\distr_x}{u({x'})}
    \end{align*}
    since $\frac{\de\distr_x}{\de\nu}\in L^1(\X;\nu)$.
\end{proof}

Using \cref{lem:pointwise_cvgc} it is relatively straightforward to prove lower semicontinuity of $\ProbJ_\Psi$ if $\Psi$ is a continuous function. 
If, however, $\Psi$ is only lower semicontinuous, e.g., $\Psi(t)=\one_{t>p}$, we will employ a suitable approximation from below of $\Psi$ by continuous functions.
\begin{proposition}[Lower semicontinuity of $\ProbJ_\Psi$]\label{prop:lsc_J_Psi}
    Under \Cref{ass:main_assumption} let $(u_n)_{n\in\N}\subset L^\infty(\X;\nu)$ be a sequence of functions with values in $[0,1]$ satisfying $u_n\wsto u$ in the sense of \Cref{def:weak_star_nu}, and let $\Psi:[0,1]\to[0,1]$ be lower semicontinuous.
    Then $0\leq u\leq 1$ holds $\nu$-almost everywhere and furthermore
    \begin{align*}
        \ProbJ_\Psi(u) \leq \liminf_{n\to\infty}\ProbJ_\Psi(u_n).
    \end{align*}
\end{proposition}
\begin{proof}
    First we show that $0\leq u \leq 1$.
    By the weak-* lower semicontinuity of the $L^\infty$-norm we get $u\leq 1$ from the fact that $0\leq u_n \leq 1$.
    To show that $u\geq 0$ we assume that on a measurable set $N$ with $\nu(N)>0$ it holds $u<0$.
    Then from the weak-* convergence and the fact that $u_n\geq 0$ we obtain
    \begin{align*}
        0>\int_\X u \one_N \de\nu = 
        \lim_{n\to\infty}
        \int_\X u_n \one_N \de\nu
        \geq 0
    \end{align*}
    which is a contradiction.
    Therefore, $u\geq 0$ holds $\nu$-almost everywhere.
    
    Since both terms in the definition of $\ProbJ_\Psi$ are dealt with symmetrically, we assume without loss of generality and for an easier notation that $\rho_1=0$ and rewrite $\ProbJ_\Psi$ as
    \begin{align*}
        \ProbJ_\Psi(u) = 
        \int_\X 
        \left(1-u(x)\right)\Psi\left(\Exp{{x'}\sim\distr_{x}}{u({x'})}\right)\de\rho_0(x).
    \end{align*}
    Since $\Psi$ is lower semicontinuous there exists a sequence of continuous functions $\Psi_\delta:[0,1]\to[0,1]$ which converge to $\Psi$ in the pointwise sense as $\delta\to 0$ and satisfy $\Psi_\delta\leq\Psi$.
    For instance, the functions
    \begin{align*}
        \Psi_\delta(t) := \inf_{s\in[0,1]} \Psi(s) + \frac{1}{\delta}\abs{s-t},\qquad t\in[0,1],
    \end{align*}
    suffice.
    \Cref{lem:pointwise_cvgc} implies that $\Exp{{x'}\sim\distr_{x}}{u_n({x'})}\to\Exp{{x'}\sim\distr_{x}}{u({x'})}$ for $\rho$-almost every $x$ as $n\to\infty$.
    Since $\Psi_\delta$ is continuous, we get $\Psi_\delta\left(\Exp{{x'}\sim\distr_{x}}{u_n({x'})}\right) \to \Psi_\delta\left(\Exp{{x'}\sim\distr_{x}}{u({x'})}\right)$ for $\rho$-almost every $x$ as $n\to\infty$.
    Since $0\leq u_n \leq 1$ and hence $\Psi_\delta\left(\Exp{{x'}\sim\distr_{x}}{u_n({x'})}\right)$ is uniformly bounded, the convergence even holds true in $L^1(\X;\rho)$.
    Taking into account \labelcref{eq:ac_rho}, this implies convergence in $L^1(\X;\rho_0)$ and therefore
    \begin{align}\label{eq:lim_Psi_delta}
        \begin{split}
        &\phantom{{}={}}
        \lim_{n\to\infty}
        \int_\X
        \left(1-u_n(x)\right)
        \Psi_\delta\left(\Exp{{x'}\sim\distr_{x}}{u_n({x'})}\right)
        \de\rho_0(x)
        \\
        &=
        \int_\X 
        \left(1-u(x)\right) \Psi_\delta\left(\Exp{{x'}\sim\distr_{x}}{u({x'})}\right)
        \de\rho_0(x).
        \end{split}
    \end{align}
    Next we would like to use the Fatou lemma to take the limit as $\delta\to 0$ on both sides. 
    For this we notice that the sequence of functions
    \begin{align*}
        f_\delta(x) 
        := 
        \left(1-u_n(x)\right)\Psi_\delta\left(\Exp{{x'}\sim\distr_{x}}{u_n({x'})}\right)
    \end{align*}
    converges to $\left(1-u_n(x)\right)\Psi\left(\Exp{{x'}\sim\distr_{x}}{u_n({x'})}\right)$ pointwise as $\delta\to 0$ and satisfies the bounds
    $f_\delta\geq 0$.
    Thanks to the non-negativity we can apply the standard Fatou lemma.
    Using $\Psi_\delta\leq\Psi$ and \labelcref{eq:lim_Psi_delta} we get
        \begin{align*}
        \ProbJ_\Psi(u)
        &=
        \int_\X \left(1-u(x)\right) \Psi\left(\Exp{{x'}\sim\distr_{x}}{u({x'})}\right)
        \de\rho_0(x)
        \\
        &\leq
        \liminf_{\delta\to 0}
        \int_\X \left(1-u(x)\right) \Psi_\delta\left(\Exp{{x'}\sim\distr_{x}}{u({x'})}\right)
        \de\rho_0(x)
        \\
        &=
        \liminf_{\delta\to 0}
        \lim_{n\to\infty}
        \int_\X
        \left(1-u_n(x)\right)
        \Psi_\delta\left(\Exp{{x'}\sim\distr_{x}}{u_n({x'})}\right)
        \de\rho_0(x)
        \\
        &\leq 
        \liminf_{n\to\infty}
        \int_\X 
        \left(1-u_n(x)\right)
        \Psi\left(\Exp{{x'}\sim\distr_{x}}{u_n({x'})}\right)
        \de\rho_0(x)
        \\
        &=
        \liminf_{n\to\infty}
        \ProbJ_\Psi(u_n).
    \end{align*}  
\end{proof}
We recall the definition of the total variation in \labelcref{eq:ProbTV_Psi} for which we now prove the following ordering with respect to $\ProbJ_\Psi$ defined in \labelcref{eq:J_Psi} in case that $\Psi$ is concave and non-decreasing.
\begin{proposition}\label{prop:TV_upper_bound}
    If $\Psi$ is concave and non-decreasing it holds for every $\A$-measurable function $u:\X\to[0,1]$ that
    \begin{align*}
        \ProbTV_\Psi(u) \leq \ProbJ_\Psi(u).
    \end{align*}
\end{proposition}
\begin{proof}
    As in the proof of \Cref{prop:lsc_J_Psi} we assume without loss of generality that $\rho_1=0$.
    We compute
    \begin{align}
        \nonumber
        \ProbTV_\Psi(u)
        &=
        \int_0^1
        \ProbPer_\Psi(\Set{u\geq t})\de t
        \\\nonumber
        &=
        \int_\X 
        \int_0^1
        \one_{u(x) < t}
        \Psi
        \left(
        \Prob{{x'}\sim\distr_{x}}{u({x'})\geq t}
        \right)
        \de t
        \de\rho_0(x)
        \\\nonumber
        &=
        \int_\X 
        \int_0^1
        \one_{u(x) < t}
        \Psi
        \left(
        \Exp{{x'}\sim\distr_{x}}{\one_\Set{u\geq t}({x'})}
        \right)
        \de t
        \de\rho_0(x)
        \\\label{eq:TV_setting_proof}
        &=
        \int_\X 
        \int_{u(x)}^1
        \Psi
        \left(
        \Exp{{x'}\sim\distr_{x}}{\one_\Set{u\geq t}({x'})}
        \right)
        \de t
        \de\rho_0(x).
    \end{align}
    Since $\Psi$ is concave and non-decreasing, we get from \labelcref{eq:TV_setting_proof} and Jensen's inequality that
    \begin{align*}
        \ProbTV_\Psi(u)
        &\leq 
        \int_\X 
        \left(1-u(x)\right)
        \Psi\left(
        \frac{1}{1-u(x)}
        \int_{u(x)}^1
        \Exp{{x'}\sim\distr_{x}}{\one_\Set{u\geq t}({x'})}
        \de t
        \right)
        \de\rho_0(x)
        \\
        &=
        \int_\X 
        \left(1-u(x)\right)
        \Psi\left(
        \frac{1}{1-u(x)}
        \Exp{{x'}\sim\distr_{x}}{\int_{u(x)}^1\one_\Set{u\geq t}({x'})\de t}
        \right)
        \de\rho_0(x)
        \\
        &=
        \int_\X 
        \left(1-u(x)\right)
        \Psi\left(
        \frac{1}{1-u(x)}
        \Exp{{x'}\sim\distr_{x}}{{\big(u({x'})-u(x)\big)_+}}
        \right)
        \de\rho_0(x)
        \\
        &\leq
        \int_\X 
        \left(1-u(x)\right)
        \Psi\left(
        \frac{1}{1-u(x)}
        \Exp{{x'}\sim\distr_{x}}{u(x')(1-u(x))}
        \right)
        \de\rho_0(x)
        \\
        &=
        \int_\X 
        \left(1-u(x)\right)
        \Psi\left(
        \Exp{{x'}\sim\distr_{x}}{u({x'})}
        \right)
        \de\rho_0(x)
        =
        \ProbJ_\Psi(u).
    \end{align*}
\end{proof}
A remarkable consequence of this lower bound and the lower semicontinuity of $\ProbJ_\Psi$ from \Cref{prop:lsc_J_Psi} is the following proof of lower semicontinuity of $\ProbTV_\Psi$ for sequences of characteristic functions that doesn't use any abstract functional analysis machinery).

\begin{corollary}\label{cor:almost_lsc_ProbTV}
    Let $(A_n)\subset\A$ be a sequence of measurable sets such that $\one_{A_n}\wsto u$ in $L^\infty(\X;\nu)$.
    Then it holds
    \begin{align*}
        \ProbTV_\Psi(u) 
        \leq 
        \liminf_{n\to\infty}
        \ProbTV_\Psi(\one_{A_n}).
    \end{align*}
\end{corollary}
\begin{proof}
    The result follows from \cref{prop:lsc_J_Psi,prop:TV_upper_bound} and the following computation
    \begin{align*}
        \ProbTV_\Psi(u) 
        &\leq 
        \ProbJ_\Psi(u)
        \leq 
        \liminf_{n\to\infty}
        \ProbJ_\Psi(\one_{A_n})
        =
        \liminf_{n\to\infty}
        \ProbPer_\Psi(A_n)
        \\
        &=
        \liminf_{n\to\infty}
        \ProbTV_\Psi(\one_{A_n}).
    \end{align*}
\end{proof}
Now we are ready to prove \cref{thm:existence_Psi,thm:existence_soft} as special cases of the meta-theorem in the beginning of this section.
\begin{proof}[Proof of \Cref{thm:existence_Psi} ]
The result follows from \cref{thm:meta} noticing that \cref{des:order} follows from \cref{prop:TV_upper_bound} and \cref{des:LSC} follows from \cref{prop:lsc_J_Psi} together with the trivial lower semicontinuity of the standard risk $u\mapsto\Exp{(x,y)\sim\mu}{\abs{u(x)-y}}$ since $\rho \ll \nu$.  
\end{proof}

We remark that the main issue with proving the existence result of \cref{thm:existence_Psi} for non-concave functions $\Psi$ is that the ordering \cref{des:order} is violated which was essential for the proof of \cref{thm:meta}.

However, if we already consider the relaxed problem \labelcref{eq:relaxed_problem} of optimizing over soft classifiers instead of characteristic functions and regularizing with $\ProbJ_\Psi$, no ordering is necessary and we obtain existence for very general functions $\Psi$.
\begin{proof}[Proof of \Cref{thm:existence_soft}]
    The result is the first part of \cref{thm:meta} for the choices made in the proof of \Cref{thm:existence_Psi}.
\end{proof}

\red 
\begin{remark}
\label{rem:EqualitySoftHard}
From the proof of \cref{thm:meta} it follows that when $\Psi$ is concave and non-decreasing, then 
\[ \min_{A \in \A} \ProbRisk_\Psi(A) = \min_{ u \in \mathcal{U}} \ProbSRisk_{\Psi}(u),
\]
where
\begin{equation}
\label{def:SRisk}
\ProbSRisk_\Psi(u):= \Exp{(x,y)\sim\mu}{\abs{u(x)-y}} + \ProbJ_\Psi(u).
\end{equation}
\end{remark}

\nc

\red 

\medskip

Next we discuss the existsence results for the original PRL model.

\begin{proof}[Proof of \Cref{thm:existence_PsiOriginal}]
    We apply \cref{thm:meta} with the choices $R= \Risk_\Psi$ and 
   $S= S_\Psi$, where $S_\Psi$ is defined over soft classifiers $u:\X\to[0,1]$ according to:
    \begin{align}
        \SRisk_\Psi(u) := \int_\X \Psi\left(\Exp{x'\sim\distr_x}{u(x')}\right)
        \de\rho_0(x)
        +
        \int_\X \Psi\left(\Exp{x'\sim\distr_x}{1-u(x')}\right)
        \de\rho_1(x),
        \label{def:SRisk2}
    \end{align}
    which by definition satisfies $\SRisk_\Psi(\one_A)=\Risk_\Psi(A)$. We also consider the functional
    \begin{align*}
        T(u) := \int_0^1 R(\Set{u\geq t})\de t.
    \end{align*}
   As topology $\mathcal{T}$ we choose the weak-* topology of $L^\infty(\X;\sigma)$ where the probability measure $\sigma$ is as in \labelcref{eq:sigma}.
    
It is straightforward to check that the quadruple $\mathfrak Q := (R,S,T,\mathcal{T})$ satisfies all five requirements for \cref{thm:meta} to be applicable:
    (1) is clear, (2) is proved as \cref{prop:lsc_J_Psi} using \cref{lem:pointwise_cvgc} with $\sigma$ instead of $\nu$, (3) and (4) are trivial consequences of the definitions, and (5) is an easy consequence of the layer cake representation and Jensen's inequality for concave $\Psi$. Note that (5) is only true if $\Psi$ is concave.
\end{proof}

\begin{proof}[Proof of \Cref{thm:existence_soft_2}]
    The result is the first part of \cref{thm:meta} for the choices made in the proof of \cref{thm:existence_PsiOriginal}. 
\end{proof}

\red 
\begin{remark}
\label{rem:EqualitySoftHard2}
As in \cref{rem:EqualitySoftHard2}, it is easy to verify that when $\Psi$ is concave and non-decreasing, then 
\[ \min_{A \in \A} \Risk_\Psi(A) = \min_{ u \in \mathcal{U}} \SRisk_{\Psi}(u).
\]
\end{remark}

\nc

\section{The functional \texorpdfstring{$\ProbPer_\Psi$}{ProbPerPsi} as a perimeter}
\label{sec:Perimeters}

In this section we shall discuss the interpretation of the functional $\ProbPer_\Psi$ defined in \labelcref{eq:ProbPerPsi} as a \emph{perimeter}. We do this in two ways. 

First, we focus on the case where $\Psi$ is concave and non-decreasing and prove that $\ProbPer_\Psi$ is a \emph{submodular functional}. If, in addition, $\Psi$ is assumed to satisfy $\Psi(0)=0$, then 
$$\ProbPer_\Psi(\X) = \ProbPer_\Psi(\emptyset) =0.$$
Following \cite{chambolle2015nonlocal}, for $\Psi$ satisfying these properties one can interpret $\ProbPer_\Psi$ as a generalized perimeter, i.e., a functional that can be used to measure the ``size'' of the boundary of a set. This discussion is summarized in the next proposition.

\begin{proposition}\label{prop:submodularity}
If $\Psi(0)=0$, then $\ProbPer_\Psi(\X) = \ProbPer_\Psi(\emptyset) =0$. If $\Psi$ is concave and non-decreasing, then the functional $\ProbPer_{\Psi}$ is submodular, meaning that
\begin{align*}
    \ProbPer_\Psi(A\cup B)
    +
    \ProbPer_\Psi(A\cap B)
    \leq 
    \ProbPer_\Psi(A)
    +
    \ProbPer_\Psi(B)
    \quad\forall A,B\in\A.
\end{align*}
\end{proposition}
\begin{example}
    For $\Psi(t)=t$ our perimeter reduces to the perimeter on the \emph{random walk space} $(\X,\distr)$, introduced in \cite{mazon2020total}: $\ProbPer_\Psi(A)=\int_{\X\setminus A}\int_{A}\de\distr_x\de\rho_0(x)+\int_{A}\int_{\X\setminus A}\de\distr_x\de\rho_1(x)$.
\end{example}
For proving \cref{prop:submodularity} we need the following lemma.
\begin{lemma}
\label{lem:AuxConcave}
Let $\Psi:[0,\infty) \rightarrow \R$ be a concave and non-decreasing function, and let $0 \leq a \leq b \leq b' \leq a'$ be real numbers with $a+ a'\leq  b+ b' $. Then
\begin{align*} \Psi(a) + \Psi(a') \leq \Psi(b) + \Psi(b'). \end{align*}
\end{lemma}

\begin{proof}
Let $a,a',b,b'$ be as stated. Since $\Psi$ is concave and finite, it satisfies the fundamental theorem of calculus and thus it is possible to write
\begin{align*} \Psi(s) =  \Psi(a) + \int_a^s \Psi'(r) \de r , \quad s \geq a \end{align*}
for a function $\Psi'$ that is non-increasing and non-negative. It follows that
\begin{align*}
   \Psi(b) - \Psi(a)  = \int_{a}^b \Psi'(r) \de r  \geq (b-a) \Psi'(b)  \geq (a'-b') \Psi'(b) \geq \int_{b'}^{a'} \Psi'(r) \de r = \Psi(a') - \Psi(b'),
\end{align*}
which is precisely what we wanted to show.
\end{proof}

With the above preliminary results in hand, we are ready to prove \Cref{prop:submodularity}.
\begin{proof}[Proof of \Cref{prop:submodularity}]
First, the fact that $\ProbPer_\Psi(\X)=\ProbPer_\Psi(\emptyset)=0$ if $\Psi(0)=0$ is easy to see from the definition of $\ProbPer_\Psi$. Second, we trivially have
\begin{align*}
    \Prob{{x'} \sim \distr_{x}}{{x'}\in A\cup B}
    +
    \Prob{{x'} \sim \distr_{x}}{{x'}\in A\cap B}
    \leq 
    \Prob{{x'} \sim \distr_{x}}{{x'}\in A}
    + 
    \Prob{{x'} \sim \distr_{x}}{{x'}\in B}.  
\end{align*}
Define
\begin{alignat*}{2}
    &a' := \Prob{{x'} \sim \distr_{x}}{{x'}\in A\cup B},
    \qquad
    &&b' := \Prob{{x'} \sim \distr_{x}}{{x'}\in B}
    \\
    &b\phantom{'}
    := \Prob{{x'} \sim \distr_{x}}{{x'}\in A},
    \qquad
    &&a\phantom{'} := \Prob{{x'} \sim \distr_{x}}{{x'}\in A\cap B};
\end{alignat*}
without the loss of generality we can assume that $b$ and $b'$ defined above satisfy $b\leq b'$, for otherwise we can simply swap these labels. We can then use \Cref{lem:AuxConcave} to conclude that:
\begin{align}\label{eq:submodularity_Psi}
    \begin{split}
    &\phantom{{}={}}
    \Psi\left(\Prob{{x'} \sim \distr_{x}}{{x'}\in A\cup B}\right)
    +
    \Psi\left(\Prob{{x'} \sim \distr_{x}}{{x'}\in A\cap B}\right)
    \\
    &\leq 
    \Psi\left(\Prob{{x'} \sim \distr_{x}}{{x'}\in A}\right)
    + 
    \Psi\left(\Prob{{x'} \sim \distr_{x}}{{x'}\in B}\right).  
    \end{split}
\end{align}
The submodularity follows directly once we have verified the following pointwise identity:
 \begin{align}\label{eq:pointwise_submod}
    \begin{split}
    &\phantom{{}={}}
    \one_{x\in (A\cup B )^c} 
    \Psi\left(\Prob{{x'} \sim \distr_{x}}{{x'}\in A\cup B}\right)
    +
    \one_{x\in(A\cap B )^c} 
    \Psi\left(\Prob{{x'} \sim \distr_{x}}{{x'}\in A\cap B}\right)
    \\
    &\leq
    \one_{x\in A^c} 
    \Psi\left(\Prob{{x'} \sim \distr_{x}}{{x'}\in A}\right)
    +
    \one_{x\in B^c} 
    \Psi\left(\Prob{{x'} \sim \distr_{x}}{{x'}\in B}\right). 
    \end{split}
 \end{align}
 To do this we consider two complementary cases:
 
 \textbf{Case 1, $x\in(A\cup B)^c$:}
 This is equivalent to $x\in A^c\cap B^c$.
 Furthermore, since $(A\cup B)^c\subset(A\cap B)^c$ we also have that $x\in(A\cap B)^c$.
 Hence, all indicator functions in \labelcref{eq:pointwise_submod} take the value one and \labelcref{eq:pointwise_submod} is the same as \labelcref{eq:submodularity_Psi}, which we have already verified.
 
 \textbf{Case 2, $x\in A\cup B$:}
 In this case the first indicator function on the left hand side of \labelcref{eq:pointwise_submod} is zero.
 
 \textbf{Case 2.1, $x\in A\cap B$:}
 In this subcase all indicator functions are equal to zero and the inequality is trivially satisfied.
 
 \textbf{Case 2.2, $x\in A\cup B \setminus (A\cap B)$:}
 Without loss of generality we can assume that $x\in A\setminus B = A \cap B^c$. 
 In this case only the second indicator function on the left hand side and the second one on the right hand side of \labelcref{eq:pointwise_submod} take the value one and the inequality reduces to the trivial inequality
 \begin{align*}
     \Psi\left(\Prob{{x'} \sim \distr_{x}}{{x'}\in A\cap B}\right)
     \leq 
     \Psi\left(\Prob{{x'} \sim \distr_{x}}{{x'}\in B}\right)
 \end{align*}
 which is true since $\Psi$ is non-decreasing and $A\cap B \subset B$.
\end{proof}

\red 
\begin{remark}
From the submodularity of $\ProbPer_\Psi$ one can show that the functional $\ProbTV_\Psi$ defined in \labelcref{eq:ProbTV_Psi} is convex. Indeed, this can be proved following the proof of Proposition 3.4 in \cite{chambolle2010continuous}, where, actually, we do not need to assume the lower semicontinuity of the functional a priori. 
\end{remark}
\nc 

Next, we consider rather general $\Psi$ and show that $\ProbPer_\Psi$ is related to a standard \emph{local} perimeter when $\X= \R^d$ and the probability measure $\distr_x$ localizes to a Dirac delta $\delta_x$; for the case of adversarial training such a connection was proved in \cite{bungert2024gamma}, where the authors utilized the notion of $\Gamma$-convergence of functionals. We take a first step in this direction by proving that for sufficiently smooth sets the probabilistic perimeter converges to a local one if the family of probability distributions $\distr_x$ localizes suitably. 
For example, one could think of $\distr_x := \operatorname{Unif}(B_\eps(x))$, which converges to a point mass at $x$ if $\eps\to 0$.
To make our setting precise, we pose the following general assumption:
\begin{assumption}\label{ass:kernel}
    We assume that $\mathcal{X}=\mathbb{R}^d$, $\Psi(0)=0$, $\Psi$ is Borel measurable and bounded, and $\rho_1, \rho_0$ have continuous densities with respect to the Lebesgue measure which we shall also denote as $\rho_1, \rho_0$.   
    Furthermore, we assume that there is $\eps>0$ and a measurable function $K:\mathcal{X}\times\R^d\to[0,\infty)$ such that  
     for every $x\in\R^d$ we have the representation
    \begin{align*}
        \de\distr_x(x') = \eps^{-d} K\left(x,\frac{x'-x}{\eps}\right)\de x'.
    \end{align*}
    We also assume that  
    for every $x\in \mathcal{X}$ we have $K(x, \bullet)\in L^1(\R^d)$, $\int_{\R^d} K(x,z)\de z=1$, $K(x,z)=0$ if $\abs{z}>1$, and that for every $z\in \R^d$ the mapping $x\mapsto K(x,z)$ is continuous.
\end{assumption}
In the more difficult case when $\Psi$ is not necessarily continuous, we shall need some additional assumptions on the kernel.
\begin{assumption}\label{ass:kernel_extra}
    For our extra assumptions on the kernel, we shall require that for every $z\in \R^d$ the mapping $x\mapsto K(x,z)$ is $C^1$ and for every $x\in \R^d$ the mapping $z\mapsto K(x,z)$ is lower semicontinuous.  Furthermore, for all $x\in \R^d$, $t\in (-1,1)$, and $n\in S^{d-1}$, the Radon transform of $K$ with respect to the $z$-variable
    \[
    \mathcal{R}(K(x,\cdot))(n,t):=\int_{\{z\cdot n=t\}} K(x,z) \de \mathcal H^{d-1}(z)
    \]
    is strictly positive.
\end{assumption}
The following theorem derives the asymptotics of the probabilistic perimeter as $\eps\to 0$ and is the rigorous counterpart of \cref{formalthm:asymptotics}. 
\begin{theorem}\label{thm:consistency}
Let $\X = \R^d$. Under \cref{ass:kernel}, if $ A$ has a compact $C^{1,1}$ boundary and either $\Psi$ is continuous or $K$ satisfies the additional \cref{ass:kernel_extra}, then 
\begin{align}\label{eq:limit_perimeter}
    \lim_{\epsilon\to 0} \frac{1}{\epsilon}\ProbPer_{\Psi}(A)=\int_{\partial A} 
    \sigma_{0,\Psi}\left[x,n(x)\right]\rho_0(x)
    +
    \sigma_{1,\Psi}\left[x, n(x)\right]
    \rho_1(x) \de \mathcal{H}^{d-1}(x),
\end{align}
where we let $n(x)$ denote the normal to $\partial A$ at a point $x\in \partial A$, and for any vector $v\in \mathbb{R}^d$ we define
\begin{align*}
\sigma_{\Psi}^0\left[x, v\right]
:=
\int_0^{1} \Psi\left(\int_{\{z\cdot v\leq -t\}} K(x,z)\de z\right)\de t,
\quad
\sigma_{\Psi}^1\left[x, v\right]
:=
\int_{0}^1 \Psi\left(\int_{\{z\cdot v\geq t\}} K(x,z)\de z\right)\de t.
\end{align*}
\end{theorem}
\begin{remark}\label{rem:perimeter}
If $K$ is radially symmetric and independent of $x\in \mathcal{X}$, then $\sigma_{\Psi}^0=\sigma_{\Psi}^1=:\sigma_\Psi$ is just a constant. 
E.g., for $K(x,z) := \abs{B_1(0)}^{-1}\one_{\abs{z}\leq 1}$ and $\Psi(t) = \one_{t>p}$ it is trivial that for $p=0$ we have $\sigma_\Psi = 1$.
However, for $p\geq\tfrac12$ one easily sees $\sigma_\Psi = 0$, hence the limiting perimeter equals zero and there is no regularization effect.
Using the function $\Psi(t) = \min\left\lbrace t/p,1\right\rbrace$ corrects this degeneracy.
\end{remark}

Notably, for radially symmetric $K$ the limiting perimeter in \labelcref{eq:limit_perimeter} coincides, provided $\sigma_\Psi>0$, with the one derived for adversarial training (problem \labelcref{eq:AT_binary}) in \cite{bungert2024gamma}, although they considered more general (potentially discontinuous) densities $\rho_i$.
In particular, our result indicates that for very small adversarial budgets the regularization effect of both probabilistically robust learning and adversarial training is dominated by the perimeter in \labelcref{eq:limit_perimeter}.
While \cref{thm:consistency} already completes half of the proof (namely the limsup inequality) of $\Gamma$-convergence of $\tfrac1\eps\ProbPer_\Psi$ to the limiting perimeter, the remaining liminf inequality is beyond the scope of this paper.
Proving that the convergence \labelcref{eq:limit_perimeter} does not only hold for sufficiently smooth sets as assumed in \cref{thm:consistency} but even in the sense of $\Gamma$-convergence is an extremely important topic for future work since only $\Gamma$-convergence allows to deduce from the convergence of the perimeters that also the solutions of probabilistically robust learning converge to certain regular Bayes classifiers as $\eps\to 0$, see \cite[Section 4.2]{bungert2024gamma}. The exploration of this is left for future work. 

\begin{proof}[Proof of \Cref{thm:consistency}]
Under\Cref{ass:kernel}, a simple change of variables shows
\begin{align*}
 \ProbPer_{\Psi}(A)
 &=
 \int_{A^c}  \Psi\left(\int_{\mathbb{R}^d} \mathbf{1}_{A}(x+\epsilon z)K(x,z)\de z\right) \rho_0(x)\de x
 \\
 &\qquad 
 +\int_{A}  \Psi\left(\int_{\mathbb{R}^d} \mathbf{1}_{A^c}(x+\epsilon z)K(x,z)\de z\right) \rho_1(x)\de x.
\end{align*}
Let $\tau:\mathbb{R}^d\to\mathbb{R}$ be the signed distance function to $\partial A$ such that $\tau(x)\leq 0$ for $x\in A$.  
Using $\tau$, we can rewrite the previous line as
\begin{align*}
    \ProbPer_{\Psi}(A)
     &=
    \int_{\{\tau(x)\geq 0\}}  \Psi\left(\int_{\{\tau(x+\epsilon z)\leq 0\}} K(x,z)\de z\right) \rho_0(x)\de x\\
    &\qquad 
    +\int_{\{\tau(x)\leq 0\}}  \Psi\left(\int_{\{\tau(x+\epsilon z)\geq 0\}} K(x,z)\de z\right) \rho_1(x)\de x.
\end{align*}
Recalling that $K(x,z)=0$ whenever $|z|>1$, it follows that
\begin{align*}
    \ProbPer_{\Psi}(A)
     &=
    \int_{\{0\leq \tau(x)\leq \epsilon\}}  \Psi\left(\int_{\{\tau(x+\epsilon z)\leq 0\}} K(x,z)\de z\right) \rho_0(x)\de x\\
    &\qquad 
    +\int_{\{-\epsilon\leq \tau(x)\leq 0\}}  \Psi\left(\int_{\{\tau(x+\epsilon z)\geq 0\}} K(x,z)\de z\right) \rho_1(x)\de x.
\end{align*}
In the rest of what follows, we will focus on proving that 
\[
\lim_{\epsilon\to 0}\frac{1}{\epsilon}\int_{\{0\leq \tau(x)\leq \epsilon\}}  \Psi\left(\int_{\{\tau(x+\epsilon z)\leq 0\}} K(x,z)\de z\right) \rho_0(x)\de x=\int_{\partial A} 
\sigma_{0,\Psi}\left[x,n(x)\right]\rho_0(x) \de \mathcal{H}^{d-1}(x),
\]
the argument for showing that 
\[
  \lim_{\epsilon\to 0} \frac{1}{\epsilon}\int_{\{-\epsilon\leq \tau(x)\leq 0\}}  \Psi\left(\int_{\{\tau(x+\epsilon z)\geq 0\}} K(x,z)\de z\right) \rho_1(x)\de x= \int_{\partial A} \sigma_{1,\Psi}\left[x, n(x)\right]
    \rho_1(x) \de \mathcal{H}^{d-1}(x)
\]
will be identical.

Since $\partial A$ is $C^{1,1}$, there exists some $\epsilon_0>0$ and such that for all $\epsilon<\epsilon_0$ the map   $T_{\epsilon}(y,t):\partial A\times [-1,1]\to \{x\in \R^d: 0\leq \tau(x)\leq \epsilon\}$ given by $T_{\epsilon}(y,t)=y+tn(y)$ is a bijection and $\epsilon^{-1}\det(DT_{\epsilon})$ converges uniformly to 1 as $\epsilon \to 0$. 
Using this change of variables, we may write
\begin{align*}
    &\phantom{{}={}}
    \frac{1}{\epsilon}\int_{\{0\leq \tau(x)\leq \epsilon\}}  \Psi\left(\int_{\{\tau(x+\epsilon z)\leq 0\}} K(x,z)\de z\right) \rho_0(x)\de x
    \\
    &=
    \int_{\partial A} 
    \int_0^1
    \frac{1}{\epsilon}\det(DT_{\epsilon}(y,t)) 
    \Psi\left(
    a_{\epsilon}(y,t)
    \right) 
    \rho_0(y+\epsilon t n(y))\de t \de\mathcal{H}^{d-1}(y)
\end{align*}
where we abbreviate
\begin{align*}
    a_{\epsilon}(y,t)&:=\int_{\{\tau(y+\epsilon (tn(y)+z))\leq 0\}} K(y+\epsilon tn(y),z)\de z
\end{align*}

Since we know that $\lim_{\epsilon\to 0} \tfrac{\det(DT_{\epsilon}(y,t))}{\epsilon}=1$ and $\lim_{\epsilon \to 0}\rho_0(y+\epsilon t n(y))=\rho_0(y)$ pointwise almost everywhere, the main difficulty lies in passing to the limit in the term involving $\Psi$.
For this we shall first prove convergence of $a_\eps(y,t)$ to
\begin{align*}
    a(y,t):=\int_{\{z\cdot n(y)\leq -t\}} K(y,z)\de z
\end{align*}
Since $\nabla \tau(y)=n(y)$ for any $y\in \partial A$, we have the expansion
\begin{align*}
\tau(y+\epsilon(tn(y)+z))&=\epsilon (t+z\cdot n(y))+O(\epsilon^2) . 
\end{align*}
It now follows from our assumptions on $K$ that for all $y\in \partial A$ and all $t\in [0,1]$
\begin{align*}
   \lim_{\epsilon\to 0} a_{\epsilon}(y,t)=a(y,t).
\end{align*}
If $\Psi$ is continuous, the result now follows from dominated convergence.  
If $\Psi$ is not continuous, then we must work harder.

Since $\Psi\in L^1([0,1])\cap L^{\infty}([0,1])$, we can always find a sequence of smooth functions $\Psi_n$ such that $\norm{\Psi_n}_{L^{\infty}([0,1])}\leq \norm{\Psi}_{L^{\infty}([0,1])}$ and $\Psi_n$ converges pointwise almost everywhere to $\Psi$ on $[0,1]$. 
If we can show that
\begin{align*}
    &\phantom{{}={}}
    \lim_{n\to\infty}\limsup_{\epsilon\to 0}\abs{\int_{\partial A} \int_0^1  \big(\Psi(a_{\epsilon}(y,t))-\Psi_n(a_{\epsilon}(y,t))\big)\de t\de \mathcal{H}^{d-1}(y)}=0,
\end{align*}
and 
\begin{align*}
    &\phantom{{}={}}
\lim_{n\to\infty}\abs{\int_{\partial A} \int_0^1  \big(\Psi(a(t,y))-\Psi_n(a(t,y))\big)\de t\de \mathcal{H}^{d-1}(y)}=0,
\end{align*}
then we can safely approximate $\Psi$ by $\Psi_n$ and use the previous argument where $\Psi$ was continuous to conclude the result. Hence, it remains to show that the above integrals vanish.

Viewing $a_{\epsilon}$ as maps from $\partial A\times [0,1]\to [0,1]$ we can construct  measures $\mu_{\epsilon}, \mu$ on $ [0,1]$ via pushforwards by setting 
\[
\int_{0}^1 g(s)\de\mu_{\epsilon}(s):=\int_0^1 \int_{\partial A} g(a_{\epsilon}(y,t)) \de t\de \mathcal H^{d-1}(y)
\]
\[
\int_0^1 g(s)\de\mu(s):=\int_0^1 \int_{\partial A}g(a(y,t)) \de t\de \mathcal H^{d-1}(y)
\]
for any bounded continuous function $g:\partial A\times [0,1]\to\R$. Using these definitions, we can write
\begin{align*}
    &\phantom{{}={}}
    \abs{\int_{\partial A} \int_0^1  \big(\Psi(a_{\epsilon}(t,y))-\Psi_n(a_{\epsilon}(t,y))\big)\de t\de \mathcal{H}^{d-1}(y)}=\abs{ \int_{\partial A}\int_0^1  \big(\Psi(s)-\Psi_n(s)\big)\de\mu_{\epsilon}(s)},
\end{align*}
and 
\begin{align*}
    &\phantom{{}={}}
    \abs{\int_{\partial A} \int_0^1  \big(\Psi(a(t,y))-\Psi_n(a_{\epsilon}(t,y))\big)\de t\de \mathcal{H}^{d-1}(y)}=\abs{ \int_{\partial A}\int_0^1  \big(\Psi(s)-\Psi_n(s)\big)\de\mu(s)}.
\end{align*}
If we can show that the $\mu_{\epsilon}$ are uniformly integrable on $\partial A\times [0,1]$ (with respect to the product measure $\de t\de H^{d-1}(y)$)
then the pointwise almost everywhere convergence of $\Psi_n$ to $\Psi$ along with the control $\norm{\Psi_n}_{L^{\infty}([0,1]}\leq \norm{\Psi}_{L^{\infty}([0,1]}$ will imply that
\[
\lim_{n\to \infty}\limsup_{\epsilon\to 0} \abs{ \int_{\partial A}\int_0^1  \big(\Psi(t')-\Psi_n(t')\big)\de\mu_{\epsilon}(y,t')}=0,
\]
and for free it will give us
\[
\lim_{n\to \infty} \abs{ \int_{\partial A}\int_0^1  \big(\Psi(t')-\Psi_n(t')\big)\de\mu(y,t')}=0,
\]
since $\mu$ is the distributional limit of the $\mu_{\epsilon}$ (from the pointwise convergence of $a_{\epsilon}$ to $a$). 

To prove that the $\mu_{\epsilon}$ are uniformly integrable, we will show that $|\partial_t a_{\epsilon}|$ is strictly bounded away from 0 whenever $t$ is bounded away from 1.  To do this, we shall need the additional \cref{ass:kernel_extra} on our kernel $K$.
Let us first assume that $K$ is $C^1$
in both variables.
Changing variables $z'=z+tn(y)$ we can write
\begin{align*}
    a_{\epsilon}(y,t)=\int_{\{\tau(y+\epsilon z')\leq 0\}} K(y+\epsilon tn(y), z'-tn(y))\, \de z',
\end{align*}
and hence
\begin{align*}
    \partial_t a_{\epsilon}(y,t)
    &=
    \int_{\{\tau(y+\epsilon z')\leq 0\}} 
    \epsilon 
    \nabla_y K(y+\epsilon tn(y), z'-tn(y))
    \cdot n(y) \de z'
    \\
    &\qquad
    -
    \int_{\{\tau(y+\epsilon z')\leq 0\}} 
    \nabla_{z'} K(y+\epsilon tn(y), z'-tn(y)) \cdot n(y) \de z'.
\end{align*}
Since the second term is the complete derivative with respect to $z'$, we can integrate by parts to obtain
\begin{align*}
    \partial_t a_\eps(y,t)
    &=
    \int_{\{\tau(y+\epsilon z')\leq 0\}}  \epsilon \nabla_y K\big(y+\epsilon tn(y), z'-tn(y)\big)\cdot n(y)\de z'
    \\
    &\qquad
    -\int_{\{\tau(y+\epsilon z')=0\}} K\big(y+\epsilon tn(y), z'-tn(y)\big)  n(y)\cdot \nabla \tau(y+\epsilon z')\de \mathcal{H}^{d-1}(z')  
\end{align*}
Using the expansion $\nabla \tau(y+\epsilon z)=n(y)+O(\epsilon)$ and the smoothness properties of $K$, we see that 
\begin{align*}
\partial_t a_{\epsilon}(y,t)=-\int_{\{\tau(y+\epsilon z')=0\}} K(y, z'-tn(y))  \de \mathcal{H}^{d-1}(z') +O(\epsilon)
\end{align*}
where we note that the constant in the big O bound does not depend on the differentiability of $K$ with respect to the $z$ variable.
Hence, after approximating $K$ with $C^1$ kernels, we can assume the above control holds for any kernel $K$ satisfying both \cref{ass:kernel,ass:kernel_extra}.

Thus, for each fixed $y\in \partial A, t\in [0,1]$ we have 
\[
\liminf_{\epsilon\to 0} |\partial_t a_{\epsilon}(t,y)|\geq\int_{\{z'\cdot n(y)=0\}} K(y,z'-tn(y)) \de \mathcal{H}^{d-1}(z')=\int_{\{z\cdot n(y)=-t\}} K(y,z)\de \mathcal{H}^{d-1}(z),
\]
where we have used the lower semicontinuity of $K$ with respect to $z$.  We can also now recognize that the right hand side is the Radon transform $\mathcal{R}(K(y,\cdot))(-t,n(y))$.  From our additional \cref{ass:kernel_extra}, we know that $\mathcal{R}(K(y,\cdot))(t,n(y))>0$ for each fixed $y\in \partial A$, $t\in (-1,1)$ and $n\in S^{d-1}$. Fix some $\delta>0$.  Since the set $\partial A\times [0, 1-\delta]\times S^{d-1}$ is compact, it follows that there exists some $c_{\delta}>0$ such that $\mathcal{R}(K(y,\cdot))(-t,n(y))\geq c_{\delta}$ for all $y\in \partial A, n(y)\in S^{d-1}, t\in [0, 1-\delta]$. Therefore, 
\[
\liminf_{\epsilon\to 0} |\partial_t a_{\epsilon}(t,y)|\geq c_{\delta}
\]
for all $y\in \partial A$ and 
$t\in [0, 1-\delta]$.  
Thus, for all $\epsilon>0$ sufficiently small we have
\[
|\partial_t a_{\epsilon}(t,y)|\geq c_{\delta}/2
\]
for all $y\in \partial A$ and 
$t\in [0, 1-\delta]$.

Now we are ready to show that $\mu_{\epsilon}$ is uniformly integrable. Let $\mathcal{L}_{[a,b]}$ denote the Lebesgue measure on the interval $[a,b]$.  From our work above, for all $\epsilon>0$ sufficiently small, we have the inequality
\[
\mu_{\epsilon}=a_{\epsilon\,\#} \big(\mathcal{H}^{d-1}_{\partial A}\otimes \mathcal{L}_{[0,1]}\big)\leq \frac{2\mathcal H^{d-1}(\partial A)}{c_{\delta}}\mathcal{L}_{[0,1]}+a_{\epsilon\,\#} \big(\mathcal{H}^{d-1}_{\partial A}\otimes \mathcal{L}_{[1-\delta,1]}\big).
\]
Thus, for any measurable $E\subset [0,1]$, we have
\begin{align*}
\mu_{\epsilon}(E)\leq& 2|E|\frac{\mathcal H^{d-1}(\partial A)}{c_{\delta}}+\int_{a_{\epsilon}^{-1}(E)\cap (\partial A\times [1-\delta, 1])} \de t\de \mathcal H^{d-1}(y)\\
\leq& 2|E|\frac{\mathcal H^{d-1}(\partial A)}{c_{\delta}}+\delta \mathcal H^{d-1}(\partial A).
\end{align*}
Hence, for all $\delta>0$,
\[
\lim_{|E|\to 0} \limsup_{\epsilon\to 0}\mu_{\epsilon}(E)\leq \delta \mathcal H^{d-1}(\partial A).
\]
Thus, the $\mu_{\epsilon}$ are uniformly integrable so we are done.

\end{proof}

\section{Interpolation properties of probabilistically robust learning}
\label{sec:Interpolation}

In this section we discuss the limiting behavior of the energies $\ProbRisk_{\Psi_p}$ and $\Risk_{\Psi_p}$, for $\Psi_p(t):= \min\left\lbrace t/p,1\right\rbrace$, as $p\to 0$ or as $p$ grows. This limiting behavior is studied in the sense of $\Gamma$-convergence, which, in particular, is closely connected to the convergence of minimizers of the sequence of  functionals. We start by recalling the definition of this notion of convergence for functionals defined over arbitrary metric spaces.

\begin{definition}
 Given a topological space $(\U, \tau)$ and a sequence of functionals $F_n: \U \rightarrow [0,\infty]$, and another functional $F: \U \rightarrow [0,\infty]$, we say that the sequence of functionals $F_n$ \mbox{$\Gamma$-converges} toward $F$ (with respect to the topology $\tau$)  if the following two conditions hold:

\begin{enumerate}
    \item \textbf{Liminf inequality:} For any $u \in \U$ and any sequence $\{ u_n \}_{n \in \N}$ converging (in the $\tau$ topology) toward $u$ we have:
    \[ \liminf_{n \rightarrow \infty}  F_n(u_n) \geq F(u). \]
    \item \textbf{Limsup inequality:} For every $u \in \U$ there exists a sequence $\{ u_n \}_{n \in \N}$ converging (in the $\tau$ topology) toward $u$ such that:
    \[ \limsup_{n \rightarrow \infty} F_n(u_n) \leq F(u). \]
\end{enumerate}
\end{definition}
The relevance of the notion of $\Gamma$-convergence becomes apparent when an additional compactness property holds. 

\begin{proposition}[Fundamental Theorem of $\Gamma$-convergence; see \cite{Braides2002} and \cite{DalMaso1993}] 
Let $(\U, \tau)$ be a topological space. Let $F_n: \U \rightarrow [0,\infty]$ be a sequence of functionals $\Gamma$-converging toward a functional $F: \U \rightarrow [0,\infty]$ that is not identically equal to $+\infty$. Suppose, in addition, that the sequence of functionals $\{ F_n \}_{n \in \N}$ satisfies the following compactness property: any sequence $\{ u_n\}_{n \in \N}$ satisfying 
\[ \sup_{n \in \N} F_n(u_n) < \infty  \]
is precompact in the topology $\tau$, i.e., every subsequence of $\{ u_n \}_{n \in \N}$ has a further subsequence that converges to an element of $\U$.

Then 
\[ \lim_{n \rightarrow \infty} \inf_{u \in \U} F_n(u) = \inf_{u \in \U} F(u).  \]
Moreover, if $u_n$ for every $n \in \N$ is a minimizer of $F_n$, then any cluster point of $\{ u_n \}_{n \in \N}$ is a minimizer of $F$.
\label{prop:GammaFund}
\end{proposition}

For our analysis in this section it will be useful to define $\Psi_0(t) := \one_{t>0}$ for $t\geq 0$, which is a concave and non-decreasing function in its domain. The probabilistic perimeter associated to this function is
\begin{align*}
    \ProbPer_{\Psi_0}(A) 
    &= 
    \int_{A^c}\one_{\Prob{x'\sim\distr_x}{x'\in A}>0}\de\rho_0(x)
    +
    \int_{A}\one_{\Prob{x'\sim\distr_x}{x'\in A^c}>0}\de\rho_1(x)
    \\
    &=
    \int_{A^c}\distr_x\text{-}\esssup \one_A\de\rho_0(x)
    +
    \int_{A}\distr_x\text{-}\esssup \one_{A^c}\de\rho_1(x),
    \\
\end{align*}
and the corresponding probabilistic risk \labelcref{eq:ProbRiskPsi} can be written as
\begin{align*}
\begin{split}
\ProbRisk_{\Psi_0}(A) & = 
\int_\X ( \one_{A^c}(x) \cdot \distr_x\text{-}\esssup \one_A   + \one_{A} (x)) \de\rho_0(x)
\\ & + \int_\X ( \one_{A}(x)  \cdot  \distr_x\text{-}\esssup \one_{A^c} + \one_{A^c}(x))\de\rho_1(x).
\end{split}
\end{align*}
Likewise, the risk $\Risk_{\Psi_0}$ takes the form 
\begin{align*}
\begin{split}
\Risk_{\Psi_0}(A) & = 
\int_\X \distr_x\text{-}\esssup \one_A \de\rho_0(x)  + \int_\X  \distr_x\text{-}\esssup \one_A^c  \de\rho_1(x).
\end{split}
\end{align*}

We start by proving $\Gamma$-convergence of $\ProbRisk_{\Psi_p}$ toward $\ProbRisk_{\Psi_0}$ as $p\to 0$ in the weak-* topology of $L^\infty(\X; \nu)$. 
\begin{proposition}[$\Gamma$-convergence of modified PRL]\label{prop:cvgc_p_to_0_mod}
    For the functions $\Psi_p(t):= \min\left\lbrace t/p,1\right\rbrace$ it holds that $\ProbPer_{\Psi_p}$ $\Gamma$-converges to $\ProbPer_{\Psi_0}$ and $\ProbRisk_{\Psi_p}$ $\Gamma$-converges to $\ProbRisk_{\Psi_0}$ as $p\to 0$ in the weak-* topology of $L^\infty(\X;\nu)$, where $\nu$ is as in \labelcref{eq:nu}.
\end{proposition}
\begin{proof}
    Without loss of generality we assume $\rho_1=0$ for a more concise notation.

    We start with the liminf inequality.
    Let $\Set{A_p}_{p\in(0,1)}$ be a sequence of sets such that $\one_{A_p}$ converges to $\one_A$ as $p\to 0$ in the weak-* sense. 
    Then \cref{lem:pointwise_cvgc} implies that 
    \begin{align*}
        \lim_{p\to 0}\Prob{x'\sim\distr_x}{x'\in A_p}
        =
        \lim_{p\to 0}\Exp{x'\sim\distr_x}{\one_{A_p}(x')}
        =
        \Exp{x'\sim\distr_x}{\one_{A}(x')}
        =
        \Prob{x'\sim\distr_x}{x'\in A}.
    \end{align*}
    Next, we note that for $0<p<l$ it holds $\Psi_p \geq \Psi_l$.
    Hence we get
    \begin{align*}
        \ProbPer_{\Psi_p}(A_p) 
        &\geq         \int_{A_p^c}\Psi_l\left(\Prob{x'\sim\distr_x}{x'\in A_p}\right)\de\rho_0(x)
        \\
        &=
        \int_{\X}(1-\one_{A_p})\Psi_l\left(\Prob{x'\sim\distr_x}{x'\in A_p}\right)\de\rho_0(x).
    \end{align*}
    As argued in the proof of \cref{prop:lsc_J_Psi}, the functions $x\mapsto\Psi_l\left(\Prob{x'\sim\mathfrak m_x}{x'\in A_p}\right)$ converge to $x\mapsto\Psi_l\left(\Prob{x'\sim\mathfrak m_x}{x'\in A}\right)$ in $L^1(\X;\rho_0)$ as $p\to 0$.
    Using this together with $\one_{A_p}\wsto\one_A$ in $L^\infty(\X;\nu)$ and taking into account \labelcref{eq:ac_nu,eq:ac_rho} we get 
    \begin{align*}
        \liminf_{p\to 0}
        \ProbPer_{\Psi_p}(A_p) 
        \geq 
        \int_{A^c}\Psi_l\left(\Prob{x'\sim\distr_x}{x'\in A}\right)\de\rho_0(x).
    \end{align*}
    Since $l>0$ was arbitrary we can let $l\to 0$ and use Fatou's lemma to obtain
    \begin{align*}
        \liminf_{p\to 0}
        \ProbPer_{\Psi_p}(A_p) 
        \geq 
        \int_{A^c}\Psi_0\left(\Prob{x'\sim\distr_x}{x'\in A}\right)\de\rho_0(x)
        =
        \ProbPer_{\Psi_0}(A).
    \end{align*}
    For the limsup inequality we observe that, since $\Psi_p\leq\Psi_0$ for all $p\in(0,1)$, it holds $\ProbPer_{\Psi_p}(A) \leq \ProbPer_{\Psi_0}(A)$ and therefore the limsup inequality $\limsup_{p\to 0}\ProbPer_{\Psi_p}(A)\leq\ProbPer_{\Psi_0}(A)$ is true for the constant sequence equal to $A$.

    Being continuous perturbations of the probabilistic perimeters with respect to the \mbox{weak-*} topology of $L^\infty(\X; \nu)$, the probabilistic risks $\ProbRisk_{\Psi_p}$ are easily seen to $\Gamma$-converge to $\ProbRisk_{\Psi_0}$ as $p\to 0$; see the background section in \cite{braides2008asymptotica}.
\end{proof}

\red
\begin{remark}
\label{rem:GammaSoft1}
With essentially the same proof as in \cref{prop:cvgc_p_to_0_mod}, one can show that $\ProbSRisk_{\Psi_p}$ $\Gamma$-converges in the weak-* topology of $L^\infty(\X; \nu)$ toward the functional $\ProbSRisk_{\Psi_0}$ as $p \rightarrow 0$. We recall that $\ProbSRisk_\Psi$ was introduced in \labelcref{def:SRisk} and corresponds to the l.s.c extension of $\ProbRisk_{\Psi}$ to soft classifiers.
\end{remark}
\nc

Next, we discuss the $\Gamma$-convergence of $\Risk_{\Psi_p}$ toward $\Risk_{\Psi_0}$. This limiting energy may in general be different from $\ProbRisk_{\Psi_0}$.

%

\begin{proposition}[$\Gamma$-convergence of original PRL]\label{prop:cvgc_p_to_0_original}
    For the functions $\Psi_p(t):= \min\left\lbrace t/p,1\right\rbrace$ it holds that $\Risk_{\Psi_p}$ $\Gamma$-converges to $\Risk_{\Psi_0}$ as $p\to 0$ in the weak-* topology of $L^\infty(\X;\sigma)$, where $\sigma$ is as in \labelcref{eq:sigma}.
    
    Furthermore, the $\Gamma$-convergence also holds with respect to the weak* topology of $L^\infty(\X; \nu)$, where $\nu$ is as in \labelcref{eq:nu}.
\end{proposition}
\begin{proof}
    For simplicity, we again assume $\rho_1=0$.
    The proof starts verbatim as the proof of \cref{prop:cvgc_p_to_0_mod}, using \cref{lem:pointwise_cvgc} with $\sigma$ instead of $\nu$. 
    Then one gets for $0<p<l$ that
    \begin{align*}
        \liminf_{p\to 0}
        \Risk_{\Psi_p}(A_p) 
        \geq 
        \int_\X \Psi_l\left(\Prob{x'\sim\distr_x}{x'\in A}\right)\de\rho_0(x).
    \end{align*}
    Taking also the limit $l\to 0$ one obtains
    \begin{align*}
        \liminf_{p\to 0}
        \Risk_{\Psi_p}(A_p) 
        &\geq 
        \int_\X \Psi_0\left(\Prob{x'\sim\distr_x}{x'\in A}\right)\de\rho_0(x).
    \end{align*}
    The limsup inequality is again trivial, using that $\Psi_p\leq \Psi_0$ for all $p>0$. 
    
    We point out that, since constant sequences work for the limsup inequality, the $\Gamma$-convergence also holds with respect to any stronger topology as well. In particular, the $\Gamma$-convergence holds under the weak-* topology of $L^\infty(\X;\nu)$.
 
\end{proof}

\red
\begin{remark}
\label{rem:GammaSoft2}
With essentially the same proof as in \cref{prop:cvgc_p_to_0_original}, one can prove that $\SRisk_{\Psi_p}$ $\Gamma$-converges in the weak-* topology of $L^\infty(\X; \sigma)$ (or of $L^\infty(\X; \nu)$) toward the functional $\SRisk_{\Psi_0}$ as $p \rightarrow 0$. We recall that $\SRisk_\Psi$ was introduced in \labelcref{def:SRisk2} and corresponds to the l.s.c. extension of $\Risk_{\Psi}$ to soft classifiers.
\end{remark}
\nc

The next example shows that the functionals $\ProbRisk_{\Psi_0}$ and $ \Risk_{\Psi_0}$, the $\Gamma$-limits of modified PRL and original PRL, respectively, may in general be different from each other.  

\begin{example}\label{exam:GamaLimitsAredifferent}
Consider the data distribution $\mu = \tfrac12\delta_{(x^0,0)} + \tfrac12\delta_{(x^1,1)}$ where $x^0=(a,0),x^1=(b,0)\in\R^2$ with $a<b$. Let furthermore $\distr_x := \operatorname{Unif}(B_\eps(x))$ where $0<\eps<\tfrac{|b-a|}{2}$. Let $A:=\{x=(x_1, x_2)\in\R^2\st x_1>a+\frac{b-a}{2}\}\cup\{x^0\}$. Then it follows that
\[ \ProbRisk_{\Psi_0}(A) = 1/2 > 0 = \Risk_{\Psi_0}(A).  \]

\end{example}

While the energies $\ProbRisk_{\Psi_0}$, $R_{\Psi_0}$, and $\Risk_{\mathrm{adv}}$ may in general be different, they are always ordered, with $\Risk_{\mathrm{adv}}$ being the largest. This is the content of the next proposition.

\begin{proposition}
	\label{prop:Ordering}
 Suppose that for every $x \in \X$ the following identity holds $\distr_x(\X \setminus B_\veps(x)) =0 $. Then for every measurable $A$ the following holds:
\begin{equation}
\Risk_{\Psi_0}(A) \leq  \ProbRisk_{\Psi_0}(A) \leq  \Risk_\mathrm{adv}(A).
\label{eq:IneqsRisks}
\end{equation}
\end{proposition}

\begin{proof}
It is straightforward to show that the following pointwise identity holds for any measurable set $A$:
\[ \sup_{\tilde x \in B_\eps(x)} \one_A(\tilde x)   \geq \one_{A^c}(x) \cdot \distr_x\text{-}\esssup \one_A + \one_{A}(x)  \geq   \distr_x\text{-}\esssup \one_A. \]
Applying this identity to $A^c$ we get:
\[ \sup_{\tilde x \in B_\eps(x)} \one_{A^c}(\tilde x)   \geq \one_{A}(x) \cdot \distr_x\text{-}\esssup \one_{A^c} + \one_{A^c}(x)  \geq   \distr_x\text{-}\esssup \one_{A^c}. \]
Integrating the first chain of inequalities with respect to $\rho_0$, the second with respect to $\rho_1$, and adding up the resulting expressions, we obtain \labelcref{eq:IneqsRisks}.

\end{proof}

Next, we specialize to the case where the admissible set of hard classifiers is $\A=\mathcal{B}(\X)$, i.e., the set of all Borel measurable subsets of $\X$ (in the machine learning literature this is an agnostic learning setting) and then show that, while $\ProbRisk_{\Psi_0}$ and $\Risk_{\Psi_0}$ may be different, their minimal values coincide with that of the adversarial training value functional $\Risk_{\text{adv}}$ if we impose some reasonable assumptions on the family of measures $\{ \distr_{x}\}_{x \in \X}$. Our result implies that, under \cref{assump:ABsCont} below, both PRL and modified PRL models are consistent with AT, provided consistency is interpreted as consistency in minimal risks.

\red 

\begin{assumption}
\label{assump:ABsCont}
For every $x \in \supp(\rho)$ (where $\supp(\rho)$ denotes the support of the Borel probability measure $\rho$) the following conditions hold:
\begin{enumerate}
    \item $\distr_x(\X \setminus B_\veps(x)) =0 $.
    \item $B_\veps(x) \subseteq \supp(\distr_x)$. \red 
\item For every $x' \in \supp(\rho)$ the measure $\distr_{x'} \lfloor_{B_\veps(x)} $ is absolutely continuous with respect to $\distr_x$. 
\end{enumerate}

\end{assumption}

\begin{remark}
\cref{assump:ABsCont} is very mild and in the Euclidean setting, when $\X = \R^d$, is satisfied by the simple model $\distr_x = \operatorname{Unif}(B_\eps(x))$ .
\end{remark}

\nc

\begin{theorem}[Consistency of optimal energies for PRL and modified PRL]
\label{prop:ConsistencyEnergies}	
Let $\mathcal{A}= \mathcal{B}(\X)$ be the Borel $\sigma$-algebra over $\X$. Under \cref{assump:ABsCont},
\[ \lim_{p \rightarrow 0} \min_{A \in \A} \Risk_{\Psi_p}(A) =  \min_{A \in \A} \Risk_{\Psi_0} (A)=  \min_{A \in \A} \Risk_{\mathrm{adv}}(A), \]
as well as
\[ \lim_{p \rightarrow 0} \min_{A \in \A} \ProbRisk_{\Psi_p}(A)  =  \min_{A \in \A} \ProbRisk_{\Psi_0}(A)     = \min_{A \in \A} \Risk_{\mathrm{adv}}(A), \]
where we recall $\Psi_p(t)= \min\{t/p,1 \}$ and $\Psi_0(t)=\one_{t >0}.$
\end{theorem}
\begin{proof}
\red

Thanks to \cref{rem:GammaSoft1,rem:GammaSoft2} and the Banach--Alaoglu theorem, we can apply the fundamental theorem of $\Gamma$-convergence (i.e., \cref{prop:GammaFund}) to conclude that 
\[ \lim_{p \rightarrow 0 } \min_{u \in \U }  \SRisk_{\Psi_p}(u) =  \min_{u \in \U }  \SRisk_{\Psi_0}(u)  \]
and 
\[ \lim_{p \rightarrow 0 } \min_{u \in \U }  \ProbSRisk_{\Psi_p}(u) =  \min_{u \in \U }  \ProbSRisk_{\Psi_0}(u).   \]
In addition, since $\Psi_p$ is concave and non-decreasing for every $p\geq 0$, we have, thanks to \cref{rem:EqualitySoftHard,rem:EqualitySoftHard2}, 
\[ \min_{A \in \A} \Risk_{\Psi_p} = \min_{u \in \U} \SRisk_{\Psi_p}(u)  \]
as well as 
\[ \min_{A \in \A} \ProbRisk_{\Psi_p} = \min_{u \in \U} \ProbSRisk_{\Psi_p}(u).  \]
\nc 
Therefore
\[ \lim_{p \rightarrow 0} \min_{A \in \A} \Risk_{\Psi_p}(A) = \min_{A \in \A } \Risk_{\Psi_0}(A)  \]
and
\[ \lim_{p \rightarrow 0} \min_{A \in \A} \ProbRisk_{\Psi_p}(A) = \min_{A \in \A} \ProbRisk_{\Psi_0}(A). \]	
On the other hand, thanks to \cref{prop:Ordering}, we have 
\[   \min_{A \in \A} \Risk_{\mathrm{adv}}(A)  \geq \min_{A \in \A} \ProbRisk_{\Psi_0}(A)  \geq \min_{A \in \A}  \Risk_{\Psi_0}(A). \]
Thus it suffices to prove that 
	\[ \min_{A \in \A}   \Risk_{\Psi_0}(A)   \geq  \min_{A \in \A} \Risk_{\mathrm{adv}}(A). \]
 
To see this, let $A$ be an arbitrary measurable subset of $\X$. We can then proceed as in the proof of \cite[Lemma 3.8]{bungert2023geometry} and use the fact that \[ \{ x \in \X \text{ s.t. } \mathrm{dist}(x, \supp(\rho)) < \veps  \}  \subseteq \supp(\sigma), \]
which follows from (2) in \cref{assump:ABsCont}, to conclude that we can find a measurable set $A'$ which is $\sigma$-equivalent to $A$ (i.e., $\sigma(A \triangle  A') =0$) and satisfies:
\[   \sigma\text{-}\esssup_{\tilde x \in B_\veps(x)} \one_A(\tilde x) = \sup_{\tilde x \in B_\veps(x)} \one_{A'}(\tilde x) , \quad \sigma\text{-}\esssup_{\tilde x \in B_\veps(x)} \one_{A^c}(\tilde x) = \sup_{\tilde x \in B_\veps(x)} \one_{A'^c}(\tilde x) \]   
for $\rho$-a.e. $x \in \X$. Thanks to \labelcref{eq:ac_sigma}, we deduce that 
\[ \distr_x(A \triangle A') =0 , \quad \rho\text{-a.e. } x \in \X.\]
It follows that for $\rho$ almost every $x \in \X$ we have
\begin{align*}
\sigma\text{-}\esssup_{\tilde x \in B_\veps(x)} \one_{A}(\tilde x) &= \sup_{\tilde x \in B_\veps(x)} \one_{A'}(\tilde x) \geq \distr_x\text{-}\esssup\one_{A'}(\tilde x)
 = \distr_x\text{-}\esssup \one_{A}(\tilde x).
\end{align*}
as well as
\[ \sigma\text{-}\esssup_{\tilde x \in B_\veps(x)} \one_{A^c}(\tilde x)  \geq \distr_x\text{-}\esssup \one_{A^c}(\tilde x).\]
Since $\sigma$ is absolutely continuous with respect to $\distr_x$ for $\rho$ a.e. $x \in \X$ (thanks to \cref{assump:ABsCont}), we deduce that
\[ \distr_x\text{-}\esssup \one_{A^c}(\tilde x) \geq \sigma\text{-}\esssup_{\tilde x \in B_\veps(x)} \one_{A^c}(\tilde x), \quad  \distr_x\text{-}\esssup \one_{A}(\tilde x) \geq \sigma\text{-}\esssup_{\tilde x \in B_\veps(x)} \one_{A}(\tilde x)  \]
for $\rho$-a.e. $x \in \X$. Combining the above, we deduce that for $\rho$-a.e. $x \in \X$
\[  \distr_x\text{-}\esssup \one_{A^c}(\tilde x) = \sup_{\tilde x \in B_\veps(x)} \one_{A'^c}(\tilde x), \quad  \distr_x\text{-}\esssup \one_{A}(\tilde x) = \sup_{\tilde x \in B_\veps(x)} \one_{A'}(\tilde x).  \]
From this fact, we deduce that
\begin{align*}
\Risk_{\Psi_0}(A)  & =  \int_\X \distr_x\text{-}\esssup_{\tilde x \in B_\veps(x)} \one_{A} \de \rho_0(x) +  \int_\X  \distr_x\text{-}\esssup_{\tilde x \in B_\veps(x)} \one_{A^c} \de \rho_1(x) 
\\  & = \int_\X \sup_{\tilde x \in B_\veps(x)} \one_{A'}(\tilde x) \de \rho_0(x) +  \int_\X \sup_{\tilde x \in B_\veps(x)} \one_{A'^c}(\tilde x ) \de \rho_1(x) 
\\  & = \Risk_{\mathrm{adv}}(A') 
\\& \geq \min_{\tilde A \in \A} \Risk_{\mathrm{adv}}(\tilde A).
\end{align*}
Taking the infimum over $A$ on the left hand side we obtain the desired result.

\end{proof}

While the above result states that both PRL and modified PRL models recover the AT model from the perspective of consistency of their optimal energies (at least in the agnostic setting, where we can prove existence of minimizers to AT; see \cite{bungert2023geometry}), the next example illustrates that minimizers of the PRL energies may not converge to AT minimizers as the parameter $p$ goes to zero. Studying the behavior of minimizers is important, since, after all, PRL models and AT are used to produce robust classifiers by minimizing energy functionals. We start with an example.

\begin{example}
Let $\X = \R$ and for every $x \in \R$ let $\distr_{x}$ be the uniform measure over $B_\veps(x)$, the open Euclidean ball of radius $\veps$ and center $x$. Consider the probability distribution $\mu$ over $\X \times \{ 0,1\}$ given by
\[\mu =  \frac{2}{5}\delta_{(x^1 , 0) }  + \frac{1}{5}\delta_{(x^2 , 0) } + \frac{2}{5}\delta_{(x^3 , 1) },  \]
where $x^1 < x^2 < x^3$,  $|x^1 - x^2| < \eps$, and $B_\eps(x^2) \cap B_\eps(x^3) \not = \emptyset$, $B_\eps(x^1) \cap B_\eps(x^3)  = \emptyset$.

It is straightforward to show that the optimal adversarial risk $\Risk_{\mathrm{adv}}^*$ in this situation is $1/5$ and that $A'$ achieves this adversarial risk if and only if $B_\veps(x^3) \subseteq A'  $ and $B_\veps(x^1) \subseteq A'^c$. On the other hand, the set $\tilde{A} = B_\veps(x^3) \cup \{  x^2\}$ satisfies $\ProbRisk_{\Psi_0}(\tilde A)= 1/5$. In particular, thanks to \cref{prop:ConsistencyEnergies}, $\tilde{A}$ is a minimizer of both $\ProbRisk_{\Psi_0}$ and $\Risk_{\Psi_0}$ (i.e., the $\Gamma$-limits of the PRL models as $p \to 0$). Note, however, that there does not exist a set $A' \sim_\nu \tilde{A}$ that is a minimizer of $\Risk_{\mathrm{adv}}$. This is because for $A'$ to be a minimizer of $\Risk_{\mathrm{adv}}$ we would need to exclude $x^2$ from  $\tilde{A}$, but it is impossible to achieve this while maintaining the condition $\tilde A \sim_\nu A'$, since $\rho$ assigns mass strictly larger than zero to $\{x^2 \}$. 
\end{example}

From the above example it is clear that in general we cannot expect that an arbitrary minimizer of $\Risk_{\Psi_0}$ or $\ProbRisk_{\Psi_0}$ can be turned into a solution to adversarial training by modifying it within a set of $\nu$ measure zero, in particular by only modifying the classifier outside of the observed data (the support of the distribution $\rho$) and their admissible perturbations. In other words, this means that there may be solutions to PRL or modified PRL for $p \approx 0$ that, putting measure theoretic technicalities aside, aren't solutions to AT. In this sense, PRL or its modified version may not necessarily recover the AT model when $p \rightarrow 0$. Fortunately, it is possible to avoid the issue described above by imposing one extra assumption on the family of distributions $\{ \distr_x\}_{x \in \X}$. 

\begin{assumption}
\label{assump:AdditionalMinimizers}
Suppose that for every $x \in \supp(\rho)$ the measure  $\mathcal{\rho}\lfloor_{B_\veps(x)}$, the restriction of $\rho$ to the ball $B_\veps(x)$, is absolutely continuous with respect to $\distr_x$.
\end{assumption}

\begin{remark}
   If for $x \in \supp(\rho)$ we choose $\distr_x$ to be $\distr_x= \frac{1}{2} \operatorname{Unif}(B_\eps(x))  +\frac{1}{2 \rho(B_\veps(x))} \rho \lfloor_{B_\veps(x)} $, then the family $\{  \distr_x\}_{x \in \X}$ satisfies \cref{assump:ABsCont,assump:AdditionalMinimizers}.
\end{remark}

\red

\begin{theorem}
Suppose that \cref{assump:ABsCont,assump:AdditionalMinimizers} hold. Suppose that for every $p>0$ the set $A_p$ is a minimizer of $\ProbRisk_{\Psi_p}$ over $\A= \mathcal{B}(\X)$, and suppose that, as $p \rightarrow 0$, $A_p$ converges towards a set $A$ in the weak* topology of $L^\infty(\X ; \nu)$. Then there is a set $A'$ with $A'\sim_\nu A$ such that $A'$ is a minimizer of $\Risk_{\mathrm{adv}}$ over the class of all Borel measurable sets. 

The above continues to be true if we substitute $\ProbRisk_{\Psi_p}$ with $\Risk_{\Psi_p}$.
\end{theorem}

\begin{proof}
The $\Gamma$-convergence in \cref{prop:cvgc_p_to_0_mod} implies that $A$ as in the statement must be a minimizer of $\ProbRisk_{\Psi_0}$. Following the proof of \cref{prop:ConsistencyEnergies} we know there exists a set $A'$ which is $\sigma$ equivalent to $A$ and satisfies $\ProbRisk_{\Psi_0}(A) \geq \Risk_{\Psi_0}(A) = \Risk_{\mathrm{adv}}(A') $. Thanks to \cref{prop:ConsistencyEnergies}, it follows that $A'$ is a minimizer of $\Risk_{\mathrm{adv}}$. The desired result now follows from \cref{assump:AdditionalMinimizers}, since this assumption implies that $\nu$ is absolutely continuous with respect to $\sigma$ and as a consequence $A \sim_\nu A'$.

\end{proof}

So far we have studied the behaviors of the energies $\ProbRisk_{\Psi_p}$ and $\Risk_p$ as $p\rightarrow 0$. To wrap up this section we discuss the limits of these energies as $p$ grows.

\begin{theorem}
For the functions $\Psi_p(t):= \min\left\lbrace t/p,1\right\rbrace$ it holds that $\ProbRisk_{\Psi_p}$ $\Gamma$-converges to 
\[ \Exp{(x,y)\sim\mu}{\abs{ \one_A(x)-y}} + \int_A \int_{A^c}    \de \distr_x(\tilde x)   \de \rho_1(x) + \int_{A^c} \int_A    \de \distr_x(\tilde x)   \de \rho_0(x)    \]
as $p \rightarrow 1$, and to $\Exp{(x,y)\sim\mu}{\abs{ \one_A(x)-y}}$ as $p \rightarrow \infty$, in the weak-* topology of $L^\infty(\X;\nu)$. We recall $\nu$ was introduced in \labelcref{eq:nu}.
\end{theorem}
\begin{proof}
This easily follows from the fact that, actually, the convergence of the energies is uniform over all measurable $A$ as $p \rightarrow 1$ and as $p\rightarrow \infty$. 
\end{proof}

\begin{remark}
    In the previous result we have highlighted the case $p=1$, as this is the case where the modified PRL model coincides with a regularized risk minimization problem with nonlocal perimeter penalty of the type investigated in \cite{mazon2020total} in the context of random walk metric spaces. 
\end{remark}

\begin{remark}
    In contrast to the modified PRL model, which recovers the standard risk minimization model as $p$ grows, the energy $\Risk_{\Psi_p}$ is easily seen to converge uniformly toward the zero functional as $p \rightarrow \infty$.
\end{remark}

\nc

\section{PRL for general learning settings and the conditional value at risk}
\label{sec:general_models}
After extensively discussing the binary and 0-1 loss case in the previous sections, we shift our attention to training general hypotheses $h\in\mathcal{H}$ using general loss functions $\ell:\Y\times\Y\to\R$, even when $\Y$ is not binary. Motivated by \cref{prop:rewrite_max,prop:rewrite_max_Psi} it is natural to consider the following probabilistically robust optimization problem
\begin{align}\label{eq:proposed_problem_general}
    \inf_{h\in\mathcal{H}}
    \Exp{(x,y)\sim\mu}{
        \max
        \left\lbrace
        \ell(h(x),y),
        p\text{-}\esssup_{x' \sim \distr_x}
        \ell(h(x'),y)
        \right\rbrace}.
\end{align}
Since the $p$-$\esssup$ operator is notoriously hard to optimize, one shall replace it with the conditional value at risk (CVaR) which is convex and easier to optimize \cite{robey2022probabilistically,rockafellar2000optimization}.
For a function $f:\X\to\R$ and a probability distribution $\distr$ the CVaR at level $p\in (0,1)$ is defined as
\begin{align}\label{eq:CVaR}
    \operatorname{CVaR}_p(f;\distr)
    := 
    \inf_{\alpha\in\R}
    \alpha
    +
    \frac{\Exp{x'\sim\distr_x}{\left(f(x')-\alpha\right)_+}}{p}.
\end{align}
It is easy to see that $p\text{-}\esssup_{x'\sim\distr}f(x')\leq\operatorname{CVaR}_p(f;\distr)$.
Using CVaR in place of the $p\text{-}\esssup$ operator, a tractable version of \labelcref{eq:proposed_problem_general} is 
\begin{align}\label{eq:proposed_problem_relaxed}
    \inf_{h\in\mathcal{H}}
    \Exp{(x,y)\sim\mu}{
    \max\Big\lbrace
        \ell(h(x),y),
        \operatorname{CVaR}_p(\ell(h(\bullet),y);\distr_x)
    \Big\rbrace
    }.
\end{align}
We emphasize that, if the loss function $\ell(\bullet,\bullet)$ is convex in its first argument, then \labelcref{eq:proposed_problem_relaxed} is a convex function of the hypothesis $h$. 
Furthermore, CVaR is positively homogeneous and hence also \labelcref{eq:proposed_problem_relaxed} is positively homogeneous in the loss function.
So, taking the maximum of the samplewise CVaR and standard risk is consistent with a standard dimensionality analysis as both terms scale in the same way.

In the binary classification case we can prove \cref{prop:rewrite_cvar_Psi} which states that the CVaR relaxation corresponds precisely to using the risk $\ProbRisk_{\Psi_p}$ with the piecewise linear and concave function $\Psi_p(t)=\min\left\lbrace t/p,1\right\rbrace$ for which our theory from \cref{sec:hard_classifiers} applies.
\begin{proof}[Proof of \cref{prop:rewrite_cvar_Psi}]
If $\ell$ is the $0$-$1$-loss then the function $f := \ell(\one_A(\cdot),y)$ for $A\in\A$ can be written as $f = \one_{S}$ for set $S\in\{A,A^c\}$, depending on whether $y=0$ or $y=1$.
So it suffices to deal with the case $f = \one_A$ where we can write
\begin{align*}
\operatorname{CVaR}_{p}(f;\distr_x)= \inf_{\alpha \in \R } \alpha +  (1- \alpha)_+ \frac{\distr_x(A)}{p}  + (-\alpha)_+  \frac{1-\distr_x(A)}{p}.   \end{align*}
Notice that the function 
\begin{align*} \zeta(\alpha) := \alpha +  (1- \alpha)_+ \frac{\distr_x(A)}{p}  + (-\alpha)_+  \frac{1-\distr_x(A)}{p}, \quad \alpha \in \R,  \end{align*}
is continuous and piecewise linear with kinks at $\alpha=0$ and $\alpha =1$. 
Moreover, since $0 < p < 1$ it holds $\zeta(\alpha)\geq \zeta(1)$ for $\alpha>1$, and $\zeta(\alpha)\geq\zeta(0)$ for $\alpha<0$. Thus the minimum of $\zeta$ is attained at either $\alpha=0$ or $\alpha =1$. Therefore
\begin{align*} \operatorname{CVaR}_{p}(f;\distr_x)  = \min \{ \zeta(0) , \zeta(1) \}  = \min\left\{ \frac{\distr_x(A)}{p}, 1\right\}
=
\Psi_p\left(\Prob{x'\sim\distr_x}{\one_A(x')}\right).
\end{align*}
This together with \cref{prop:rewrite_max_Psi} implies the claim.
\end{proof}

\section{Numerical experiments}
\label{sec:numerics}
In this section, we conduct a comparative analysis between the original formulation of PRL (denoted as ``PRL'' in \Cref{tab:combined_results_compressed_clean_3}) and our modification of PRL which is based on \labelcref{eq:proposed_problem_relaxed} (denoted as ``m-PRL''). We build upon the code of \cite{robey2022probabilistically} and our implementation is available on \texttt{GitHub}.\footnote{\url{https://github.com/DanielMckenzie/Begins_with_a_boundary}}
The algorithmic realization of \labelcref{eq:proposed_problem_relaxed} is a straightforward adaptation of their algorithm, which alternatingly minimizes the inner optimization problem that defines CVaR and the outer optimization to find a suitable classifier, see \Cref{alg_algorithm} in \Cref{sec:computational}. Our experiments are conducted on MNIST and CIFAR-10 and to ensure a fair comparison we adhere to the hyperparameter settings described in \cite{robey2022probabilistically}, such that both the original and modified algorithms utilize the same set of hyperparameters for each specified value of $p$. The corresponding results for several baseline algorithms including empirical risk minimization and adversarial training can be found in \cite{robey2022probabilistically}.

\begin{table}[ht]
    \caption{Accuracies [\%] of the geometric and original algorithm for different values of $p$.}
    \label{tab:combined_results_compressed_clean_3}
    \centering
    \begin{tabular}{cllcccccc}
    \toprule
    \textbf{Data} & \textbf{$p$} & \textbf{Algorithm} & \textbf{Clean} & \textbf{Adv} & \textbf{Aug} & \textbf{\tiny ProbAcc(0.1)} & \textbf{\tiny ProbAcc(0.05)} & \textbf{\tiny ProbAcc(0.01)} \\
    \midrule
    \multirow{8}{*}{\rotatebox[origin=c]{90}{\textbf{MNIST}}}
        &\multirow{2}{*}{0.01}
        & m-PRL & \bf 99.20 & \bf 12.19 & 99.04 & 98.18 & 97.69 & 96.38 \\
        && PRL & 99.19 & 10.76 & 98.90 & 97.94 & 97.38 & 95.67 \\
        \cmidrule{2-9}
        &\multirow{2}{*}{0.1}
        & m-PRL & 99.28 &\bf  14.20 & 99.22 & 98.70 & 98.45 & 97.86 \\
        && PRL & \bf 99.32 & 8.94 & 99.22 & 98.70 & 98.46 & 97.80 \\
        \cmidrule{2-9}
        &\multirow{2}{*}{0.3}
        & m-PRL &\bf  99.29 &\bf  3.02 & 99.21 & 98.76 & 98.53 & 97.95 \\
        && PRL & 99.27 &\bf  3.02 & 99.22 & 98.77 & 98.55 & 98.01 \\
        \cmidrule{2-9}
        &\multirow{2}{*}{0.5}
        & m-PRL &\bf  99.27 &\bf  1.80 & 99.21 & 98.72 & 98.44 & 97.93 \\
        && PRL & 99.26 & 1.68 & 99.19 & 98.72 & 98.47 & 97.80 \\
        \midrule[2pt]
        \multirow{8}{*}{\rotatebox[origin=c]{90}{\textbf{CIFAR-10}}}
        &\multirow{2}{*}{0.01}
        & m-PRL & 80.65 & 0.15 & 78.13 & 73.44 & 72.13 & 68.80 \\
        && PRL &\bf  81.73 &\bf  0.24 & 79.16 & 74.61 & 73.19 & 69.96 \\
        \cmidrule{2-9}
        &\multirow{2}{*}{0.1}
        & m-PRL & 88.15 & 0.14 & 85.96 & 82.55 & 81.46 & 78.81 \\
        && PRL &\bf  88.28 &\bf  0.19 & 85.61 & 82.21 & 81.06 & 78.28 \\
        \cmidrule{2-9}
        &\multirow{2}{*}{0.3}
        & m-PRL &\bf  90.43 &\bf  11.80 & 88.70 & 85.17 & 83.93 & 80.93 \\
        && PRL & 89.97 & 7.20 & 88.62 & 85.07 & 83.75 & 80.87 \\
        \cmidrule{2-9}
        &\multirow{2}{*}{0.5}
        & m-PRL & \bf 91.51 & 1.93 & 88.94 & 85.53 & 84.18 & 81.21 \\
        && PRL & 90.74 & \bf 1.99 & 88.94 & 85.54 & 84.35 & 81.57 \\
        \bottomrule
    \end{tabular}
    \end{table}

\subsection{Comparing accuracy and robustness}
We report the clean accuracy, adversarial accuracy (subject to PGD attacks), accuracy on noise-augmented data, and quantile accuracies for different values of $\rho$ defined---as in \cite[(6.1)]{robey2022probabilistically}---for a single data point $(x,y)$ by
\begin{equation}
    \ProbAcc_{\rho}(x,y) = \one_{\Prob{{x'}\sim\distr_x}{h(x')=y}>p}.
\end{equation}
In practice, we use an empirical proportion, computed over 100 samples from $\distr_x$, in lieu of the true probability, as done in \cite{robey2022probabilistically}. All quantities are averaged over three runs, and we perform model selection based on the best clean validation accuracy; see \Cref{sec:training_details} for more training details.

The results in \Cref{tab:combined_results_compressed_clean_3} demonstrate that the geometric modification does not compromise the accuracy of the original algorithm, and sometimes leads to improved performance, even as it provides stronger theoretical guarantees. Note that neither PRL nor m-PRL should be expected to match the adversarial robustness of classifiers trained with PGD attacks \cite{madry2017towards} or other worst-case optimization techniques. Instead, they shine with superior clean accuracy and easier training while maintaining probabilistic robustnesss, as well as a certain degree of adversarial robustness. This corroborates the findings of \cite{robey2022probabilistically}.

\subsection{On the existence of pathological points}
\label{sec:Pathological_Sec}
We say $(x,y)$  is a pathological data point if $x$ is incorrectly classified yet a large proportion of perturbations to $x$ yield a correct classification. See \cref{fig:spike_solution} for a simple example. Although the discussion of \Cref{sec:Introduction} asserts that such points {\em can} arise for classifiers trained using PRL, it is interesting to verify whether this phenomenon occurs in practice.

To do so, we examine all misclassified CIFAR-10 images, in both test and train splits, for a model trained using PRL as described above. For each image $x$, we generate 100 samples $x^{\prime}\sim \mathfrak{m}_x$ and compute the proportion of such samples which are correctly classified. As shown in \cref{fig:cifar_001_histogram}, while for the vast majority of such $x$, all $100$ $x^{\prime}$ remain incorrectly classified, images $x$ exist where an arbitrarily large proportion of the $x^{\prime}$ are correctly classified. In short, pathological points do occur in practice.

We repeated the same experiment, but with a model trained using m-PRL. Intriguingly, while m-PRL is guaranteed in theory to prevent pathological points (see \cref{prop:Patho} in the Appendix), in practice they still occur, see \cref{fig:cifar_001_histogram}.  We did not find strong empirical evidence that using m-PRL results in fewer pathological points in practice, see also \cref{fig:cifar_05_histogram,fig:cifar_01_histogram,fig:cifar_03_histogram}. We attribute this gap between theory and practice to the fact that m-PRL (as well as PRL)  approximates the value of $\operatorname{CVaR}_p$, which involves a high-dimensional expectation, with an empirical average, see \labelcref{eq:proposed_problem_relaxed} and line 7 in \cref{alg_algorithm}).

\begin{figure}[ht]
    \begin{subfigure}[b]{0.45\linewidth}
        \centering
        \includegraphics[width=\linewidth]{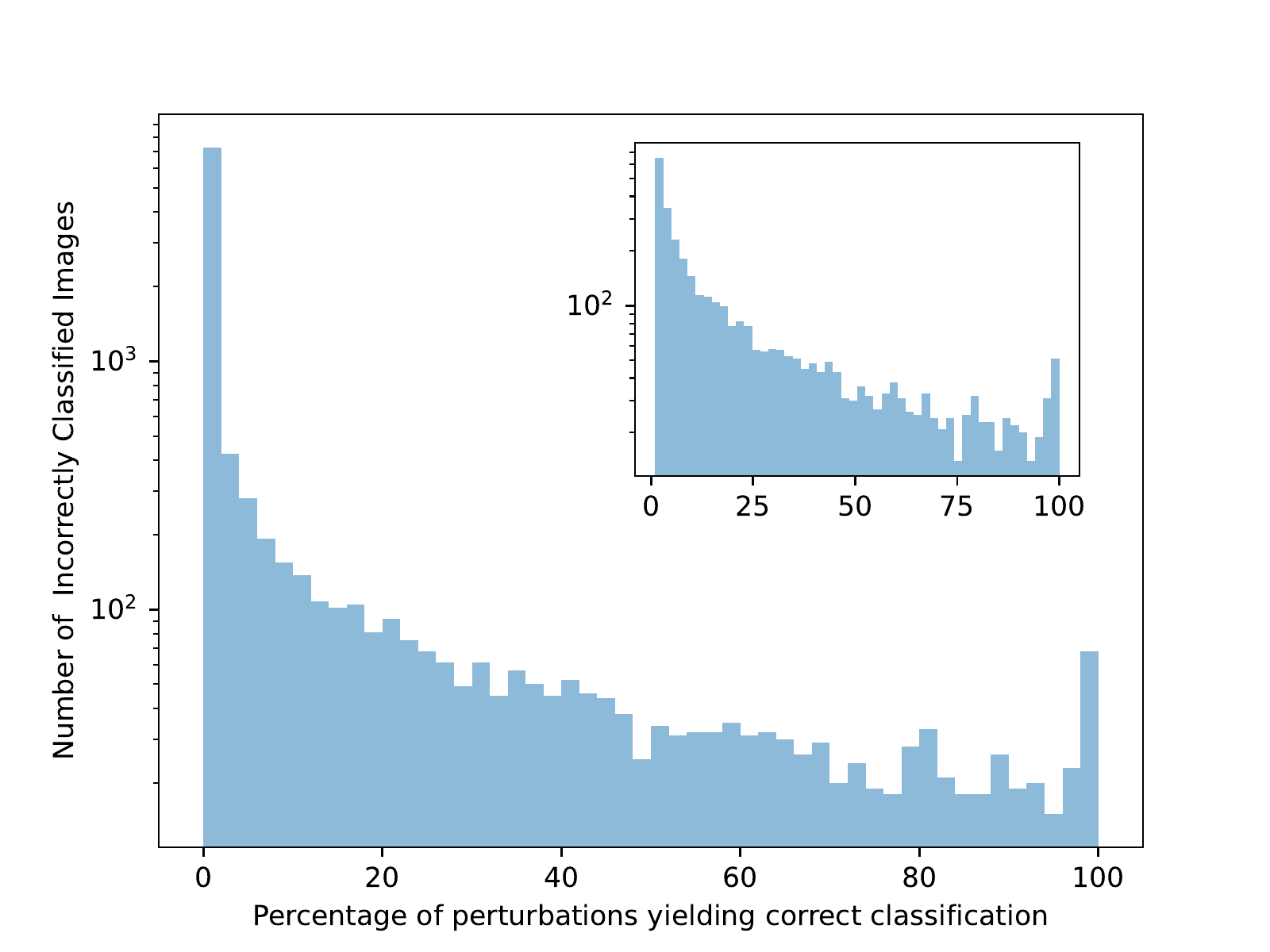}
        \caption{$p=0.01$, train, original PRL}
    \end{subfigure}
    \hfill
    \begin{subfigure}[b]{0.45\linewidth}
        \centering
        \includegraphics[width=\linewidth]{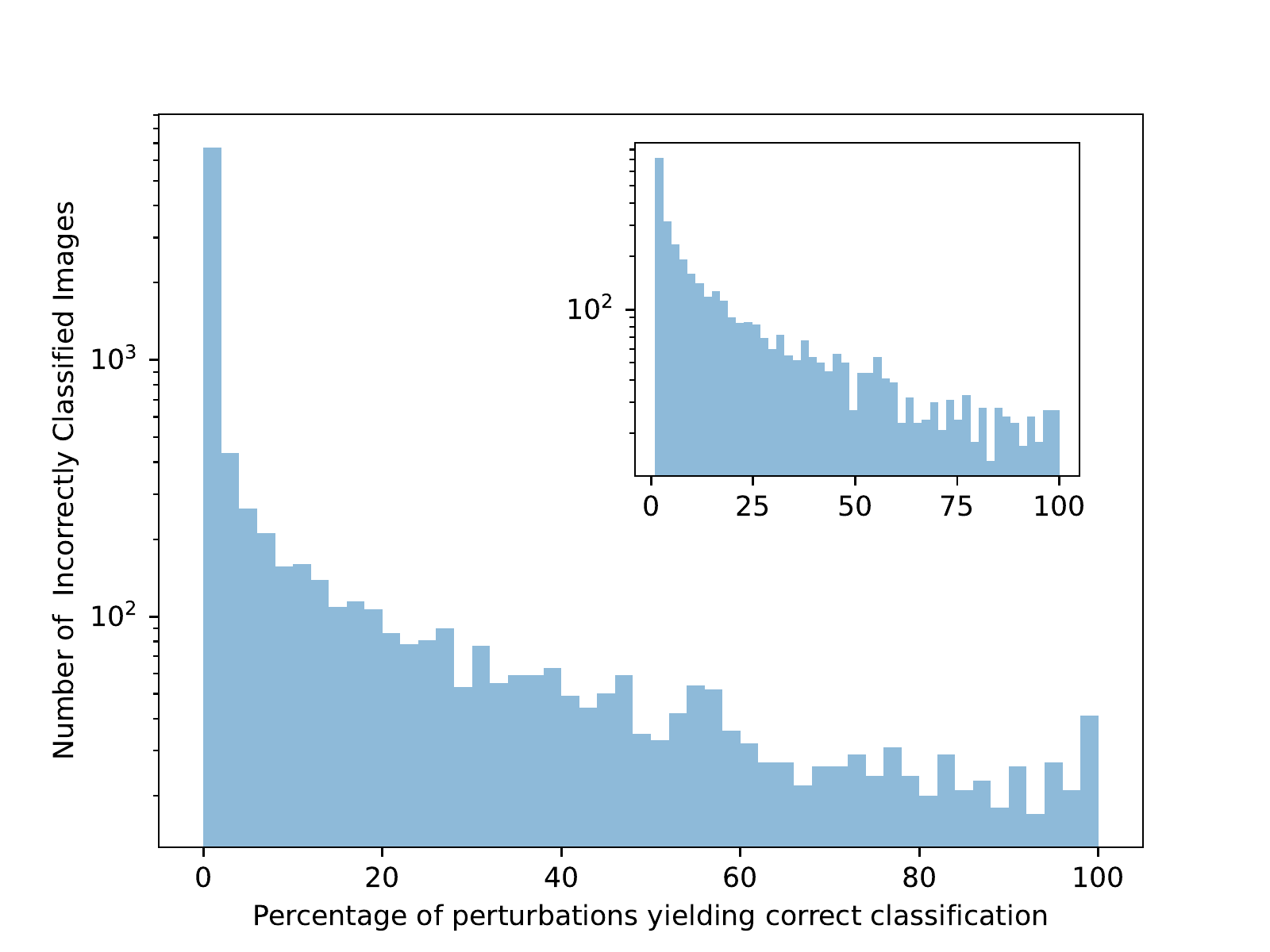}
        \caption{$p=0.01$, train, m-PRL}
    \end{subfigure}

    \begin{subfigure}[b]{0.45\linewidth}
        \centering
        \includegraphics[width=\linewidth]{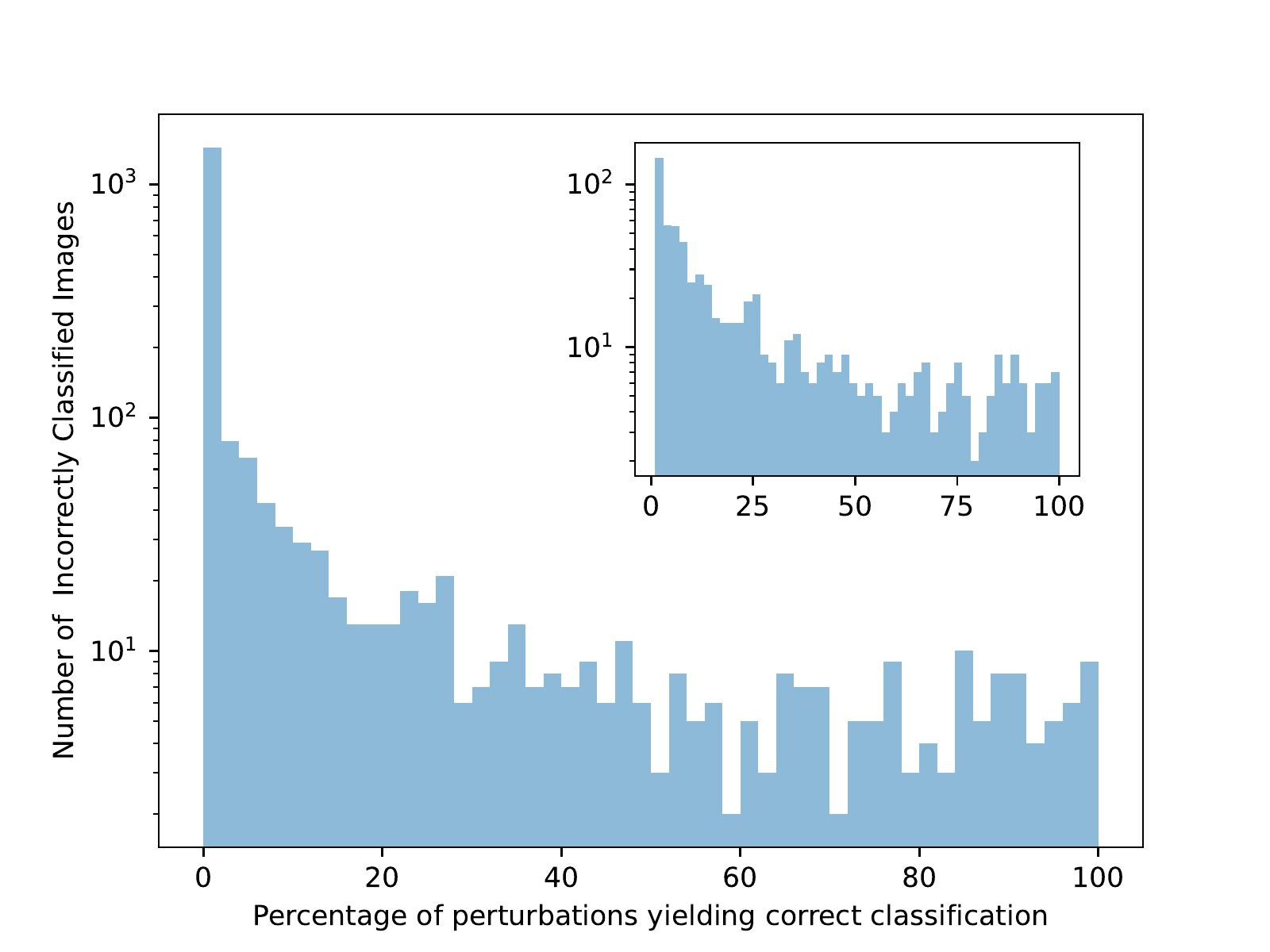}
        \caption{$p=0.01$, test, original PRL}
    \end{subfigure}
    \hfill
    \begin{subfigure}[b]{0.45\linewidth}
        \centering
        \includegraphics[width=\linewidth]{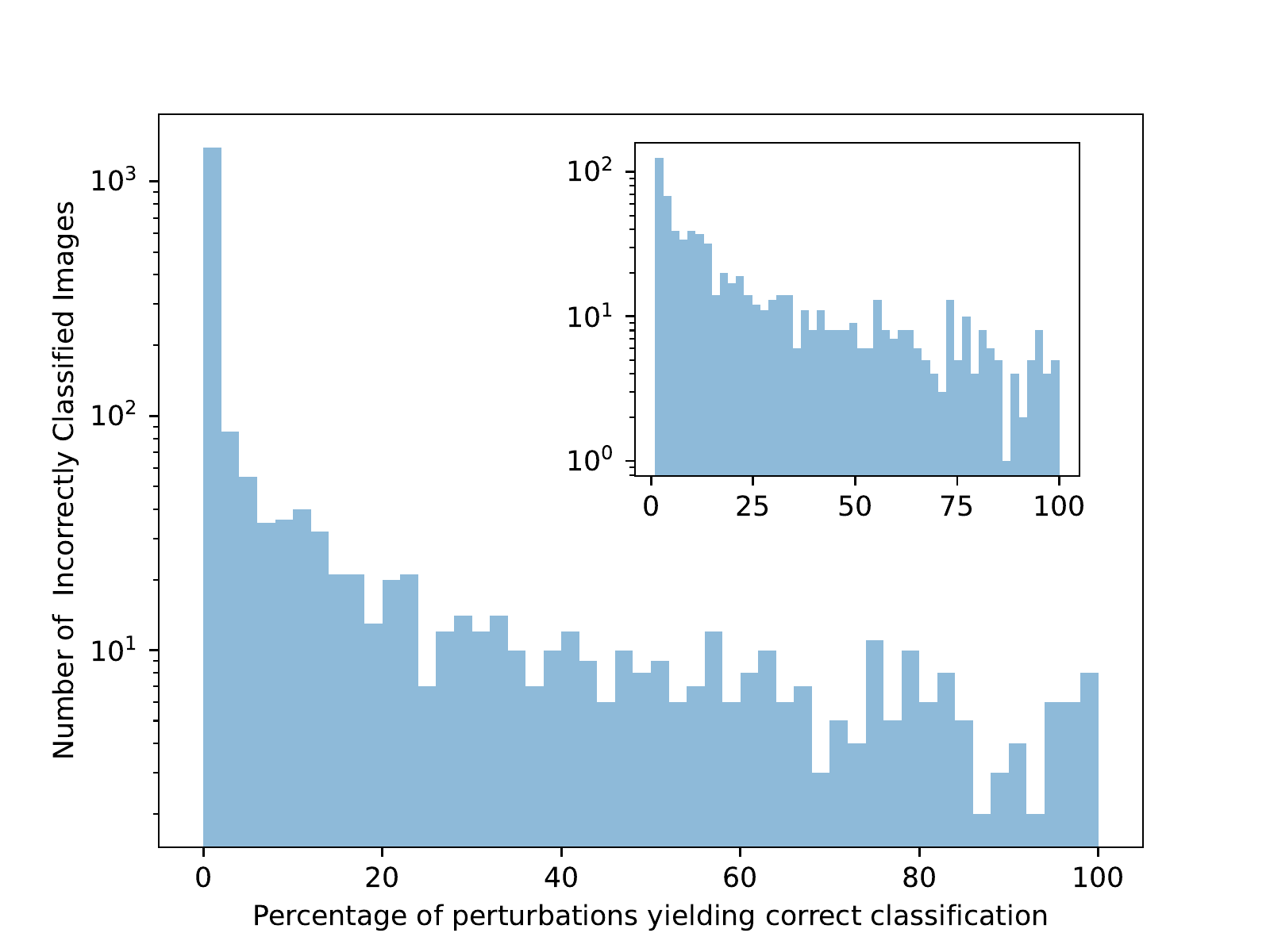}
        \caption{$p=0.01$, test, m-PRL}
    \end{subfigure}
    \caption{Histograms display the distribution of percentages of correctly classified perturbations among misclassified images for both original and m-PRL with parameter $p=0.01$. The inner plot excludes the prevalent $0\%$ case. The plots show that pathological data points---i.e. data points which are misclassified yet most perturbations to the data point are correctly classified---occur in real datasets.}
    \label{fig:cifar_001_histogram}
\end{figure}

\section{Discussion and Conclusion}
\label{sec:Conclusions}
In this paper we considered probabilistically robust learning (PRL), originally proposed in \cite{robey2022probabilistically}. We introduced a modification of the original PRL model that allowed us to address, at least theoretically, the possibility that certain solutions to the model possess pathological points as described in Figure 1. This modification has the appeal of being interpretable through the lens of regularized risk minimization, where the regularization terms take the form of a non-local perimeters of interest in their own right. We discussed an asymptotic expansion for smooth decision boundaries to show that for small adversarial budgets these probabilistic perimeters induce the same regularization effect as the original adversarial training model. 
For binary classification we proved existence of optimal hard classifiers and of very general classes of soft classifiers including neural networks in both the original and modified PRL settings. This was done through novel relaxation techniques taking advantage of the structure of our functionals. Finally, through rigurous $\Gamma$-convergence analysis we provided a detailed discussion on the relation between adversarial training, risk minimization, and all the PRL models, highlighting that the original PRL model fails to interpolate between risk minimization and adversarial training, contrary to claims in previous works. For general (not necessarily binary) problems we showed that the natural loss function to choose is the sample-wise maximum of the standard loss and conditional value at risk (CVaR).

One limitation of PRL is that it does not completely solve the accuracy vs. robustness trade-off, which remains a challenging problem.
Furthermore, while the formal limit of PRL as $p\to 0$ is the worst-case adversarial problem, the algorithms for solving PRL exhibit limitations for very small values of $p$ (in the computation of $\operatorname{CVaR}_p$). 
Still, the results for moderately large values of $p$ are encouraging and future work should focus on understanding of this trade-off better.

The rich mathematical theory developed in this paper opens up new avenues for research, such as the explicit design of probabilistic regularizers for algorithms and exploring the variational convergence of the probabilistic perimeter and its implications for adversarial robustness.

\section*{Acknowledgment}

This material is based upon work supported by the National Science Foundation under Grant Number DMS 1641020 and was started during the summer of 2022 as part of the AMS-MRC program \textit{Data Science at the Crossroads of Analysis, Geometry, and Topology}. NGT was supported by the NSF grants DMS-2005797 and DMS-2236447.   MJ was supported by NSF grant DMS-2400641.

\printbibliography



\appendix

\red

\section{Pathological points in the modified PRL model} 

Contrary to the situation presented in \cref{fig:spike_solution}, where we describe unintuitive features of a solution of the original PRL model, the modified PRL model possesses an interesting stability property that prevents their minimizers from having pathological points (recall the discussion in \cref{sec:Pathological_Sec}). This property is independent of any specifics of the family $\{ \distr_{x} \}_{x \in \X}$ and for example holds in the Euclidean setting when we set $\distr_x$ to be the uniform measure over the ball $B_\veps(x)$.

\begin{proposition}
\label{prop:Patho}
If $A$ is a minimizer of $\ProbRisk_{\Psi_0}$ with $\Psi_0(t) = \one_{t >0}$, then for $\rho_0 $-a.e. $x \in \X$ we have: $\one_{A}(x)=1$ implies $ \distr_x\text{-}\esssup \one_A  =1$. Likewise, for $\rho_1$-a.e. $x \in \X$, $\one_{A^c}(x)=1$ implies $\distr_x\text{-}\esssup \one_{A^c}  =1$.
\end{proposition}
\begin{proof}
We first observe that if $A$ minimizes $\ProbRisk_{\Psi_0}$, then it also minimizes $\Risk_{\Psi_0}$. To see this, let $\ProbRisk_{\Psi_0}^*$ and $\Risk_{\Psi_0}^*$ be the infima of $\ProbRisk_{\Psi_0}$ and $\Risk_{\Psi_0}$, respectively. We then observe that for any given measurable $\tilde{A}$ we have
\[  \Risk_{\Psi_0}(\tilde A) \geq \Risk_{\Psi_0}^*=  \ProbRisk_{\Psi_0}^* =  \ProbRisk_{\Psi_0}(A) \geq \Risk_{\Psi_0}(A), \]
where the second equality follows from \cref{prop:ConsistencyEnergies} and the last inequality from \cref{prop:Ordering}. This proves that $A$ is indeed a minimizer of $\Risk_{\Psi_0}$. From the proof of \cref{prop:Ordering} we had already observed that for every $x \in \X$ 
\[ \one_{A^c}(x) \cdot \distr_x\text{-}\esssup \one_A + \one_{A}(x)  \geq   \distr_x\text{-}\esssup \one_A,  \]
as well as
\[ \one_{A}(x) \cdot \distr_x\text{-}\esssup \one_{A^c} + \one_{A^c}(x)  \geq   \distr_x\text{-}\esssup \one_{A^c} . \]
But since we also have
\begin{align*}
  \Risk_{\Psi_0}(A)=  \ProbRisk_{\Psi_0}(A) & \geq \int _{\X} (\one_{A^c}(x) \cdot \distr_x\text{-}\esssup \one_A + \one_{A}(x) ) \de \rho_0(x)
  \\  & +  \int _{\X}  (\one_{A}(x) \cdot \distr_x\text{-}\esssup \one_{A^c} + \one_{A^c}(x)) \de \rho_1(x)
  \\ & \geq \int _{\X}  \distr_x\text{-}\esssup \one_A  \de \rho_0(x) +  \int _{\X}  \distr_x\text{-}\esssup \one_{A^c}  \de \rho_1(x)
  \\& = \Risk_{\Psi_0}(A),
\end{align*}
we deduce that for $\rho_0$-a.e. $x\in \X$ 
\[  \one_{A^c}(x) \cdot \distr_x\text{-}\esssup \one_A + \one_{A}(x)  =   \distr_x\text{-}\esssup \one_A  \]
and for $\rho_1$-a.e. $x \in \X$ 
\[ \one_{A}(x) \cdot \distr_x\text{-}\esssup \one_{A^c} + \one_{A^c}(x)  =   \distr_x\text{-}\esssup \one_{A^c} . \]
The desired implications now follow.
\end{proof}

\nc

\section{Computational aspects of PRL}
\label{sec:computational}
\subsection{Pseudocode for geometric probabilistically robust learning}

In \Cref{alg_algorithm} we provide a pseudocode for parametrized classifiers $f \equiv f_\theta : \X \to \Y$ based on stochastic gradient descent with batch size $B$.
Furthermore, it involves a sample size of $M$ samples from a distribution $\distr_x$ around an input $x\in\X$, a learning rate $\eta_\alpha$ for the inner optimization in CVaR, and a learning rate $\eta$ for the parameter updates.
The pseudocode is a straightforward generalization of \cite[Algorithm 1]{robey2022probabilistically} and we implemented it in their code framework.\footnote{\url{https://github.com/arobey1/advbench}}
The code which can be used to reproduce our results is part of the supplementary material of this paper.

\begin{algorithm}[htb]
\renewcommand{\algorithmiccomment}[1]{\hfill\eqparbox{COMMENT}{$\triangleright$ #1}}
\begin{algorithmic}[1]
\For{minibatch $(x_j,y_j)_{j=1}^B$} 
    \For{$T$ steps} \Comment{Approximate solution of inner problem}
        \State Draw $\displaystyle x'_k \sim \distr_{x_j}$, $k=1,\dots,M$
        \State $\displaystyle g_{\alpha_j}\gets 1-\frac{1}{p M}\sum_{k=1}^M\one_{\displaystyle\ell(f_\theta(x'_k),y_j)\geq\alpha_j}$
        \State $\displaystyle \alpha_j \gets \alpha_j - \eta_\alpha g_{\alpha_j}$
    \EndFor
    \State $\displaystyle S_j \gets \alpha_j + {\frac{1}{pM}\sum_{k=1}^M\left(\ell(f_\theta(\displaystyle x'_k),y_j)-\alpha_j\right)_+}$ \Comment{Approximate value of $\operatorname{CVaR}_p$}
    \If{$\displaystyle S_j > \ell(f_\theta(x_j),y_j))$} \Comment{If $\operatorname{CVaR}_p$ kicks in}
        \State $\displaystyle g_j\gets \frac{1}{p M}\sum_{k=1}^M\nabla_\theta\left(\ell(f_\theta(\displaystyle x'_k),y_j)-\alpha_j\right)_+$
    \Else  \Comment{If it doesn't}
        \State $\displaystyle g_j\gets\nabla_\theta\ell(f_\theta(x_j),y_j)$
    \EndIf
    \State $\displaystyle g\gets \frac{1}{B}\sum_{j=1}^B g_j$ \Comment{Compute full $\theta$-gradient}
    \State $\displaystyle v \gets \operatorname{optimizer}(g)$ \Comment{AdaDelta or SGD(+M)}
    \State $\displaystyle \theta \gets \theta - \eta v$ \Comment{Update parameters}
\EndFor
\end{algorithmic}
\caption{\label{alg_algorithm}Proposed algorithm for solving \labelcref{eq:proposed_problem_relaxed} for $p\in(0,1)$.}
\end{algorithm}

\subsection{Training details}
\label{sec:training_details}
Hyperparameter values specific to \Cref{alg_algorithm} used in training are presented in \Cref{tab:hyperparams}. 
Following \cite{robey2022probabilistically},
we use AdaDelta \cite{zeiler2012adadelta} for MNIST experiments and SGD with momentum for CIFAR-10 experiments. 
The MNIST experiments use a CNN architecture with two convolutional layers (32 and 64 filters, size 3x3), two dropout layers (dropout probabilities 0.25, 0.5), and two fully connected layers (dimensions 9216 to 128 and 128 to 10).
A ResNet-18 \cite{he2016deep} is used in the CIFAR-10 experiments. 
The hyperparameter values used for these algorithms are contained in {\tt hparams\_registry.py} in the accompanying code. 

\begin{table}[htbp]
    \caption{Hyperparameters used for training. The probability distribution $\distr_x$ is always taken to be the uniform distribution over the ball $B_\eps(x)$. Note that $p$ is called {\tt beta} in the code. For consistency we always used the same hyperparameter values for the ``Geometric'' and ``Original'' versions.}
    \label{tab:hyperparams}
    \centering
    \begin{tabular}{cllcccc}
    \toprule
    \textbf{Data} & \textbf{$p$} & $\varepsilon$ & $\eta_{\alpha}$ & M & T \\
    \midrule
    \multirow{4}{*}{\rotatebox[origin=c]{90}{\textbf{MNIST}}} & $0.01$ & $0.3$&$0.1$ & $20$ & $5$ \\
                      & $0.1$ & $0.3$& $1.0$&$20$ & $5$ \\
                      & $0.3$ & $0.3$&$1.0$ & $20$ & $5$ \\
                      & $0.5$ & $0.3$&$1.0$ & $20$ & $5$ \\
     \midrule[2pt]
    \multirow{4}{*}{\rotatebox[origin=c]{90}{\textbf{CIFAR-10}}} & $0.01$ & $8/255$ & $0.1$ & $20$ & $5$ \\
        & $0.1$ & $8/255$ & $1.0$ & $20$ & $5$ \\
        & $0.3$ & $8/255$ & $1.0$ & $20$ & $5$ \\
        & $0.5$ & $8/255$ & $1.0$ & $20$ & $5$ \\
     \bottomrule
    \end{tabular}
    \end{table}

\subsection{Computational resources used}
We performed the majority of the prototyping and some experimentation on a LambdaLabs Vector workstation equipped with 3 NVIDIA A6000 GPUs. We estimate that we used approximately 500 GPU-hours on this machine. We supplemented this with 550 GPU-hours of cloud compute---using the Lambda GPU cloud---predominately on instances equipped with a single A10 GPU. 

\clearpage
\subsection{Further numerical experiments}\label{sec:further_exp}
In \cref{fig:cifar_01_histogram,fig:cifar_03_histogram,fig:cifar_05_histogram}
we provide further numerical experiments on pathological points with parameters $p=0.1, 0.3$, and $0.5$
\begin{figure}[ht]
    \begin{subfigure}[b]{0.45\linewidth}
        \centering
        \includegraphics[width=\linewidth]{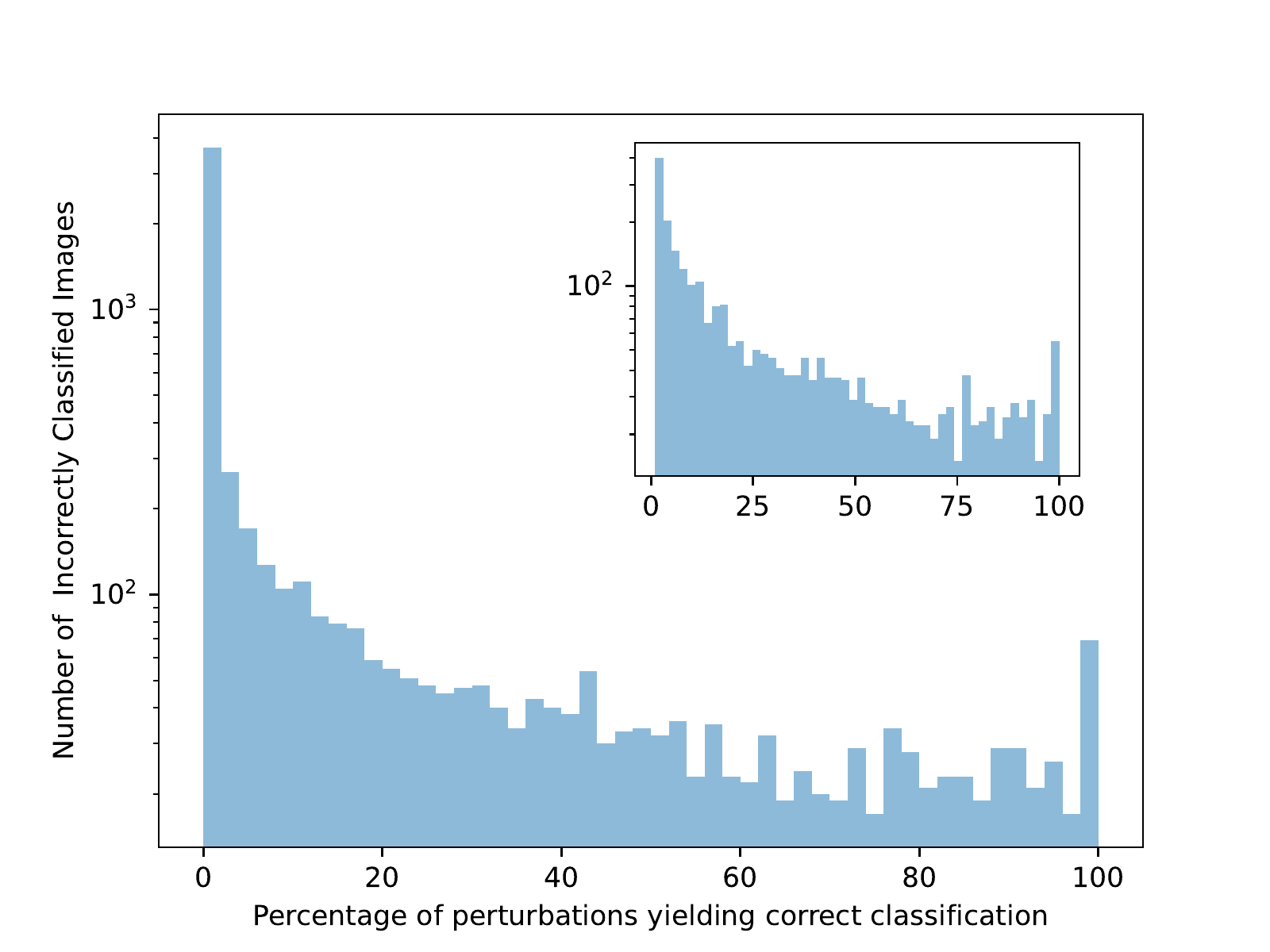}
        \caption{$p=0.1$, train, original PRL}
    \end{subfigure}
    \hfill
    \begin{subfigure}[b]{0.45\linewidth}
        \centering
        \includegraphics[width=\linewidth]{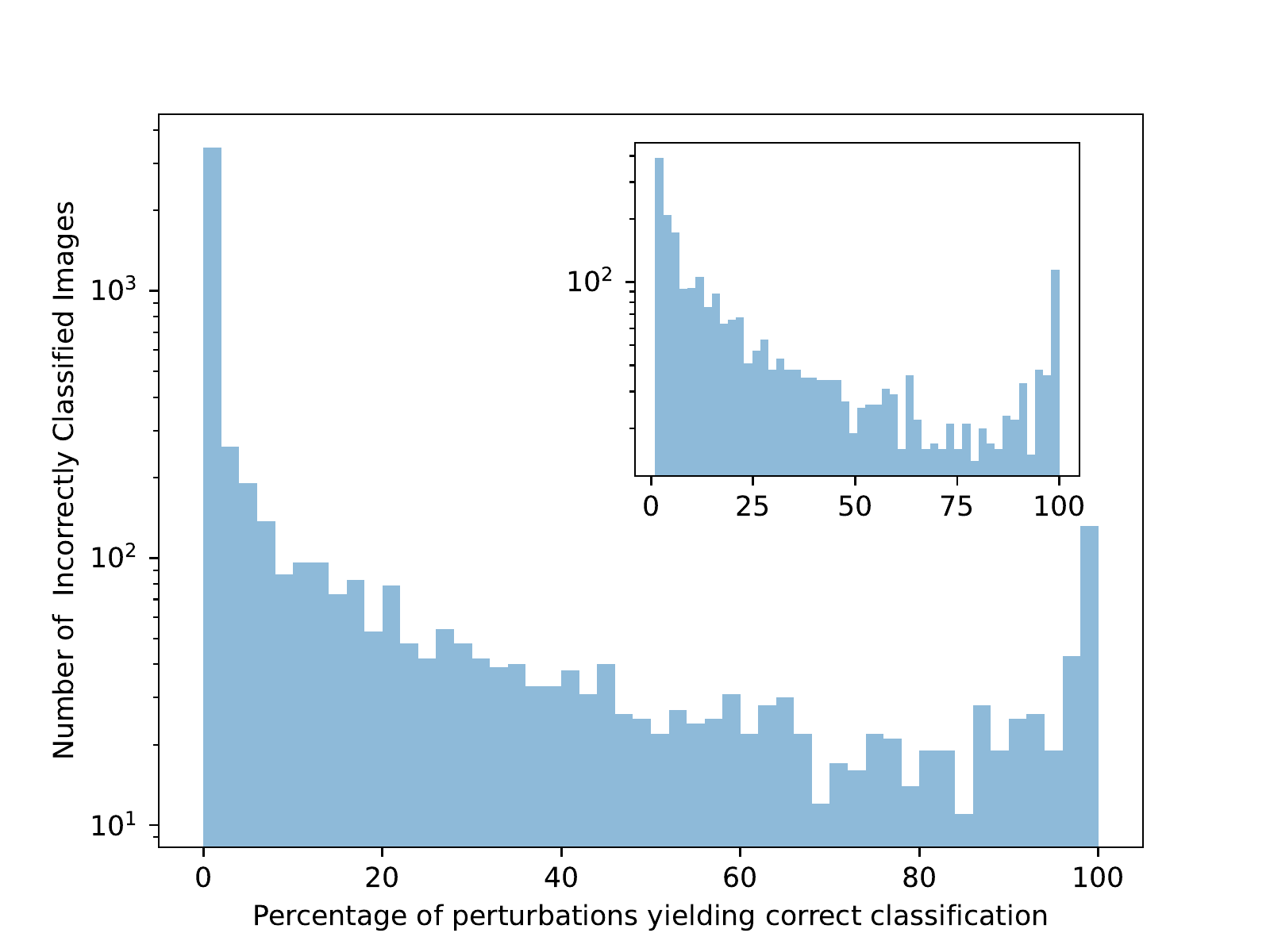}
        \caption{$p=0.1$, train, m-PRL}
    \end{subfigure}

    \begin{subfigure}[b]{0.45\linewidth}
        \centering
        \includegraphics[width=\linewidth]{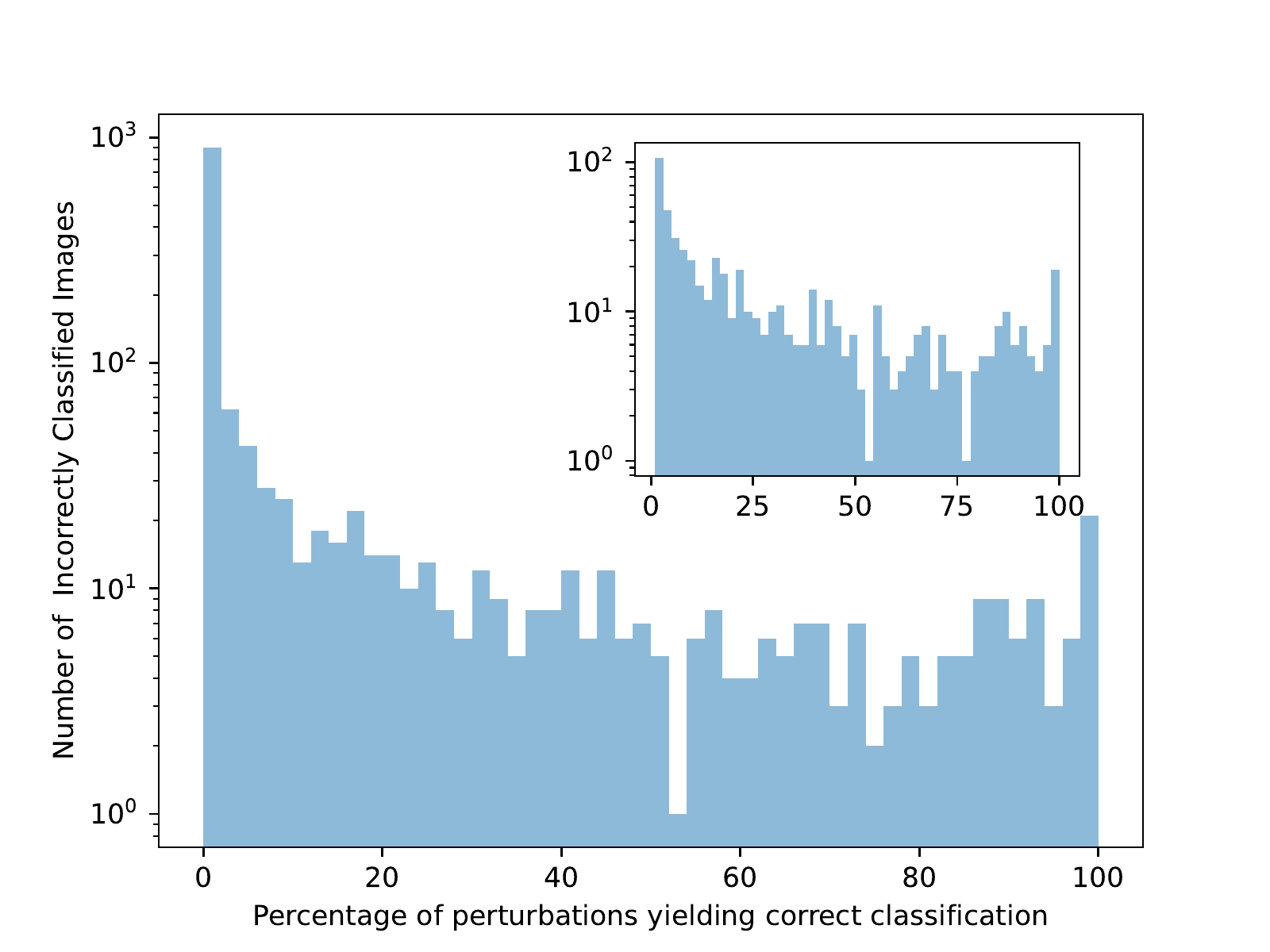}
        \caption{$p=0.1$, test, original PRL}
    \end{subfigure}
    \hfill
    \begin{subfigure}[b]{0.45\linewidth}
        \centering
        \includegraphics[width=\linewidth]{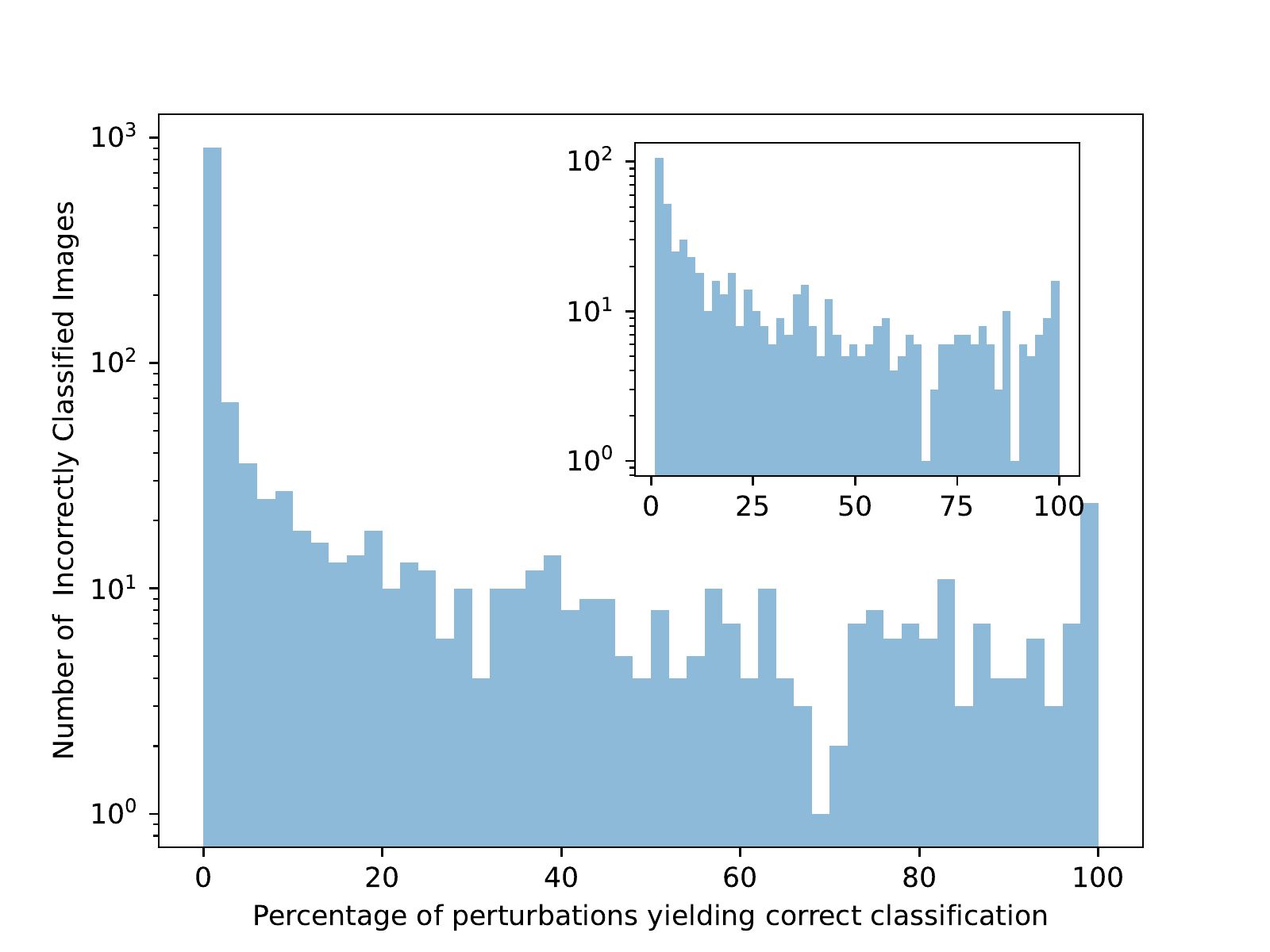}
        \caption{$p=0.1$, test, m-PRL}
    \end{subfigure}
    \caption{Histograms showing the distribution of percentages of correctly classified perturbations among misclassified images for both original and m-PRL with parameter $p=0.1$.}
    \label{fig:cifar_01_histogram}
\end{figure}

\begin{figure}
    \begin{subfigure}[b]{0.45\linewidth}
        \centering
        \includegraphics[width=\linewidth]{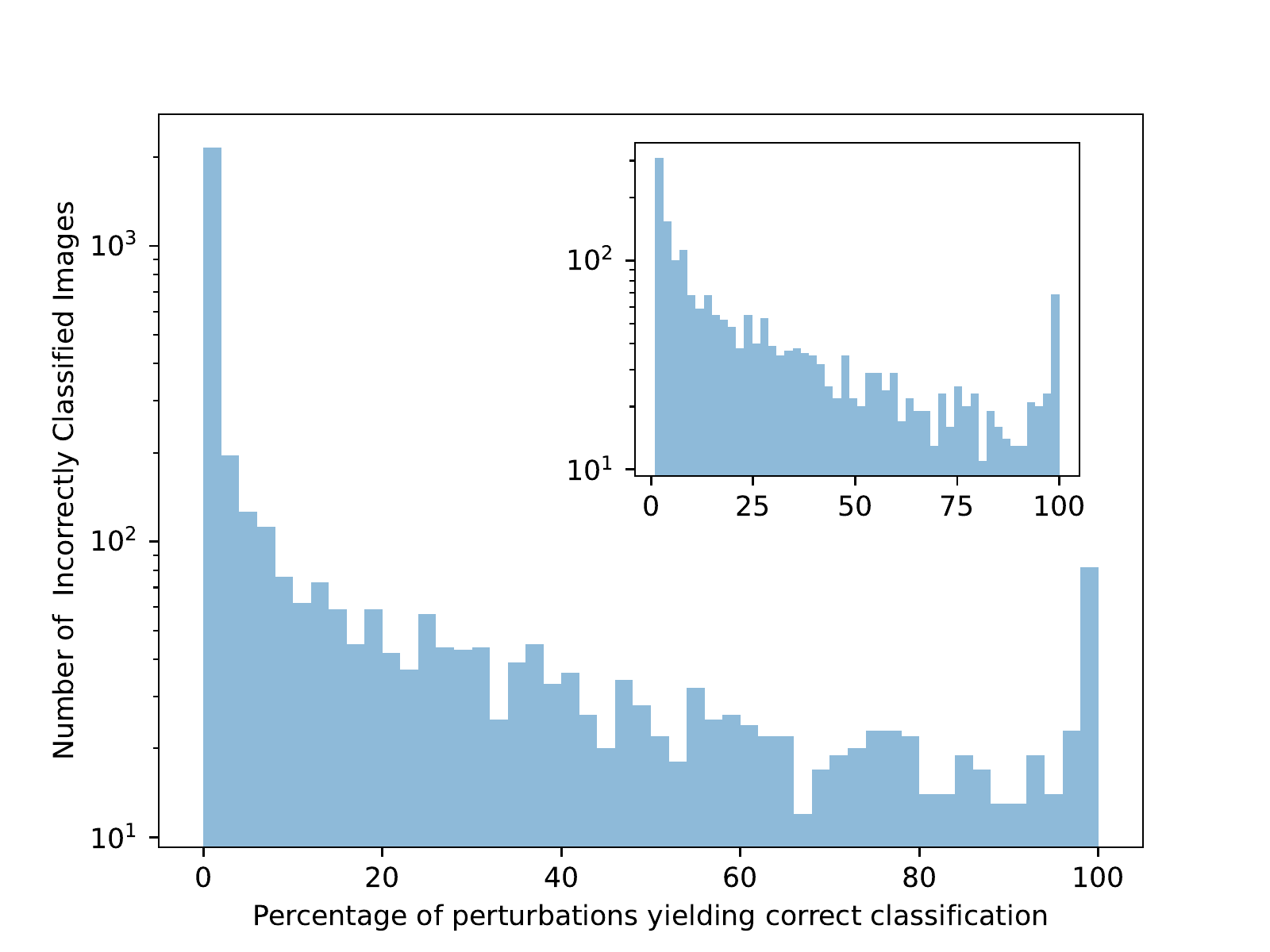}
        \caption{$p=0.3$, train, original PRL}
    \end{subfigure}
    \hfill
    \begin{subfigure}[b]{0.45\linewidth}
        \centering
        \includegraphics[width=\linewidth]{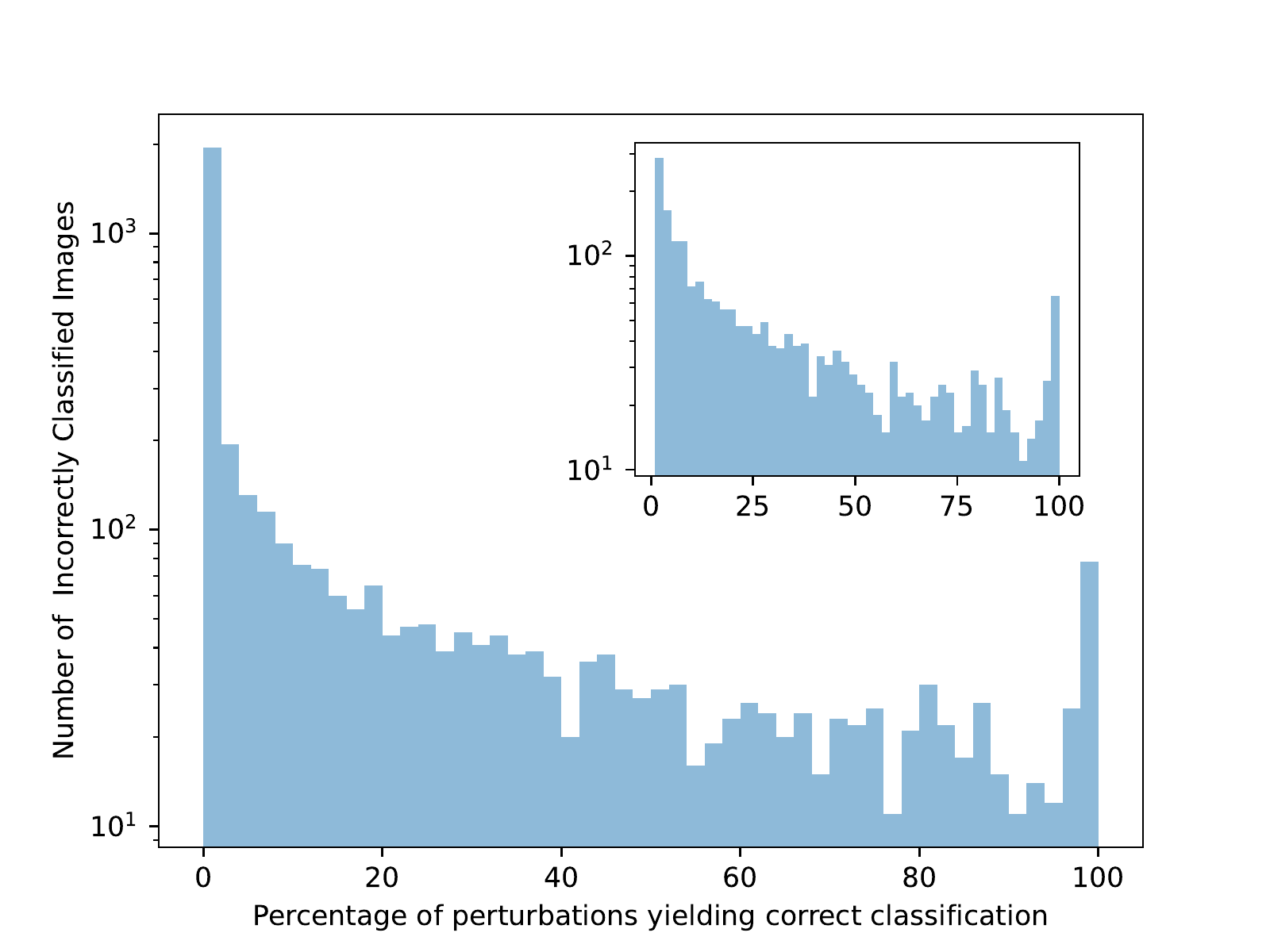}
        \caption{$p=0.3$, train, m-PRL}
    \end{subfigure}

    \begin{subfigure}[b]{0.45\linewidth}
        \centering
        \includegraphics[width=\linewidth]{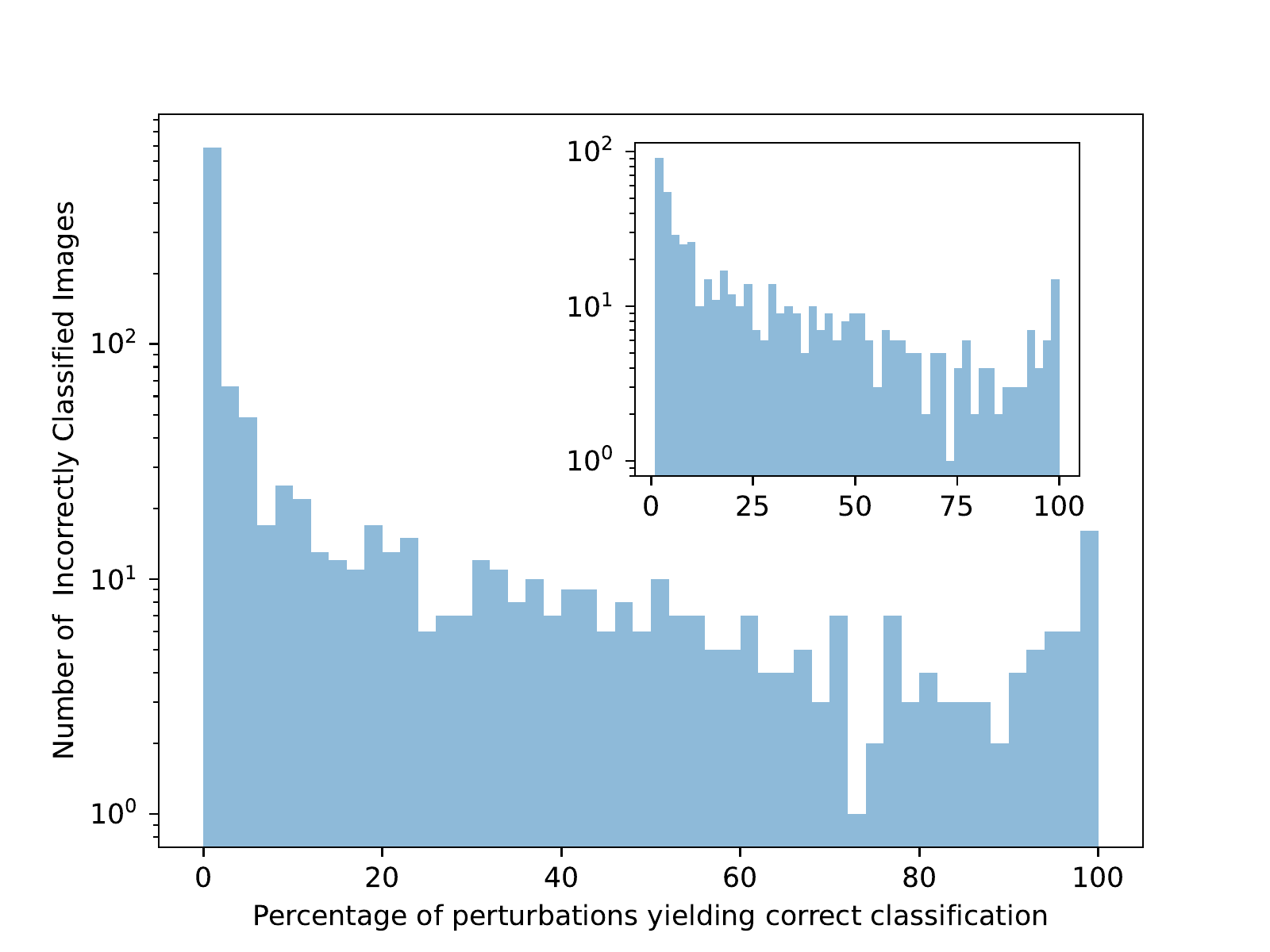}
        \caption{$p=0.3$, test, original PRL}
    \end{subfigure}
    \hfill
    \begin{subfigure}[b]{0.45\linewidth}
        \centering
        \includegraphics[width=\linewidth]{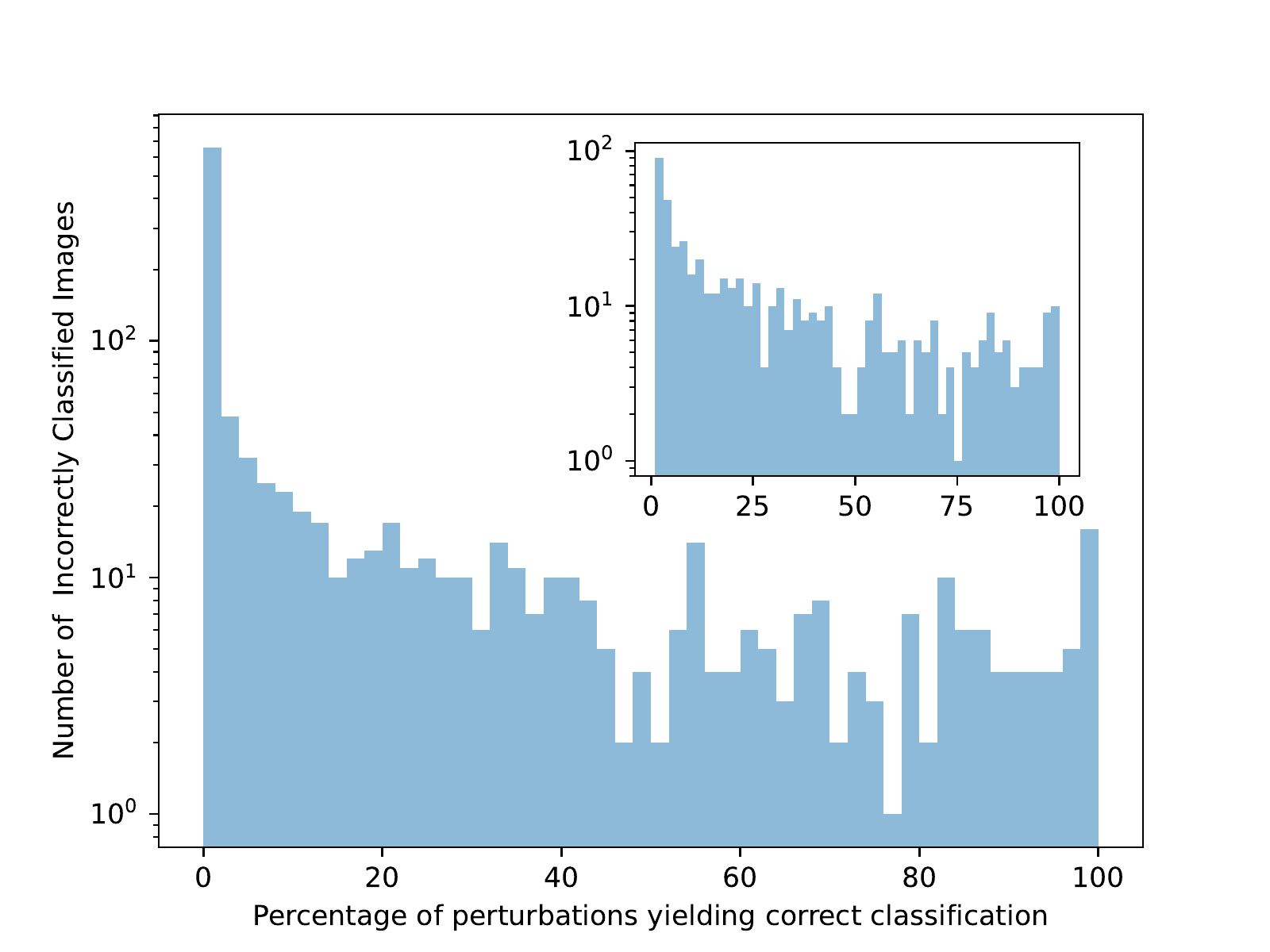}
        \caption{$p=0.3$, test, m-PRL}
    \end{subfigure}
    \caption{Histograms showing the distribution of percentages of correctly classified perturbations among misclassified images for both original and m-PRL with parameter $p=0.3$.}
    \label{fig:cifar_03_histogram}
\end{figure}

\begin{figure}
    \begin{subfigure}[b]{0.45\linewidth}
        \centering
        \includegraphics[width=\linewidth]{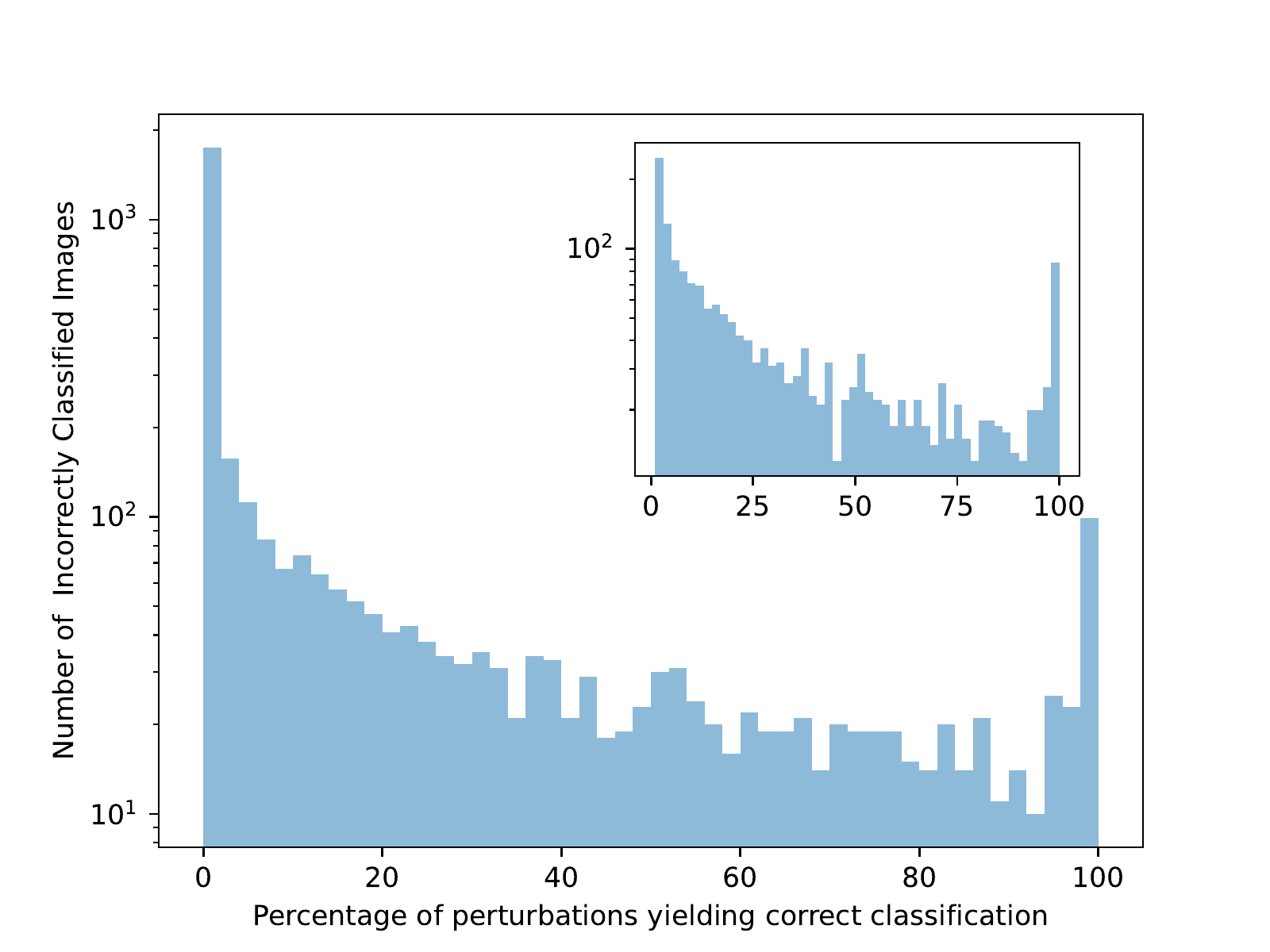}
        \caption{$p=0.5$, train, original PRL}
    \end{subfigure}
    \hfill
    \begin{subfigure}[b]{0.45\linewidth}
        \centering
        \includegraphics[width=\linewidth]{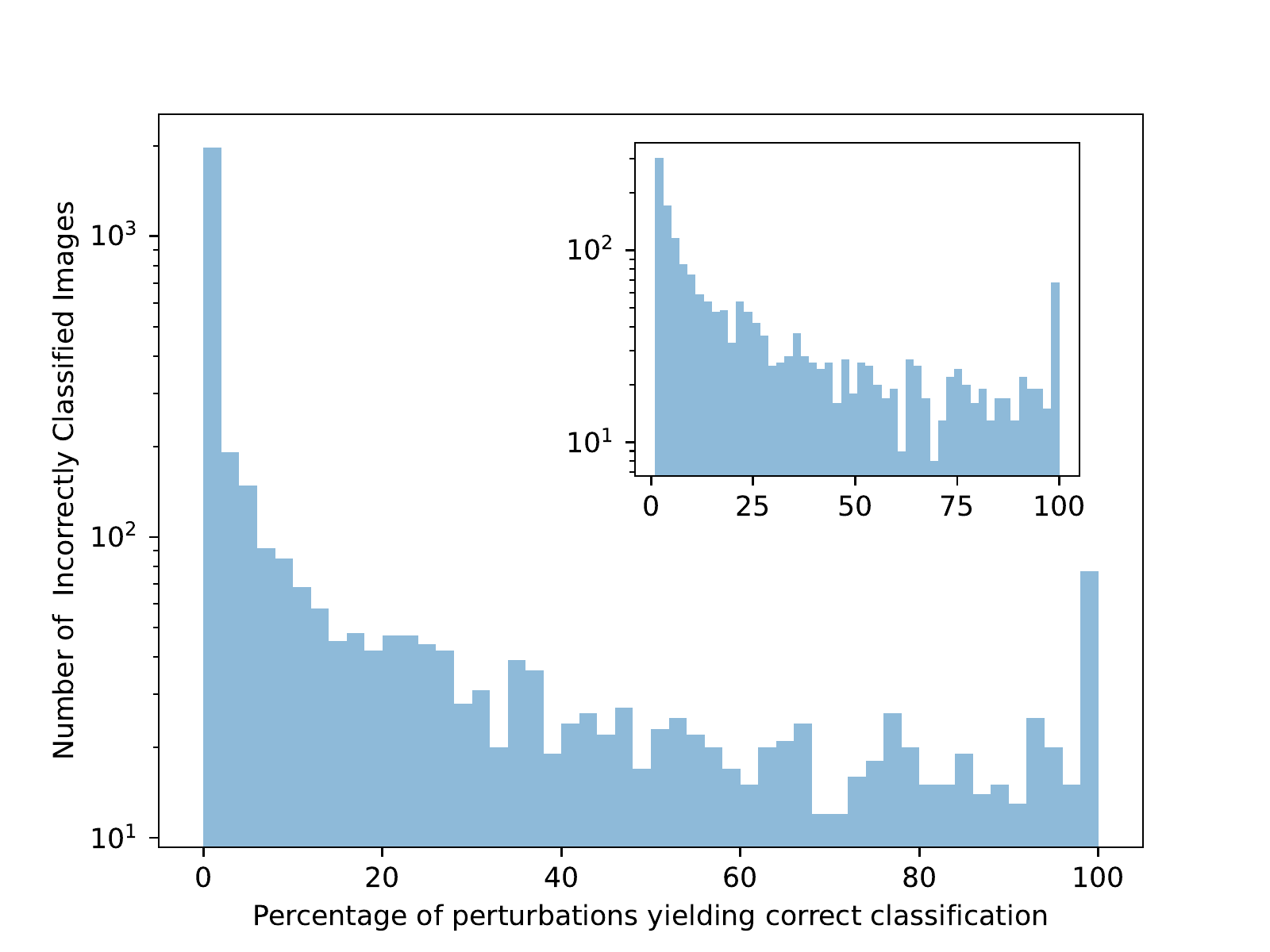}
        \caption{$p=0.5$, train, m-PRL}
    \end{subfigure}

    \begin{subfigure}[b]{0.45\linewidth}
        \centering
        \includegraphics[width=\linewidth]{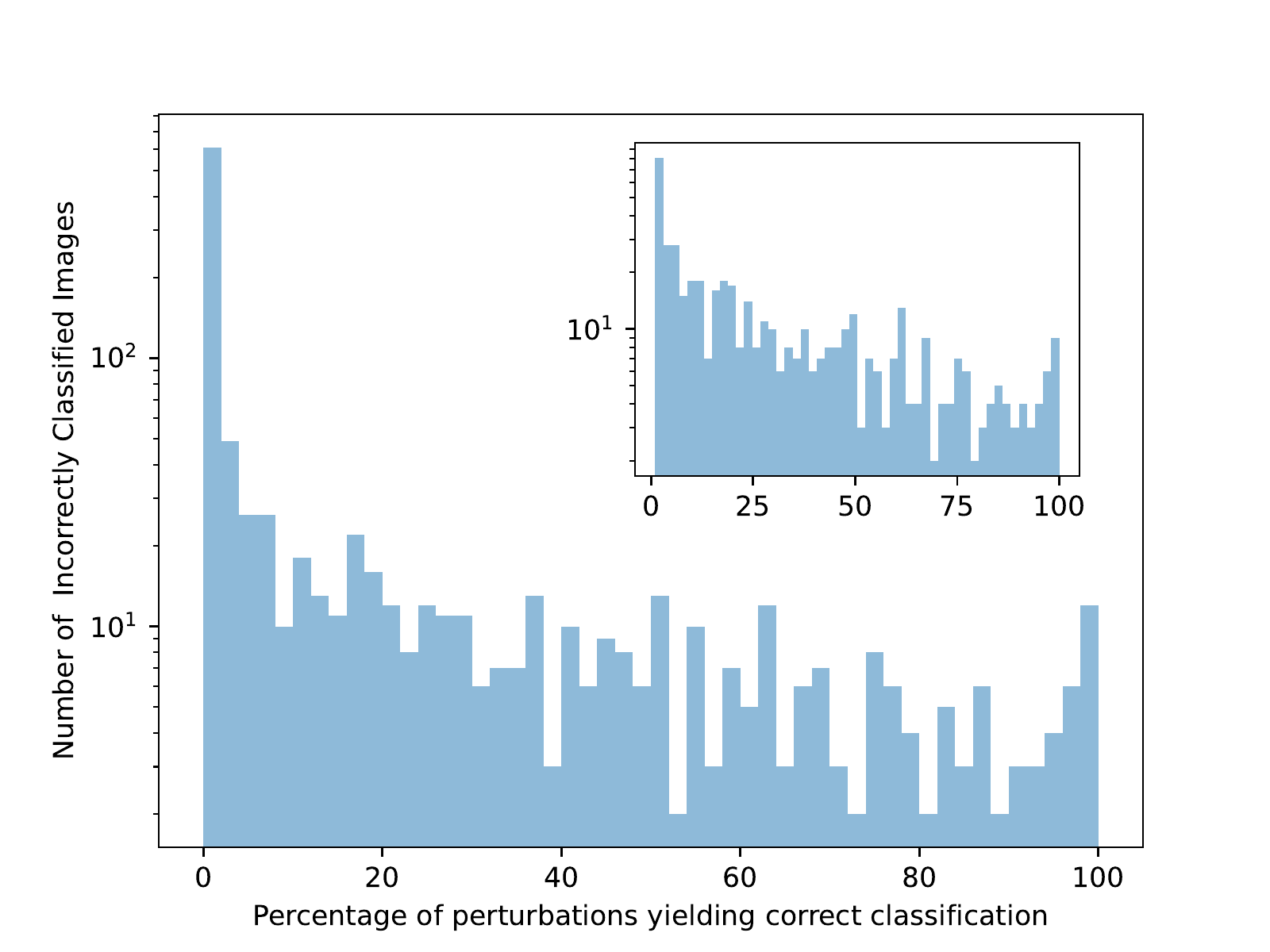}
        \caption{$p=0.5$, test, original PRL}
    \end{subfigure}
    \hfill
    \begin{subfigure}[b]{0.45\linewidth}
        \centering
        \includegraphics[width=\linewidth]{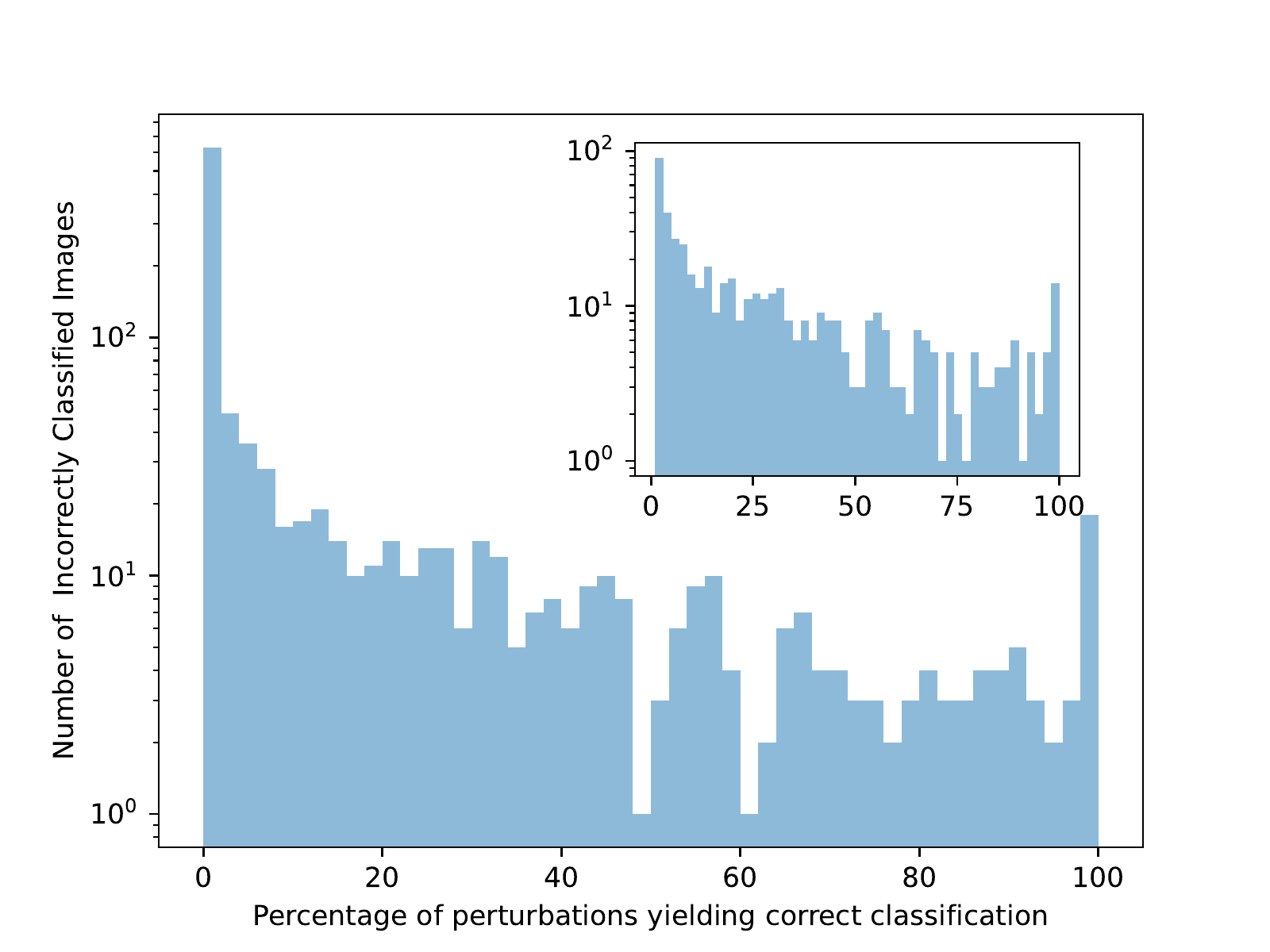}
        \caption{$p=0.5$, test, m-PRL}
    \end{subfigure}
    \caption{Histograms showing the distribution of percentages of correctly classified perturbations among misclassified images for both original and m-PRL with parameter $p=0.5$.}
    \label{fig:cifar_05_histogram}
\end{figure}

\clearpage

\end{document}